%% file: main_arxiv_v2.tex
\documentclass[runningheads]{llncs}

\usepackage[dvipsnames,table,xcdraw]{xcolor}
\usepackage{graphicx}
\usepackage{tikz}
\usepackage{comment}
\usepackage{amsmath,amssymb} % define this before the line numbering.
\input{math_commands.tex}

\usepackage{booktabs}
\usepackage{url}
\usepackage{bbm}
\usepackage{wrapfig}
\usepackage{subcaption}
\usepackage{rotating}

\usepackage[accsupp]{axessibility}  % Improves PDF readability for those with disabilities.

\usepackage[width=122mm,left=12mm,paperwidth=146mm,height=193mm,top=12mm,paperheight=217mm]{geometry}

\usepackage[pagebackref,breaklinks,colorlinks]{hyperref}

\usepackage[capitalize]{cleveref}

\include{prel}
\begin{document}
% \linenumbers
\pagestyle{headings}
\mainmatter
\def\ECCVSubNumber{7187}  % Insert your submission number here

\title{BINAS: Bilinear Interpretable \\ Neural Architecture Search}

\begin{comment}
\titlerunning{ECCV-22 submission ID \ECCVSubNumber} 
\authorrunning{ECCV-22 submission ID \ECCVSubNumber} 
\author{Anonymous ECCV submission}
\institute{Paper ID \ECCVSubNumber}
\end{comment}
%******************

% \begin{comment}
\titlerunning{BINAS: Bilinear Interpretable Neural Architecture Search}
\author{Niv Nayman\thanks{Equal contribution}\inst{1,2} \and
Yonathan Aflalo$^*$\inst{1} \and \\
Asaf Noy\inst{1} \and
Rong Jin\inst{1} \and
Lihi Zelnik-Manor\inst{2}}
\authorrunning{N. Nayman et al.}
\institute{Alibaba Group, Tel Aviv, Israel \\ 
\email{\{niv.nayman,jonathan.aflalo,asaf.noy,jinrong.jr\}@alibaba-inc.com}
\and
Technion - Israel Institute of Technology, Haifa, Israel \\
\email{lihi@technion.ac.il}}
% \end{comment}
%******************
\maketitle
\vspace{-5mm}
\input{abstract}
\input{introduction}
\input{related}
\input{method}
\input{exp}
\input{conclusion}

\clearpage
\bibliographystyle{splncs04}
\renewcommand{\doi}[1]{\url{https://doi.org/#1}}
\bibliography{references}

\input{supplementary}
\end{document}

%% file: math_commands.tex
\usepackage{amsmath,amsfonts,bm}

\def\eqref#1{equation~\ref{#1}}
\def\1{\bm{1}}

\newcommand{\test}{\mathcal{D_{\mathrm{test}}}}

\def\vu{{\bm{u}}}

\DeclareMathAlphabet{\mathsfit}{\encodingdefault}{\sfdefault}{m}{sl}
\SetMathAlphabet{\mathsfit}{bold}{\encodingdefault}{\sfdefault}{bx}{n}

\newcommand{\R}{\mathbb{R}}

%% file: prel.tex
\makeatletter
\@namedef{ver@everyshi.sty}{}
\makeatother

\usepackage{pgfplots}
\usepackage{adjustbox}
\usepackage{enumerate}
\usepackage{textcomp}
\usepackage{tikz}
\usepackage{tabu}
\usepackage{caption}
\usepackage{dsfont}
\usetikzlibrary{shapes, arrows, decorations.text, shadows.blur, backgrounds, positioning,fit,calc}
\usepackage{array} % and/or
\usepackage{longtable} % and/or
\pgfplotsset{compat=1.16}
\usepackage{microtype}
\usepackage{multirow}
\usepackage{makecell} 
\usepackage{pifont}
\usepackage{booktabs} % for professional tables
\usepackage{algorithmic}
\usepackage{algorithm}
\usepackage{mathtools}

\newtheorem{de}{Definition}[section]
\newtheorem{theo}{Theorem}[section]

\newtheorem{lem}[theo]{Lemma}
\newcommand{\argmax}{\operatornamewithlimits{argmax~}}
\newcommand{\argmin}{\operatornamewithlimits{argmin~}}

\newcommand{\balpha}{\bm{\alpha}}
\newcommand{\bbeta}{\bm{\beta}}
\newcommand{\btau}{\bm{\tau}}
\newcommand{\brho}{\bm{\rho}}

\newcommand{\prob}[1]{\mathbb{P}\left[ #1 \right]}
\newcommand{\expec}[1]{\mathbb{E}\left[ #1 \right]}

\newenvironment{disarray}%
 {\everymath{\displaystyle\everymath{}}\array}%
 {\endarray}
\usepackage{pgfpages}

\usepackage{relsize}
\usepackage{lmodern}

\usepackage[nodisplayskipstretch]{setspace}
\AtBeginDocument{%
  \setlength\abovedisplayskip{4pt}
  \setlength\belowdisplayskip{4pt}
  }

\usepackage{listings}

\definecolor{codegreen}{rgb}{0,0.6,0}
\definecolor{codegray}{rgb}{0.5,0.5,0.5}
\definecolor{codepurple}{rgb}{0.58,0,0.82}
\definecolor{backcolour}{rgb}{0.95,0.95,0.92}

\lstdefinestyle{mystyle}{
    backgroundcolor=\color{backcolour},   
    commentstyle=\color{codegreen},
    keywordstyle=\color{magenta},
    numberstyle=\tiny\color{codegray},
    stringstyle=\color{codepurple},
    basicstyle=\ttfamily\footnotesize,
    breakatwhitespace=false,         
    breaklines=true,                 
    captionpos=b,                    
    keepspaces=true,                 
    numbers=left,                    
    numbersep=5pt,                  
    showspaces=false,                
    showstringspaces=false,
    showtabs=false,                  
    tabsize=2
}

\definecolor{darkbluetemp}{rgb}{0,0.08,0.45}

%% file: abstract.tex
\begin{abstract}
Making neural networks practical often requires adhering to resource constraints such as latency, energy and memory.
To solve this we introduce a \textit{Bilinear Interpretable} approach for constrained \textit{Neural Architecture Search (BINAS)} that is based on an accurate and simple bilinear formulation of both an accuracy estimator and the expected resource requirement, jointly with a scalable search method with theoretical guarantees. 
One major advantage of BINAS is providing interpretability via insights about the contribution of different design choices. For example, we find that in the examined search space, adding depth and width is more effective at deeper stages of the network and at the beginning of each resolution stage.
BINAS differs from previous methods that typically use complicated accuracy predictors that make them hard to interpret, sensitive to many hyper-parameters, and thus with compromised final accuracy.
Our experiments
\footnote{The full code: \url{https://github.com/Alibaba-MIIL/BINAS}
} 
show that BINAS generates comparable to or better than state of the art architectures, while reducing the marginal search cost, as well as strictly satisfying the resource constraints.
% \keywords{Neural Architecture Search, Computer Vision, Deep Learning, Optimization}
\end{abstract}

%% file: introduction.tex
\vspace{-8mm}
\section{Introduction}
The increasing utilization of Convolutional Neural Networks (CNN) in real systems and commercial products puts neural networks with both high accuracy and fast inference speed in high demand. Early days architectures, such as VGG~\cite{VGG} or ResNet~\cite{ResNet}, were designed for powerful GPUs as those were the common computing platform for deep CNNs, however, in recent years the need for deployment on standard CPUs and edge devices emerged. These computing platforms are limited in their abilities and as a result require lighter architectures that comply with strict requirements on real time latency and power consumption. This has spawned a line of research aimed at finding architectures with both high performance and constrained resource demands. 
  
The main approaches to solve this evolved from Neural Architecture Search (NAS)~\cite{zoph2016neural,liu2018darts,cai2018proxylessnas}, while a constraint on the target latency is added over various platforms, e.g., CPU, TPU, FPGA, MCU etc. 
Those constrained-NAS methods can be grouped into two categories: (i) Reward based methods such as Reinforcement-Learning (RL) or Evolutionary Algorithms (EA) \cite{OFA,tan2019mnasnet,effnet,mobilenetv3}, where the latency and accuracy of sampled architectures are predicted by evaluations on the target devices over some validation set to perform the search.
The predictors are typically made of complicated models and hence require many samples and sophisticated fitting techniques~\cite{white2021powerful}. Overall this oftentimes leads to inaccurate, expensive to acquire, and hard to optimize objective functions due to their complexity. (ii) Resource-aware gradient based methods formulate a differentiable loss function consisting of a trade-off between an accuracy term and either a proxy soft penalty term~\cite{TF-NAS,fbnet} or a hard constraint~\cite{nayman2021hardcore}. Therefore, the architecture can be directly optimized via bi-level optimization using stochastic gradient descent (SGD)~\cite{SGD} or stochastic Frank-Wolfe (SFW)~\cite{hazan2016variance}, respectively. However, the bi-level nature of the problem introduces many challenges~\cite{P-DARTS,DARTS+,noy2020asap,nayman2019xnas} and recently~\cite{wang2021rethinking} pointed out the inconsistencies associated with using gradient information as a proxy for the quality of the architectures, especially in the presence of skip connections in the search space. These inconsistencies call for making NAS more interpretable, by extending its scope from finding optimal architectures to interpretable features \cite{ru2021interpretable} and their corresponding impact on the network performance.

In this paper, we propose an interpretable search algorithm that is fast and scalable, yet produces architectures with high accuracy that satisfy hard latency constraints.
At the heart of our approach is an accuracy estimator which is interpretable, easy to optimize and does not have a strong reliance on gradient information. Our proposed predictor measures the performance contribution of individual design choices by sampling sub-networks from a one-shot model~(\cite{bender2018understanding,fairnas,SPOS,OFA,nayman2021hardcore}). 
Constructing the estimator this way allows making insights about the contribution of the design choices. It is important to note, that albeit its simplicity, our predictor's performance matches that of previously proposed predictors, that are typically expensive to compute and hard to optimize due to many hyper-parameters.
 
The predictor we propose has a bilinear form that allows formulating the latency constrained NAS problem as an \textit{Integer Quadratic Constrained Quadratic Programming} (IQCQP). Thanks to this, the optimization can be efficiently solved via a simple algorithm with some off-the-shelf components. The algorithm we suggest solves it within a few minutes on a common CPU.
 
Overall our optimization approach has two main performance related advantages. First, the outcome networks provide high accuracy and closely comply with the latency constraint. 
Second, the search is highly efficient, which makes our approach scalable to multiple target devices and latency demands. 

%% file: related.tex
\section{Related Work}
% \textbf{Efficient Neural Networks} are designed to meet the rising demand of deep learning models for numerous tasks per hardware constraints. Manually-crafted architectures such as MobileNets~\cite{howard2017mobilenets,sandler2018mobilenetv2} and ShuffleNet~\cite{zhang2018shufflenet} were designed for mobile devices, while TResNet~\cite{ridnik2020tresnet} and ResNesT~\cite{zhang2020resnest} are tailor-made for GPUs. Techniques for improving efficiency include pruning of redundant channels~\cite{dong2019network,aflalo2020knapsack} and layers~\cite{han2015learning}, model compression~\cite{han2015deep, he2018amc} and weight quantization methods~\cite{hubara2016binarized,umuroglu2017finn}. Dynamic neural networks adjust models based on their inputs to accelerate the inference, via gating modules~\cite{wang2018skipnet}, graph branching~\cite{huang2017multi} or dynamic channel selection~\cite{lin2017runtime}. These techniques are applied on  predefined architectures, hence cannot utilize or satisfy specific hardware constraints. 

\textbf{Neural Architecture Search} methods automate models' design per provided constraints. Early methods like NASNet~\cite{zoph2016neural} and AmoebaNet~\cite{real2019regularized} focused solely on accuracy, producing SotA classification models~\cite{huang2019gpipe} at the cost of GPU-years per search, with relatively large inference times. DARTS~\cite{liu2018darts} introduced a differential space for efficient search and reduced the training duration to days, followed by XNAS~\cite{nayman2019xnas} and ASAP~\cite{noy2020asap} that applied pruning-during-search techniques to further reduce it to hours.  

\textit{Predictor based} methods
recently have been proposed based on training a
model to predict the accuracy of an architecture just from an encoding of the
architecture. Popular choices for these models include Gaussian processes, neural networks, tree-based methods. See~\cite{lu2020nsganetv2} for such utilization and \cite{white2021powerful}
%,lu2020nsganetv2}
for comprehensive survey and comparisons. %In this work a much simpler quadratic accuracy predictor is proposed.

\textit{Interpretabe NAS} was firstly introduced by \cite{ru2021interpretable} through a rather elaborated Bayesian optimisation with Weisfeiler-Lehman kernel to identify beneficial topological features. % of architectures besides soley generating optimal ones. 
We propose an intuitive and simpler approach for NAS interpretibiliy for the efficient search space examined. This leads to more understanding  and applicable design rules.

\textit{Hardware-aware} methods such as ProxylessNAS~\cite{cai2018proxylessnas}, Mnasnet~\cite{tan2019mnasnet}, SPNASNet~\cite{stamoulis2019single}, FBNet~\cite{fbnet}, and TFNAS~\cite{TF-NAS} 
generate architectures that comply to the constraints by applying simple heuristics such as soft penalties on the loss function.
% included the latency of a target device in the architecture optimization,
% resulting with improved fit and accelerated inference on that device. \\
% leading to improved fit and accelerated inference over that device. \\
% HardCoRe-NAS~\cite{nayman2021hardcore} uses backpropagation~\cite{kelley1960backprop} over a supernetwork and is thus instant over CPU, compared to 7 GPU hours required by the latter and generates better networks. 
OFA~\cite{OFA} and HardCoRe-NAS~\cite{nayman2021hardcore} proposed a scalable approach across multiple devices by training an one-shot model~\cite{brock2017smash,bender2018understanding} once. 
This pretrained super-network is highly predictive
% This provides a strong pretrained super-network being highly predictive 
for the accuracy ranking of extracted sub-networks, e.g. FairNAS~\cite{SPOS}, SPOS~\cite{SPOS}.
OFA applies evolutionary search~\cite{real2019regularized} over a complicated multilayer perceptron (MLP)~\cite{rumelhart1985learning} based accuracy predictor with many hyperparameters to be tuned. 
HardCore-NAS searches by backpropagation~\cite{kelley1960backprop} over a supernetwork under strict latency constraints for several GPU hours per network.
Hence it requires access to a powerful GPU to perform the search and both approaches lack interpretability.
This work relies on such one-shot model, for intuitively building an interpretable and simple bilinear accuracy estimator that matches in performance without any tuning and optimized under strict latency constraints by solving an IQCQP problem in several CPU minutes following by a short fine-tuning.

%% file: method.tex
\section{Method} \label{sec:method}
\begin{figure*}[htb]
 \begin{minipage}{0.7\textwidth}
    \centering
         \includegraphics[width=\textwidth]{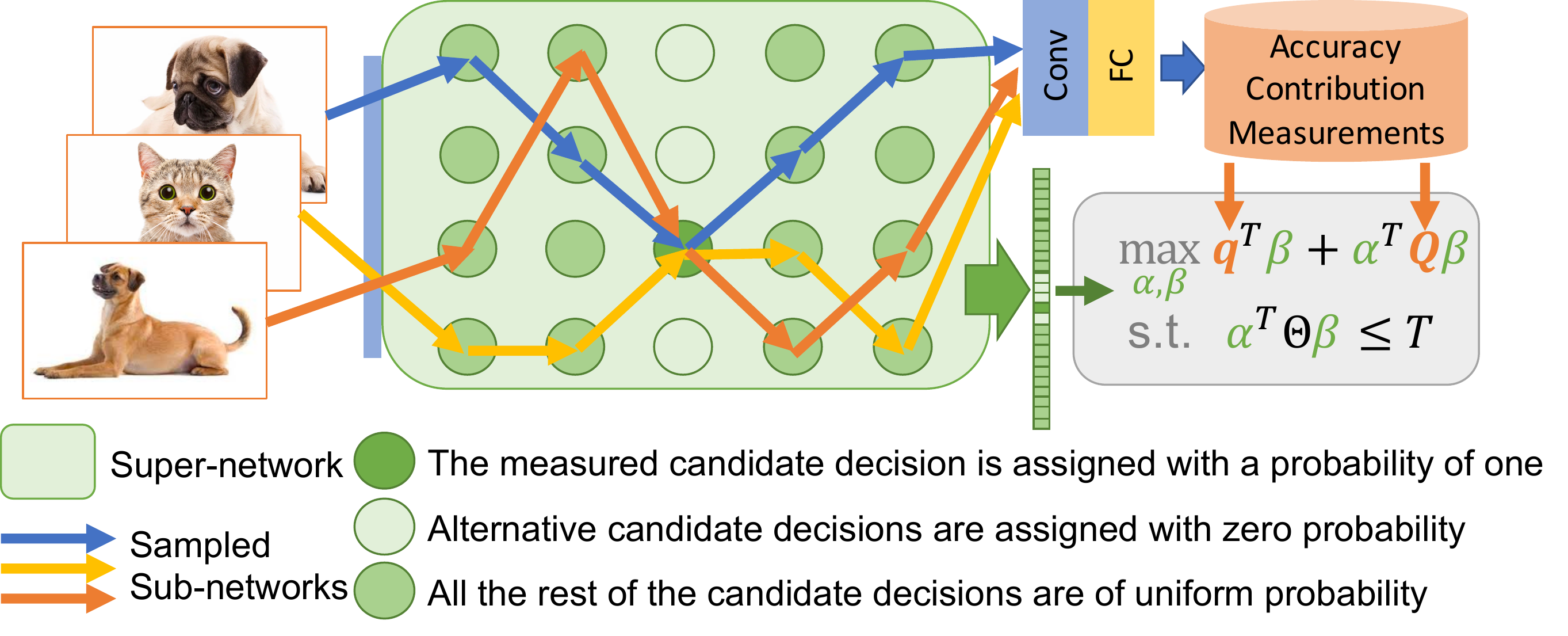}
  \end{minipage}
  \hfill
  \begin{minipage}{0.29\textwidth}
         \includegraphics[width=\textwidth]{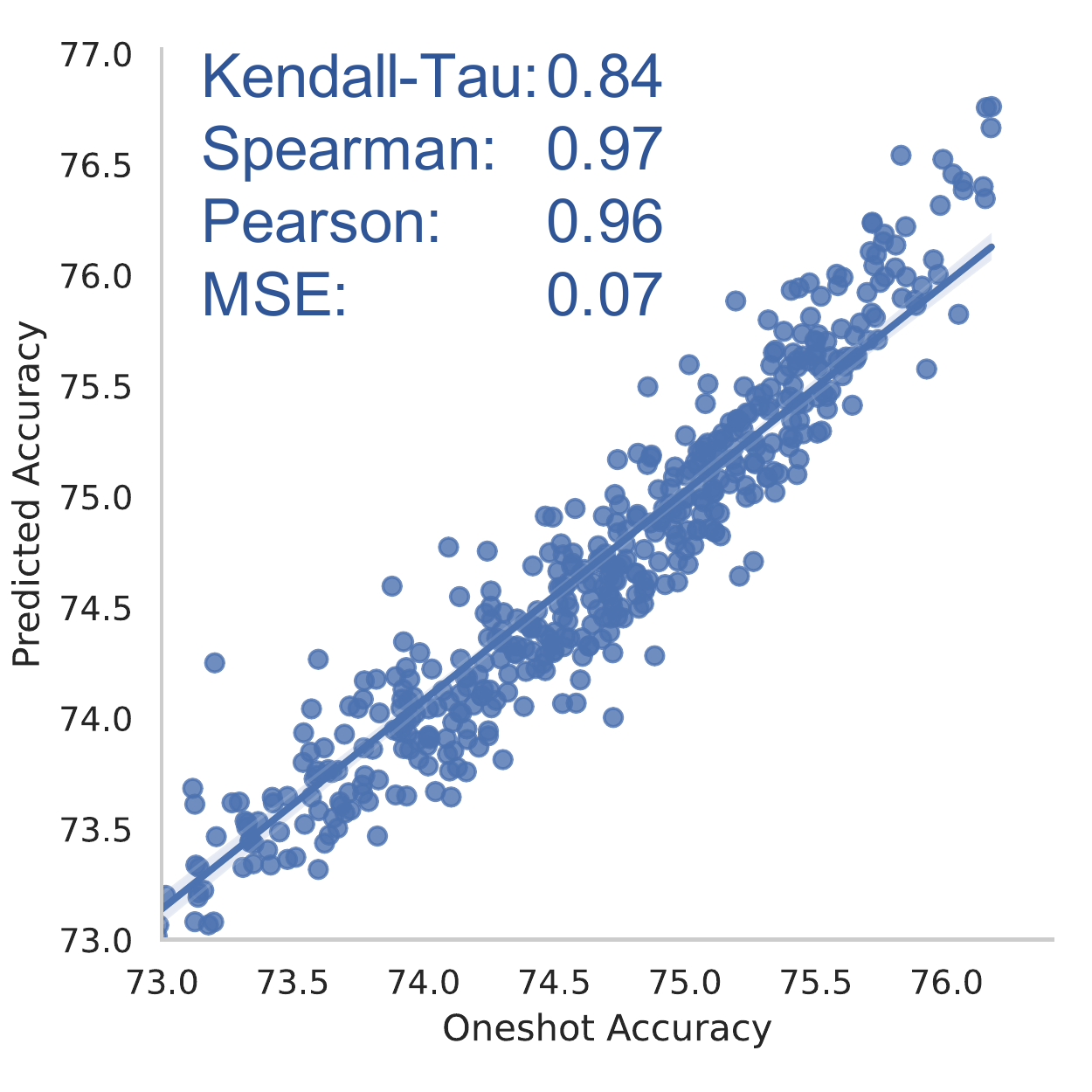}
    % \vspace{-5mm}
  \end{minipage}
    \captionof{figure}{(Left) The BINAS scheme constructs a bilinaer accuracy estimator by measuring the accuracy contribution of individual design choices and then maximizing this objective under bilinear latency constraints. (Right) Accuracy predictions vs measured accuracy of 500 subnetworks sampled uniformly at random. High ranking correlations are achieved.}
    \label{fig:scheme}
    \vspace{-3mm}
\end{figure*}

% \textcolor{red}{You need to start by describing at a high level the method. You cannot dive into the details before explaining the big picture of the proposed approach.
% Start from the standard formulation of the objective function. Only later you can show the IQCQP. BTW, you say that this is shown in Section 3.1, but it is not. In fact I couldn't find where it is shown.
% }

In this section we propose our method for latency-constrained NAS. We search for an architecture with the highest validation accuracy under a predefined latency constraint, denoted by $T$. 
We start by following~\cite{bender2018understanding}, training a supernetwork that accommodates the search space $\mathcal{S}$, as described in section~\ref{sec:search_space}. The supernetwork training ensures that the accuracy of subnetworks extracted from it, together with their corresponding weights, are properly ranked~\cite{SPOS,fairnas,OFA,nayman2021hardcore} as if those were trained from scratch. Given such a supernetwork, we sample subnetworks from it for estimating the individual accuracy contribution of each design choice  (section~\ref{sec:acc_contrib}), and construct a bilinear accuracy estimator for the expected accuracy of every possible subnetwork in the search space (section~\ref{sec:quad_estimator}). Finally, the latency of each possible block configuration is measured on the target device and aggregated to form a bilinear latency constraint. %, which is utilized 
Putting it all together (section~\ref{sec:formulate_IQCQP}) we formulate an IQCQP: 
% Our architecture search space $\mathcal{S}$ is parametrized by a vector $\zeta\in\mathcal{S}\subset\mathbbm{Z}^N$, governing the architecture structure, and $w$, the convolution weights.
% We show in Section~\ref{sec:search_space} that the latency-constrained NAS problem can be formulated as an IQCQP:
% \begin{align}
% \label{eqn:NAS_QCQP}
% \max_{\zeta}  
% ACC(\zeta)=q^T\zeta + \zeta^T Q\zeta
% \qquad\text{s.t. } 
% LAT(\zeta)= \zeta^T \Theta\zeta \leq T, \quad A_{\mathcal{S}}\cdot\zeta \leq b_{\mathcal{S}}, \quad \zeta\in \mathbbm{Z}^N
% \end{align}
\begin{align}
\label{eqn:NAS_QCQP}
&\max_{\zeta}  
ACC(\zeta)=q^T\zeta + \zeta^T Q\zeta
\\ &\text{s.t. } 
LAT(\zeta)= \zeta^T \Theta\zeta \leq T, \quad A_{\mathcal{S}}\cdot\zeta \leq b_{\mathcal{S}}, \quad \zeta\in \mathbbm{Z}^N
\notag
\end{align}
where $\zeta\!=\!(\balpha,\bbeta)\in\mathcal{S}\!\subset\!\mathbbm{Z}^N$, is the parametrization of the design choices in the search space that govern the architecture structure, $q\1\in\!\R^N$, $Q\!\in\!\R^{N\times N}$, $\Theta\!\in\!\R^{N\times N}$, $A_{\mathcal{S}}\!\in\!\R^{C \times N}$, $b_{\mathcal{S}}\!\in\!\R^C$ and $\zeta\in\mathcal{S}$ can be expressed as a set of $C$ linear equations. 
% Specifically, 
% We define the accuracy predictor $ACC(\zeta)$ in section~\ref{sec:quad_predictors} and adopt a bilinear formula for the latency computation $LAT(\zeta)$ in section~\ref{sec:quad_estimator}. 
Finally, in section~\ref{sec:solvers} we propose an optimization method to efficiently solve Problem~\ref{eqn:NAS_QCQP}.
Figure~\ref{fig:scheme} (Left) shows a high level illustration of the scheme.

\input{search_space}
% \subsection{Adapted Latency Constraint}
% State that we adopt the latency formula from ~\cite{nayman2021hardcore} with a correction term for high latency. Show the quadratic fit to the scatter plot of latency formula vs measured latency up to the highest latency in the search space.

% \vspace{-3mm}
\subsection{Estimating the Accuracy Contribution of Design Choices}\label{sec:acc_contrib}
 Next we introduce a simple way to estimate the accuracy contribution of each design choice, given a trained one-shot model. With this at hand, we will be able to select those design choices that contribute the most to the accuracy under some latency budget. Suppose a supernetwork is constructed, such that every sub-network of it resides in the search space of Section~\ref{sec:search_space}. The supernetwork is trained to rank well different subnetworks~\cite{SPOS,fairnas,OFA,nayman2021hardcore}. The expected accuracy of such a supernetwork $\mathbb{E}[Acc]$ is estimated by uniformly sampling a different subnetwork for each input image from 20\% of the Imagenet train set (considered as a validation set), as illustrated in Figure~\ref{fig:estimate_base}.
\begin{figure}[htb]
    \centering
     \includegraphics[width=.6\textwidth]{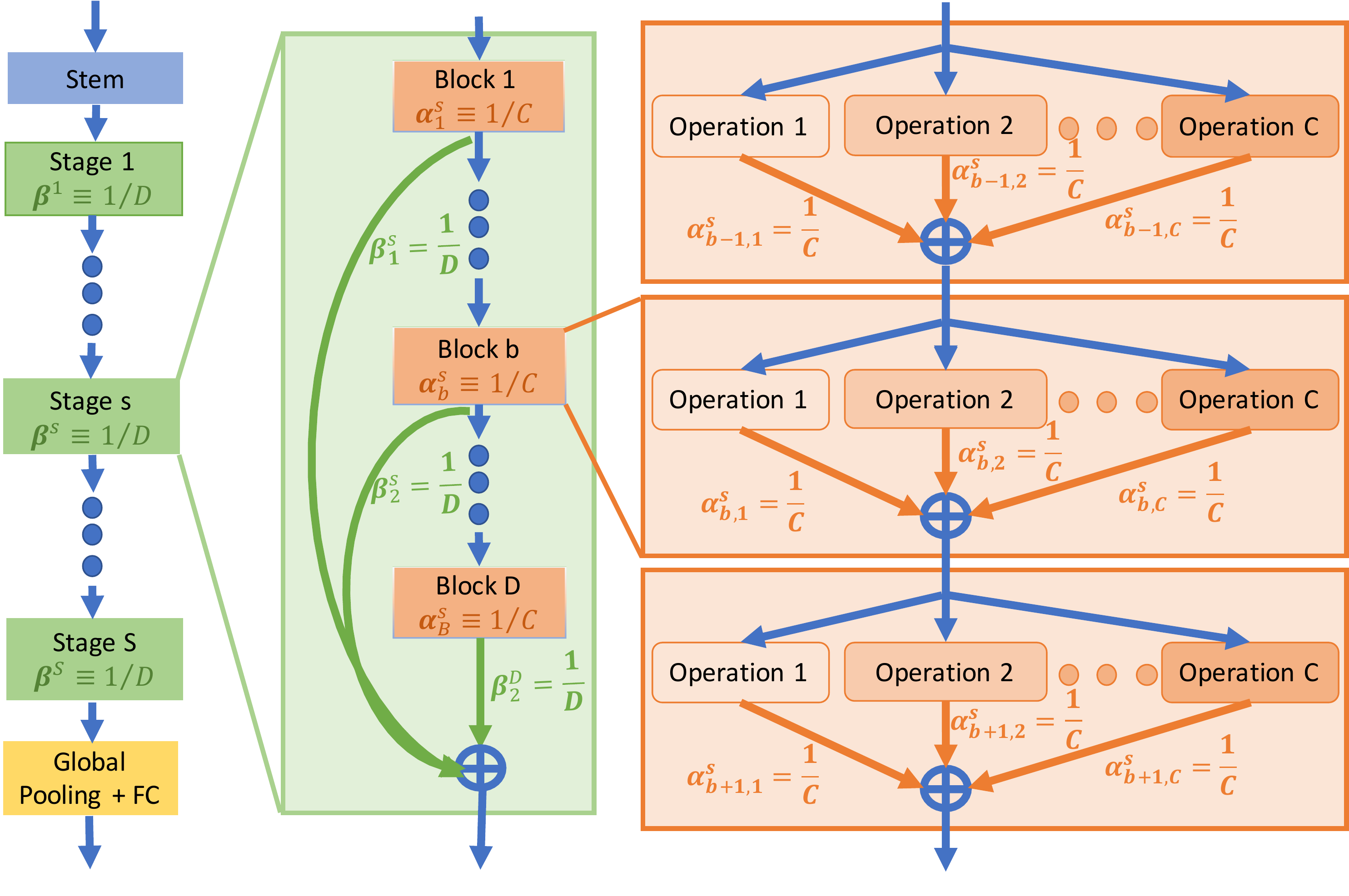}
     \caption{Estimating the base accuracy of the supernetwork is done by random uniform sampling of subnetworks for each input image.
    %  All accuracy contributions are with respect to this base accuracy.
     }
     \vspace{4mm}
     \label{fig:estimate_base}
\end{figure}
This estimate serves as the base accuracy of the supernetwork and thus every design choice should be evaluated by its contribution on top of it. Hence, the individual accuracy contribution of setting the depth of stage $s$ to $b$ is the gap: $\Delta_b^s=\mathbb{E}[Acc | d^s=b]-\mathbb{E}[Acc]$, where the first expectation is estimated by setting $\beta^s_b=1$ and uniform distribution for the rest of the design choices. This is done for every possible depth of every stage (Figure~\ref{fig:estimate_beta}).
\begin{figure}[htb]
    \centering
     \includegraphics[width=1\textwidth]{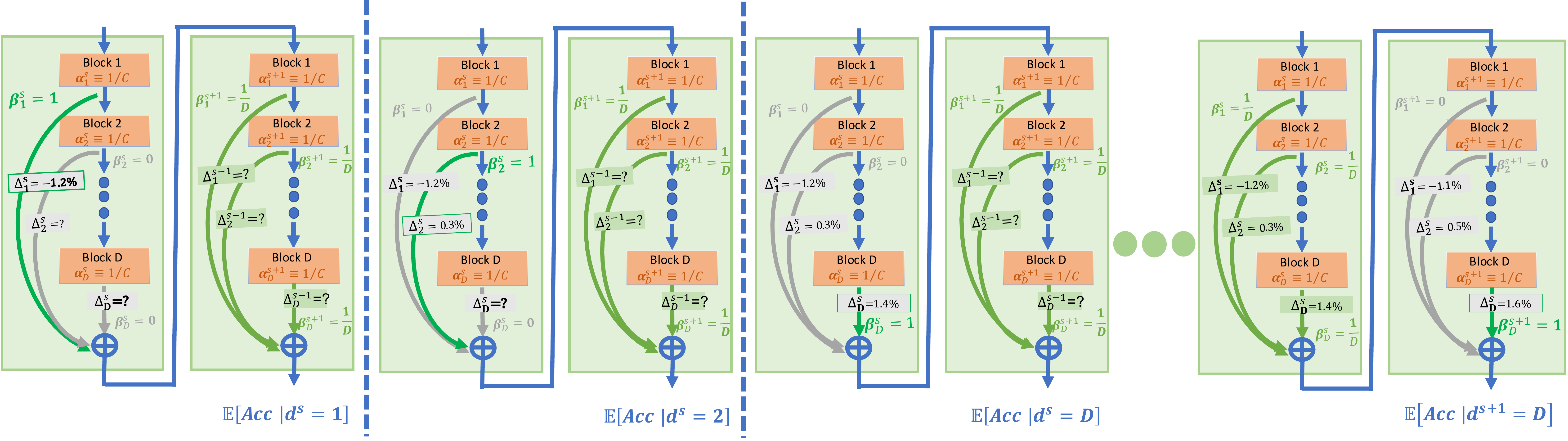}
     \caption{Estimating the expected accuracy gap caused by macroscopic design choices of the depth of the stages.}
     \label{fig:estimate_beta}
     \vspace{5mm}
\end{figure}

Similarly, the individual accuracy contribution of choosing configuration $c$ in block $b$ of stage $s$ is given by the gap: $\Delta_{b,c}^s=\mathbb{E}[Acc | O_{b,c}^s=O_c, d^s=b]-\mathbb{E}[Acc]$, where the first expectation is estimated by setting $\alpha_{b,c}^s=1$ and $\beta^s_b=1$, while keeping a uniform distribution for all the rest of the design choices in the supernetwork, as illustrated in Figure~\ref{fig:estimate_alpha}.
\begin{figure}[htb]
    \centering
     \includegraphics[width=1\textwidth]{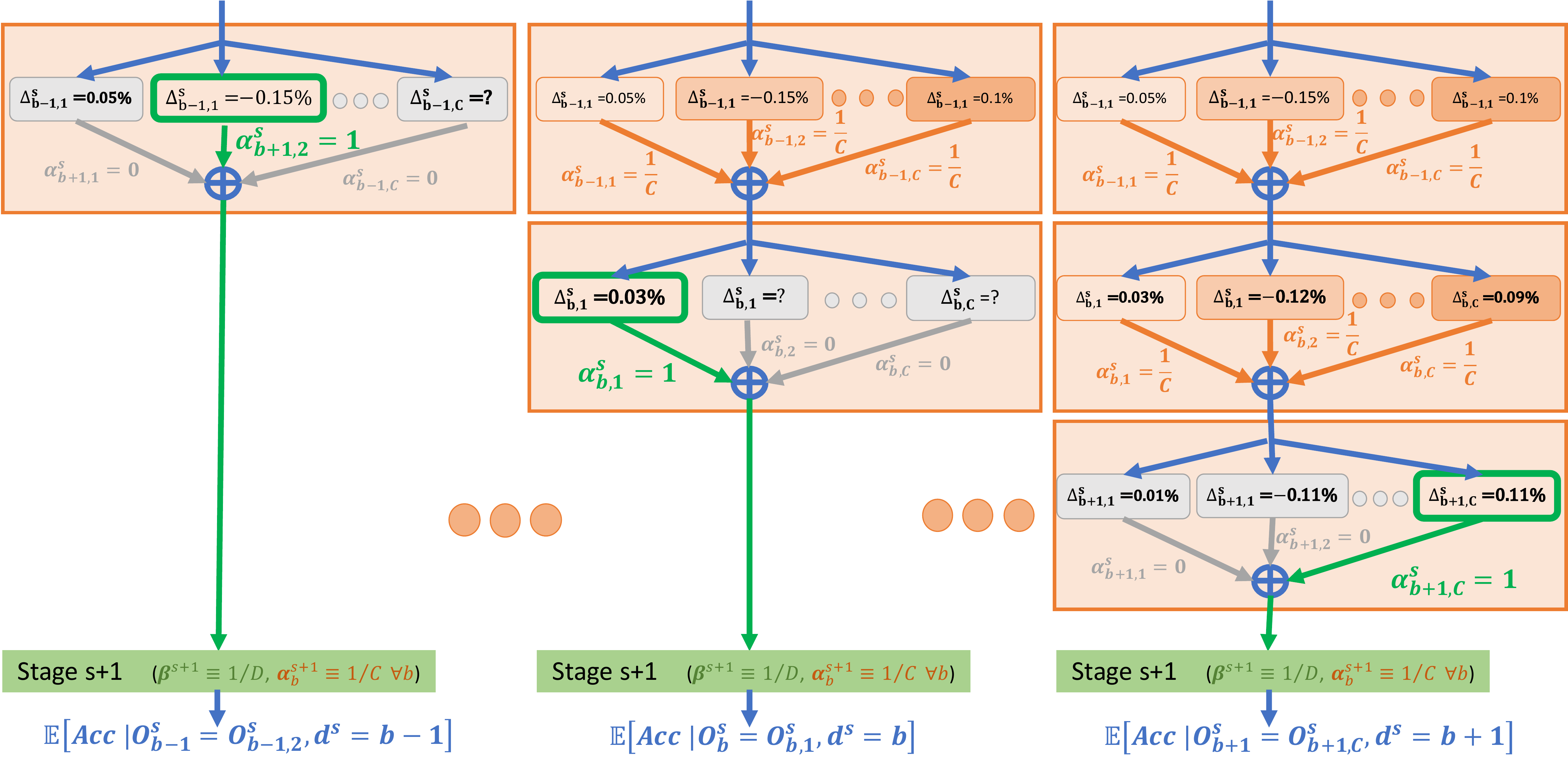}
     \caption{Estimating the expected accuracy gap caused by microscopic design choices on the operation applied at every block one at a time.}
     \label{fig:estimate_alpha}
\end{figure}

\subsection{Constructing a Bilnear Accuracy Estimator} \label{sec:quad_estimator}
% \begin{itemize}
%     \item Explain the formula and how it is derived by sampling a given supernetwork
%     \item Briefly explain the multipath sampling of subnetworks from \cite{nayman2021hardcore} using the Gumbel Softmax Trick 
%     \item Show the benefits of multipath over singlepath by the graphs of ranking correlation vs validation epochs per sample
%     \item Show the contribution of different terms of the estimator's formula
%     \item Theorem for the derived estimator
% \end{itemize}

% \textcolor{red}{Here you start talking about sub-networks - it is not clear what are these sub-networks and where they come from. You did no explain that you train a super-network and then sample from it sub-networks. This explanation needs to come before section 3.1
% }
We next propose an intuitive and effective way to utilize the estimated individual accuracy contribution of each design choice (section~\ref{sec:acc_contrib}) to estimate the expected overall accuracy of a certain architecture in the search space.

The expected accuracy contribution of a block $b$ can be computed by summing over the accuracy contributions $\Delta^s_{b,c}$ of every possible configuration $c\in\mathcal{C}$:
$\bar{\delta}^s_{b}=\Sigma_{c\in\mathcal{C}} \alpha^s_{b,c} \cdot \Delta^s_{b,c}$
Thus the expected accuracy contribution of the stage $s$ of depth $b'$ is 
$\delta^s_{b'}=\Delta^s_{b'} + \Sigma_{b=1}^{b'}\bar{\delta}^s_{b}$, where the first term is the accuracy contribution associated solely with the choice of depth $b'$ and the second term is the aggregation of the expected accuracy contributions of the first $b'$ blocks in the stage.
% \begin{align}\label{eqn_iclr:latency_correction}
% \ell^s_{b'}=\sum_{b=1}^{b'}\bar{\ell}^s_{b}
% \end{align}
Taking the expectation over all possible depths for stage $s$ yields 
$\delta^s= \sum_{b'=1}^D \beta^s_{b'}\cdot \delta^s_{b'}$
and summing over all the stages results in the aggregated accuracy contribution of all design choices. 
%   $\sum_{s=1}^S \sum_{b'=1}^D\beta^s_{b'}\cdot\Delta^s_{b'} + \sum_{s=1}^S \sum_{b'=1}^D \sum_{b=1}^{b'}\sum_{c\in\mathcal{C}} \alpha^s_{b,c} \cdot \Delta^s_{b,c} \cdot \beta^s_{b'}$.
Hence the accuracy of a subnetwork can be calculated by this estimated contribution on top of the estimated base accuracy $\mathbb{E}[Acc]$ (Figure~\ref{fig:estimate_base}):
\begin{align}\label{eqn:estimator}
%   ACC(\zeta)=
    ACC(\balpha, \bbeta)=
%   E[Acc|\cap_{s=1}^S\cap_{b=1}^D O^s_b, \cap_{s=1}^S d_s]  
   \mathbb{E}[Acc] &+
    %   \sum_{s=1}^S \sum_{c\in\mathcal{C} \alpha^s_{1,c}\cdot \Delta_{1,c}^s
    % \left(
    % E[A | O^s_{b=1}=O_c]
    % -
    % E[A]\right)
%  \\&+ 
    % +
      \sum_{s=1}^S \sum_{b=1}^D \beta^s_b\cdot 
    \Delta_b^s
    % \left(E[A\mid d_s=b] 
    % -
    % E[A]\right)
% \\ \notag &
+ 
      \sum_{s=1}^S \sum_{b=1}^D \sum_{b'=b}^{D}\sum_{c\in\mathcal{C}} \alpha^s_{b,c}\cdot 
    \Delta_{b,c}^s
    % \left(
    % E[A | O^s_b=O_c]
    % -
    % E[A ]
    % \right)
    \cdot\beta^s_{b'}
\end{align}
% The expectations are with respect to a uniform sampling of sub-networks $\zeta=(\balpha,\bbeta)\in\mathcal{S}$ excluding decisions specified in the conditional events, as illustrated in Figure~\ref{fig:scheme}. 
% In practice, while \eqref{eqn:estimator} is quadratic in $\zeta$, 
And its vectorized form can be expressed as the following bilinear formula in $\balpha$ and $\bbeta$: $ACC(\balpha, \bbeta) = r + q_\beta^T\bbeta + \balpha^T Q_{\alpha\beta}\bbeta$, 
% \begin{align}\label{eqn:vectorized_estimator}
% %   ACC(\zeta)=
%   ACC(\balpha, \bbeta) = r + q_\beta^T\bbeta + \balpha^T Q_{\alpha\beta}\bbeta
% \end{align}
where $r=E[Acc]$, $q_\beta\in\R^{D\cdot S}$ is a vector composed of $\Delta_b^s$ and $Q_{\alpha\beta}\in\R^{\mathcal{C}\cdot D\cdot S\times D\cdot S}$ is a matrix composed of $\Delta_{b,c}^s$.

We next present a theorem (with proof in Appendix~\ref{apdx:proof_estimator}) that states that the estimator in~\eqref{eqn:estimator} approximates well the expected accuracy of an architecture.%: $  \mathbb{E}[Acc|\cap_{s=1}^S\cap_{b=1}^D O^s_b, \cap_{s=1}^S d^s] $.
\begin{theo}\label{theo:estimator}
Assume $\{O^s_b, d_s\}$ for $s=1,\dots,S$ and $b=1,\dots,D$ are conditionally independent with the accuracy $Acc$. Suppose that there exists a positive real number $0 < \epsilon\ll 1$ such that for any $X\in\{O^s_b, d_s\}$ the following holds $|\mathbb{P}[Acc|X]- \mathbb{P}[Acc]|<\epsilon\mathbb{P}[Acc]$.
% $$
% \left|\frac{\mathbb{P}[Acc|X]}{\mathbb{P}[Acc]}-1\right|<\epsilon
% $$
Then:
\begin{align}
%   ACC(\zeta)=ACC(\balpha, \bbeta)=
  \mathbb{E}\left[Acc \middle\vert   
    \cap_{s=1}^S\cap_{b=1}^D O^s_b, 
    \cap_{s=1}^S d^s  
%   \begin{array}{c}
%     \cap_{s=1}^S\cap_{b=1}^D O^s_b  \\
%     \cap_{s=1}^S d^s  
%   \end{array}
  \right]  
     &= \mathbb{E}[Acc] 
%   +
    %   \sum_{s=1}^S \sum_{c\in\mathcal{C} \alpha^s_{1,c}\cdot \Delta_{1,c}^s
    % \left(
    % E[A | O^s_{b=1}=O_c]
    % -
    % E[A]\right)
%  \\ \notag &+ 
    +
    \left(1
    +\mathcal{O}(N\epsilon)\right)\cdot
      \sum_{s=1}^S \sum_{b=1}^D \beta^s_b\cdot 
    \Delta_b^s
    % \left(E[A\mid d_s=b] 
    % -
    % E[A]\right)
\\ \notag &+ 
    % +
    \left(1
    +\mathcal{O}(N\epsilon)\right)\cdot
      \sum_{s=1}^S \sum_{b=1}^D \sum_{b'=b}^{D}\sum_{c\in\mathcal{C}} \alpha^s_{b,c}\cdot 
    \Delta_{b,c}^s
    % \left(
    % E[A | O^s_b=O_c]
    % -
    % E[A ]
    % \right)
    \cdot\beta^s_{b'}
\end{align}
\end{theo}
% Refer to appendix~\ref{apdx:proof_estimator} for the proof.

% \begin{wrapfigure}{R}{0.38\textwidth}
%   \begin{center}
%      \includegraphics[width=0.36\textwidth]{images/scatter_kt.pdf}
%   \end{center}
%     \vspace{-5mm}
%     \caption{Subnetworks' accuracy predictions by the proposed estimator vs measurements}
%     \label{fig:predicted_vs_oneshot}
% \end{wrapfigure}
Theorem~\ref{theo:estimator} and Figure~\ref{fig:scheme} (right) demonstrate the effectiveness of relying on $\Delta_{b,c}^s$, $\Delta_b^s$ to express the expected accuracy of networks. Since those terms measure the accuracy contributions of individual design decisions, many insights and design rules can be extracted from those, as discussed in section~\ref{sec:interpretability}, making the proposed estimator intuitively interpretable.
Furthermore, the transitivity of ranking correlations is used in appendix~\ref{apdx:kt_transitivity} for guaranteeing good prediction performance with respect to architectures trained from scratch.

\subsection{The Interger Quadratic Constraints Quadratic Program}
\label{sec:formulate_IQCQP}
In this section we formulate latency-constrained NAS as an IQCQP of bilinear objective function and bilinear constraints. For the purpose of maximizing the validation accuracy of the selected subnetwork under latency constraint, we utilize the bilinear accuracy estimator derived in section~\ref{sec:quad_estimator} as the objective function and define the bilinear latency constraint  similarly to~
\cite{TF-NAS,nayman2021hardcore}:
\begin{align}\label{eqn:latency_formula}
   LAT(\balpha, \bbeta) = 
   \sum_{s=1}^S \sum_{b=1}^D \sum_{b'=b}^{D}\sum_{c\in\mathcal{C}} \alpha^s_{b,c} \cdot t^s_{b,c} \cdot \beta^s_{b'} = \balpha^T \Theta\bbeta
\end{align}
where $\Theta\in\R^{\mathcal{C}\cdot D\cdot S\times D\cdot S}$ is a matrix composed of the latency measurements $t^s_{b,c}$  on the target device of each configuration $c\in\mathcal{C}$ of every block $b$ in every stage $s$ (Figure~\ref{fig:search_space}). 
The bilinear version of problem~\ref{eqn:NAS_QCQP} turns to be:
\begin{align}
\label{eqn:bilinear_QCQP}
&\max_{\alpha^s_{b,c},\beta^s_b}  &&
% q_\beta^T\bbeta + \balpha^T Q_{\alpha\beta}\bbeta
% = 
\sum_{s=1}^S \sum_{b=1}^D \beta^s_b\cdot 
\Delta_b^s
+ 
\sum_{s=1}^S \sum_{b=1}^D \sum_{b'=b}^{D}\sum_{c\in\mathcal{C}} \alpha^s_{b,c}\cdot 
\Delta_{b,c}^s    
\cdot\beta^s_{b'}
\\ &\text{s.t. } && 
\notag
% \balpha^T \Theta\bbeta 
% =
\sum_{s=1}^S \sum_{b=1}^D \sum_{b'=b}^{D}\sum_{c\in\mathcal{C}} \alpha^s_{b,c} \cdot t^s_{b,c} \cdot \beta^s_{b'} 
\leq T
\\ &&&
\notag
\Sigma_{c\in\mathcal{C}} \alpha^s_{b,c} = 1
\,;\,
\alpha^s_{b,c}\in \{0, 1\}^{\left|\mathcal{C}\right|}
\,\,
\forall s\in\{1,..,S\}, b\in\{1,..,D\},c\in\mathcal{C}
\\ &&&
\notag
\Sigma_{b=1}^D \beta^s_b = 1 
\,; \, 
\beta^s_b\in \{0, 1\}^D 
\, \,  
\forall s\in\{1,..,S\}, b\in\{1,..,D\}
\end{align}
And in its vectorized form:
\begin{align}
\label{eqn:bilinear_QCQP_vector}
\max_{\balpha,\bbeta\in\{0,1\}} 
q_\beta^T\bbeta + \balpha^T Q_{\alpha\beta}\bbeta
\text{    s.t.    }
\balpha^T \Theta\bbeta \leq T
\, ; \, 
A_{\mathcal{S}}^\alpha\cdot \balpha\leq b_{\mathcal{S}}^\alpha 
\, ; \, 
A_{\mathcal{S}}^\beta\cdot \bbeta\leq b_{\mathcal{S}}^\beta 
% \, ; \, \balpha,\bbeta\in\{0,1\}
\end{align}

% \subsection{Solving the Integer Quadratic Constraints Quadratic Program}\label{sec:solvers}
% \vspace{-3mm}
\subsection{Solving the Integer Quadratic Constraints Quadratic Program}\label{sec:solvers}
% \vspace{-2mm}
% \begin{wrapfigure}{R}{0.5\textwidth}
%   \begin{center}
%       \input{estimator_error_bars}
%   \end{center}
% \end{wrapfigure}
By formulating the latency-constrained NAS as a binary problem in section~\ref{sec:formulate_IQCQP},
% quadratic objective function of~\eqref{eqn:vectorized_estimator}, the quadratic latency constraint of~\eqref{eqn:latency_formula} and the integer (binary) linear constraints of~\eqref{eqn:integer_search_space} that specify the search space, 
we can now use out-of-the-box \textit{Mixed Integer Quadratic Constraints Programming} (MIQCP) solvers to optimize problem~\ref{eqn:bilinear_QCQP}. We use IBM CPLEX~\cite{CPLEX} that supports non-convex binary QCQP and utilizes the Branch-and-Cut algorithm~\cite{padberg1991branch} for this purpose. 
% \newline
% \newline
A heuristic alternative for optimizing an objective function 
% beyond quadratic, e.g., of section~\ref{sec:beyond_quad}, 
under integer constraints is evolutionary search~\cite{real2019regularized}. 
% \newline
% \newline
Next we propose a more theoretically sound alternative.

\vspace{-3mm}
% \subsubsection{Utilizing the Block Coordinate Frank-Wolfe Algorithm}
\subsubsection{Utilizing the Block Coordinate Frank-Wolfe Algorithm} 

As pointed out by \cite{nayman2021hardcore}, since $\Theta$ is constructed from measured latency in \eqref{eqn:latency_formula}, it is not guaranteed to be positive semi-definite, hence, the induced quadratic constraint makes the feasible domain in problem \ref{eqn:NAS_QCQP} non-convex in general.
To overcome this we adapt the \textit{Block-Coordinate Frank-Wolfe} (BCFW)~\cite{BCFW} for solving a continuous relaxation of problem~\ref{eqn:NAS_QCQP}, such that $\zeta\in \R_+^N$.
Essentially BCFW adopts the Frank-Wolfe~\cite{frank_wolfe} update rule for each block of coordinates in $\delta\in\{\balpha,\bbeta\}$ picked up at random at each iteration $k$, such that $\delta_{k+1} = (1-\gamma_k)\cdot\delta_k + \gamma_k\cdot\hat{\delta}$ with $0\leq\gamma_k\leq 1$, for any partially differentiable objective function $ACC(\balpha,\bbeta)$:
\begin{align}
  \hat{\balpha}=&\argmax_{\balpha} \nabla_{\balpha}ACC(\balpha,\bbeta_k)^T\cdot\balpha 
%   \\\notag &
  \quad
  \text{ s.t. } \bbeta_k^T \Theta^T \cdot \balpha\leq T \quad ; \quad
  A_{\mathcal{S}}^\alpha\cdot\balpha\leq b_{\mathcal{S}}^\alpha
%   A_{\mathcal{S}}\cdot(\balpha, \bbeta_k)\leq b_{\mathcal{S}}
   \\
    \hat{\bbeta}=&\argmax_{\bbeta} \nabla_{\bbeta}ACC(\balpha_k,\bbeta)^T\cdot\bbeta 
    % \\ \notag &
    \quad
    \text{ s.t. }  \balpha_k^T \Theta \cdot \bbeta\leq T \quad; \quad
    A_{\mathcal{S}}^\beta\cdot \bbeta\leq b_{\mathcal{S}}^\beta
    % A_{\mathcal{S}}\cdot(\balpha_k, \bbeta)\leq b_{\mathcal{S}}
    % \vspace{-2mm}
\end{align}
where $\nabla_\delta$ stands for the partial derivatives with respect to $\delta$. 
% This applies for both the differentiable MLP predictor of section \ref{sec:beyond_quad} and the quadratic one of \eqref{eqn:regression_quad}. 
Convergence guarantees are provided in \cite{BCFW}. Then, once converged to the solution of the continuous relaxation of the problem, we need to project the solution back to the discrete space of architectures, specified in \eqref{eqn:integer_search_space}, as done in \cite{nayman2021hardcore}. This step could deviate from the solution and cause degradation in performance.

Due to the formulation of the NAS problem~\ref{eqn:bilinear_QCQP_vector} as a \textit{Bilinear Programming} (BLP) \cite{gallo1977bilinear} with bilinear constraints (BLCP) we design Algorithm~\ref{alg:BCFW_QCQP} that applies the BCFW with line-search for this special case. Thus 
% as Leveraging the specific bilinear forms in problem~\ref{eqn:bilinear_QCQP},  Algorithm~\ref{alg:BCFW_QCQP} applies the BCFW with line-search for the special case of using the bilinear objective function specified in section~\ref{sec:quad_estimator}. Together with the bilinear constraints of \eqref{eqn:latency_formula}, the resulting problem is a \textit{Bilinear Programming} (BLP) \cite{gallo1977bilinear} with bilinear constraints, i.e., BLCP. For this case, 
more specific convergence guarantees can be provided together with the sparsity of the solution, hence no additional discretization step is required. 
The following theorem states that after $\mathcal{O}(1/\epsilon)$ iterations, Algorithm~\ref{alg:BCFW_QCQP} obtains an $\epsilon$-approximate solution to problem \ref{eqn:bilinear_QCQP_vector}. 
\vspace{3mm}
\input{fw_algo}
 \begin{theo}\label{theo:convergence}
    For each $k>0$ the iterate $\zeta_k=(\balpha_k, \bbeta_k)$ of Algorithm~\ref{alg:BCFW_QCQP} satisfies:
    $$E[ACC(\zeta_k)] - ACC(\zeta^*) \leq \frac{4}{k+4}\left(ACC(\zeta_0) - ACC(\zeta^*)\right)$$
    where $\zeta^*=(\balpha^*, \bbeta^*)$ is the solution of a continuous relaxation of problem \ref{eqn:bilinear_QCQP_vector} and the expectation is over the random choice of the block $\balpha$ or $\bbeta$. 
    % Furthermore, there exists an iterate $0\leq \hat{k}\leq K$ of Algorithm~\ref{alg:BCFW_QCQP}  with a duality gap bounded by $E[g(\zeta_{\hat{k}})]\leq \frac{12}{K+1}\left(ACC(\zeta_0) - ACC(\zeta^*)\right)$.
    \end{theo}
The proof is in Appendix~\ref{apdx:proof_convergence_bcfw}. 
We next provide a guarantee that Algorithm~\ref{alg:BCFW_QCQP} directly yields a \textit{sparse solution}, representing a valid sub-network
% \footnote{Except from a single probability vector from those composing $\balpha$ and $\bbeta$, which contains up to two non-zero entries each, as all the rest are one-hot vectors} 
without the additional discretization step required by other continuous methods \cite{liu2018darts,fbnet,TF-NAS,nayman2021hardcore}.
% up to a single probability vector from those composing $\balpha$ and $\bbeta$, which contains up to two non-zero entries each, as all the rest are one-hot vectors. 
% Hence, no further discretization step is required.
\begin{theo}
\label{thm:one_hot_sol}
The output solution $(\balpha, \bbeta)=\zeta^*$ of Algorithm~\ref{alg:BCFW_QCQP} contains only one-hot vectors for $\alpha^s_{b,c}$ and $\beta^s_b$, except from a single one for each of those blocks, which contains a couple of non-zero entries.
% admits:
% \begin{align}\notag
%   &\sum_{c\in\mathcal{C}}|\alpha^s_{b,c}|^0 = 1 \,\forall 
%   (s, b)\in\{1,..,S\}\otimes \{1,..,D\}\setminus{\{(s_{\balpha}, b_{\balpha})\}}
%   \\ \notag
%   &\sum_{b=1}^D|\beta^s_b|^0 = 1 \quad \forall s\in\{1,..,S\}\setminus{\{s_{\bbeta}\}}
% \end{align}
% where $|\cdot|^0=\mathds{1}\{\cdot>0\}$ and $(s_{\balpha}$, $b_{\balpha})$, $s_{\bbeta}$  are single block and stage respectively, satisfying:
% \begin{align}\label{eqn:dense_remiders}
%   \sum_{c\in\mathcal{C}}|\alpha^{s_{\balpha}}_{{b_{\balpha}},c}|^0 \leq 2
%   \quad ; \quad 
%   \sum_{b=1}^D|\beta^{s_{\bbeta}}_b|^0 \leq 2
% \end{align}
\end{theo}
The proof is in Appendix~\ref{apdx:proof_sparsity_bcfw}.
In practice, a negligible latency deviation is associated with taking the $\argmax$ over the only two couples. Differently from the sparsity guarantee for the discretization step in \cite{nayman2021hardcore}, Theorem~\ref{thm:one_hot_sol} guarantees similar desirable properties but for the fundamentally different Algorithm~\ref{alg:BCFW_QCQP} that does not involve an additional designated projection step.%referred to in~\eqref{eqn:dense_remiders}. %

%% file: search_space.tex
\vspace{-3mm}
\subsection{The Search Space}
\label{sec:search_space}
% \begin{figure}
%     \vspace{-12mm}
%     \centering
%          \includegraphics[width=0.45\textwidth]{images/search_space.pdf}
%     \caption{Search space via the one-shot model}
%     \label{fig:super_net} 
% \end{figure}

% Aiming at latency efficient architectures, we adopt the search space introduced in~\cite{nayman2021hardcore}% and illustrated in Figure~\ref{fig:super_net}
% , which is closely related to those used by~\cite{fbnet,mobilenetv3,tan2019mnasnet,TF-NAS,OFA}. 
We consider a general search space that integrates a macro search space and a micro search space. 
The macro search space is composed of $S$ stages $s\in\{1,..,S\}$ of different input resolutions, each composed of blocks $b\in\{1,..,D\}$ with the same input resolution, and defines how the blocks are connected, see Figure~\ref{fig:search_space}. 
The micro search space 
% is based on \textit{Mobilenet Inverted Residual} (MBInvRes) blocks~\cite{mobilenetv2} and 
controls the internal structures of each block. 
Specifically, this search space includes latency efficient search spaces introduced in~\cite{fbnet,mobilenetv3,tan2019mnasnet,TF-NAS,OFA,nayman2021hardcore}.
% \begin{wrapfigure}{R}{0.7\textwidth}
\begin{figure}[htb]
    \centering
     \includegraphics[width=.7\textwidth]{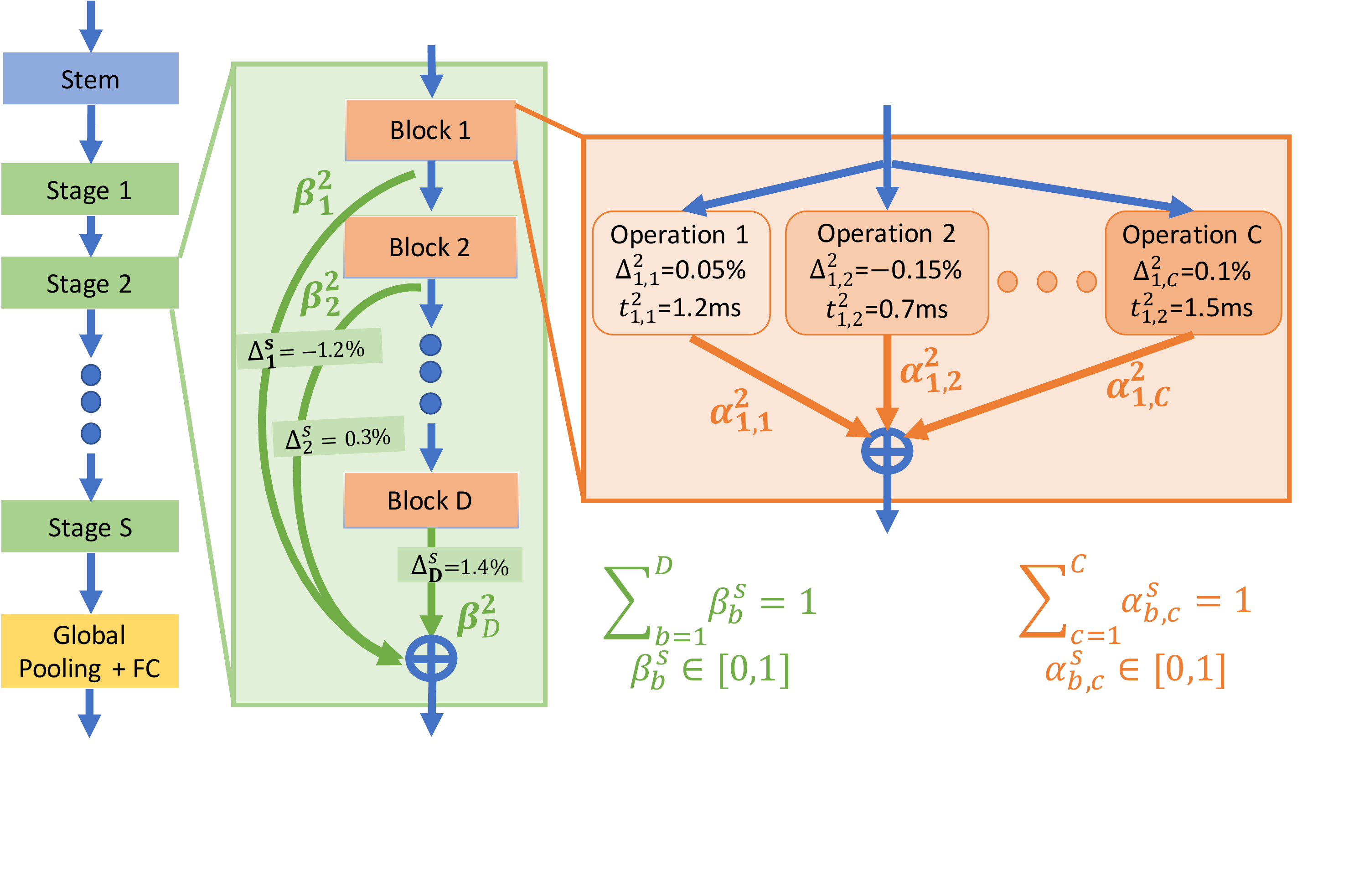}
     \vspace{-30pt}
     \caption{BINAS search space. The individual accuracy contribution and latency of each configured operation are measured.}
     \vspace{-15pt}
     \label{fig:search_space}
\end{figure}

% For a better readability we repeat the details next.
% Every MBInvRes block is configured by an expansion ratio $er\in\{3,4,6\}$ of the point-wise convolution, kernel size $k\in\{3\times3, 5\times5\}$ of the Depth-Wise Separable convolution (DWS), and Squeeze-and-Excitation (SE) layer~\cite{SE} $se\in\{\text{on}, \text{off}\}$ as
% shown at the bottom of Figure~\ref{fig:super_net} and 
% detailed in Appendix~\ref{apdx:search_space}.

%
% Table~\ref{tab:configurations}.
%  \begin{minipage}{\textwidth}
%   \begin{minipage}[b]{0.49\textwidth}
%   \end{minipage}
%   \hfill
%   \begin{minipage}[b]{0.49\textwidth}
%     \centering
%       \begin{tabular}{|c||c|c|c|}
%     \hline
%     c & er & k & se \\
%     \hline
%     1 & 3 & $3\times 3$ & off \\
%     2 & 3 & $5\times 5$ & on \\
%     3 & 3 & $3\times 3$ & off \\
%     4 & 3 & $5\times 5$ & on \\
%     5 & 4 & $3\times 3$ & off \\
%     6 & 4 & $5\times 5$ & on \\
%     7 & 4 & $3\times 3$ & off \\
%     8 & 4 & $5\times 5$ & on \\
%     9 & 6 & $3\times 3$ & off \\
%     10 & 6 & $5\times 5$ & on \\
%     11 & 6 & $3\times 3$ & off \\
%     12 & 6 & $5\times 5$ & on \\ \hline
%     \end{tabular}
%     \label{tab:block_config}
%     \captionof{table}{Specifications for each indexed configuration $c\in\mathcal{C}$. The configurations are indexed according to their expected latency.}
%     \end{minipage}
%     \vspace{5mm}
%   \end{minipage}
%
%
% Each joint configuration $(er, k, se)$ implies 
A block configuration $c\!\in\!\mathcal{C}$ (specified in Appendix~\ref{apdx:search_space}) corresponds to parameters $\balpha$.
 For each block $b$ of stage $s$ we have $\alpha^s_{b,c}\in \{0, 1\}^{\left|\mathcal{C}\right|}$ and $\Sigma_{c\in\mathcal{C}} \alpha^s_{b,c} = 1$.
An input feature map $x^s_b$ to block $b$ of stage $s$ is processed as follows:
$x^s_{b+1}=\sum_{c\in\mathcal{C}}\alpha^s_{b,c}\cdot O^s_{b,c}(x^s_b)$, 
where $O^s_{b,c}(\cdot)$ is the operation configured by $c\in\mathcal{C}$. %according to $c=(er,k,se)$. 
% The output of each block of every stage is also directed to the end of the stage. 
% as illustrated in the center of Figure~\ref{fig:super_net}. 
The depth of each stage $s$ is controlled by the parameters $\bbeta$:
$x^{s+1}_1=\Sigma_{b=1}^{D}\beta^s_b\cdot x^s_{b+1}$,
% $\bbeta\in\mathcal{B}=\bigotimes_{s=1}^S\bigotimes_{b=1}^D \beta^s_b$, 
such that $\beta^s_b\in \{0, 1\}^D$ and $\Sigma_{b=1}^D \beta^s_b = 1$.
% The depth is $d^s\!\in\!\left\{b\mid\beta^s_b=1, b\in\{1,..,D\}\right\}$, since 

To summarize, the search space is composed of both the micro and macro search spaces parameterized by $\balpha$ and $\bbeta$, respectively:
% , such that for all $s\in\{1,..,S\}, b\in\{1,..,D\},c\in\mathcal{C}$:
% \begin{align}\label{eqn:integer_search_space}
%   \mathcal{S} = \left\{(\balpha, \bbeta) \left\vert
%     \balpha\in\mathcal{A}, \bbeta\in\mathcal{B}
%     \, ; \,
%      \alpha^s_{b,c}\in \{0, 1\}^{\left|\mathcal{C}\right|} \,; \, 
%     \Sigma_{c\in\mathcal{C}} \alpha^s_{b,c} = 1  
%     \, ; \,
%   \beta^s_b\in \{0, 1\}^D \,; \,  
%     \Sigma_{b=1}^D \beta^s_b = 1 
%   \right.\right\}
% \end{align}
\begin{align}\label{eqn:integer_search_space}
   \mathcal{S} = \left\{(\balpha, \bbeta)\left\vert
   \begin{array}{c}
    % \balpha\in\mathcal{A}, \bbeta\in\mathcal{B}
    % \\
     \alpha^s_{b,c}\in \{0, 1\}^{\left|\mathcal{C}\right|} \,; \, 
    \Sigma_{c\in\mathcal{C}} \alpha^s_{b,c} = 1  \\
   \beta^s_b\in \{0, 1\}^D \,; \,  
    \Sigma_{b=1}^D \beta^s_b = 1 
    \\
   \end{array}\right.
   ;
   \begin{array}{l}
    \forall s\in\{1,..,S\}
    \\ 
    \forall b\in\{1,..,D\},c\in\mathcal{C}
   \end{array}
   \right\}
%   \notag
\end{align}
% the overall search space is composed of both the micro and macro search spaces parameterized by $\balpha\in\mathcal{A}$ and $\bbeta\in\mathcal{B}$, respectively, such that:
% \begin{align}
%   \mathcal{S} = \left\{(\balpha, \bbeta)\left\vert
%   \begin{array}{c}
%     \balpha\in\mathcal{A}, \bbeta\in\mathcal{B}
%     \\
%      \alpha^s_{b,c}\in \{0, 1\}^{\left|\mathcal{C}\right|} \,; \, 
%     \Sigma_{c\in\mathcal{C}} \alpha^s_{b,c} = 1  \\
%   \beta^s_b\in \{0, 1\}^D \,; \,  
%     \Sigma_{b=1}^D \beta^s_b = 1 
%     \\
%     \forall s\in\{1,..,S\}, b\in\{1,..,D\},c\in\mathcal{C}
%   \end{array}\right.
%   \right\}
%   \notag
% \end{align}
%

such that a continuous probability distribution is induced over the space, by relaxing $\alpha^s_{b,c}\in \{0, 1\}^{\left|\mathcal{C}\right|}$ to $\alpha^s_{b,c}\in \mathbb{R}_+^{\left|\mathcal{C}\right|}$ and $ \beta^s_b\in \{0, 1\}^D$ to $\beta^s_b\in \mathbb{R}_+^D$ to be continuous rather than discrete.
Therefore, this probability distribution can be expressed by a set of linear equations and one can view the parametrization $\zeta=(\balpha, \bbeta)$ as a composition of probabilities in $\mathcal{P}_{\zeta}(\mathcal{S})=\{\zeta \mid A_{\mathcal{S}}\zeta\leq b_{\mathcal{S}}\}=\{(\balpha,\bbeta) \mid A_{\mathcal{S}}^\alpha\cdot\balpha\leq b_{\mathcal{S}}^\alpha, A_{\mathcal{S}}^\beta\cdot\bbeta\leq b_{\mathcal{S}}^\beta\}$ or as degenerate one-hot vectors in $\mathcal{S}$.
% Effectively we include at least a couple of blocks in each stage by setting $\beta^s_1\equiv 0$, hence, the overall size of the search space is:
% \begin{align*}
% &|\mathcal{S}|=(\sum_{b=2}^d \left|\mathcal{C}\right|^b)^S =
% (\sum_{b=2}^d \left|\mathcal{A}_{er}\right|^b\cdot\left|\mathcal{A}_{k}\right|^b\cdot\left|\mathcal{A}_{se}\right|^b)^S \\
% &=((3\times 2\times 2)^2+(3\times 2 \times 2)^3+(3\times 2\times 2)^4)^5\approx 10^{27}
% \end{align*}

%% file: fw_algo.tex
\begin{algorithm}[htb]
   \caption{BCFW with Line Search for BLCP}
   \label{alg:BCFW_QCQP}
\begin{algorithmic}[1]
\INPUT $(\balpha_0, \bbeta_0) \in \left\{(\balpha,\bbeta) \left| \begin{array}{c}
      \balpha^T \Theta \bbeta \leq T,
      A_{\mathcal{S}}^\alpha\balpha \leq b_{\mathcal{S}}^\alpha, A_{\mathcal{S}}^\beta\bbeta \leq b_{\mathcal{S}}^\beta 
\end{array}\right.\right\}$% , $0<p<1$
\FOR {$k=0,\dots,K-1$}
% \STATE Pick $\delta:=\balpha$ or $\delta:=\bbeta$ at random
% \STATE Set $\gamma_k=\frac{4}{k+4}$ and sample $\delta\sim Bernoulli(p)$
% \IF {$\delta=1$}
\IF {$Bernoulli(p)==1$}
\STATE $\balpha_{k+1}=\argmax_{\balpha} (q^T_\alpha + \beta_k^T Q_{\alpha\beta}^T) \cdot \balpha$ s.t. $\bbeta_k^T \Theta^T \cdot \balpha\leq T$ ;
% \STATE $\hat{\balpha}=\argmax_{\balpha} (q^T_\alpha + \beta_k^T Q_{\alpha\beta}^T) \cdot \balpha$ s.t. $\bbeta_k^T \Theta^T \cdot \balpha\leq T$ ;
$A_{\mathcal{S}}^\alpha\cdot\balpha\leq b_{\mathcal{S}}^\alpha$
 and $\bbeta_{k+1}=\bbeta_k$
% \STATE Update $\balpha_{k+1} = (1-\gamma_k)\cdot\balpha_k + \gamma_k\cdot\hat{\balpha}$ and $\bbeta_{k+1}=\bbeta_k$
% $A_{\mathcal{S}}\cdot(\balpha, \bbeta_k)\leq b_{\mathcal{S}}$
\ELSE
\STATE $\bbeta_{k+1}=\argmax_{\bbeta} (q^T_\beta + \alpha_k^T Q_{\alpha\beta}) \cdot \bbeta$ s.t. $\balpha_k^T \Theta \cdot \bbeta\leq T$ ;
% \STATE $\hat{\bbeta}=\argmax_{\bbeta} (q^T_\beta + \alpha_k^T Q_{\alpha\beta}) \cdot \bbeta$ s.t. $\balpha_k^T \Theta \cdot \bbeta\leq T$ ;
$A_{\mathcal{S}}^\beta\cdot\bbeta\leq b_{\mathcal{S}}^\beta$
 and $\balpha_{k+1}=\balpha_k$
% $A_{\mathcal{S}}\cdot(\balpha_k, \bbeta)\leq b_{\mathcal{S}}$
% \STATE Update $\bbeta_{k+1} = (1-\gamma_k)\cdot\bbeta_k + \gamma_k\cdot\hat{\bbeta}$ and $\balpha_{k+1}=\balpha_k$
\ENDIF
\ENDFOR
\OUTPUT $\zeta^*=(\balpha_K, \bbeta_K)$
\end{algorithmic}
\end{algorithm}

%% file: exp.tex
\section{Experimental Results}\label{sec:experiments}
% \subsection{Dataset and Setting} \label{sec:exp_setting}
\begin{wrapfigure}{R}{0.5\textwidth}
\vspace{-5mm}
\begin{center}
    \input{acc_vs_lat_graph}
\end{center}
\vspace{-5mm}
% \captionof{figure}{Imagenet Top-1 accuracy vs latency. Circles are trained from scratch and squares marginally cost 15 GPU hours per generated model. 
% % % as proposed by the corresponding methods. 
% BINAS can generate models of comparable accuracy to other methods with a reduced marginal cost or better models with a similar cost.}
\captionof{figure}{Imagenet Top-1 accuracy vs latency. All models are trained from scratch. 
BINAS generates better models than other methods with a reduced scalable marginal cost of 8 GPU hours for each additional model.}
\label{fig:acc_vs_lat}
\end{wrapfigure}
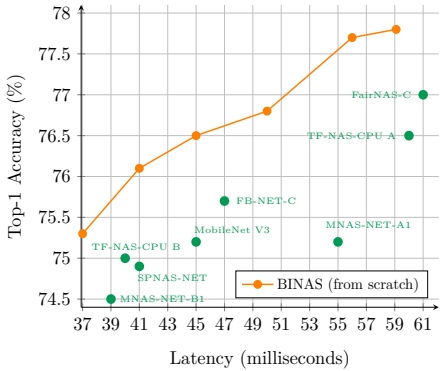
% Throughout this section we apply our search method on the ImageNet dataset, selecting $20\%$ of the train set as a validation set and the other $80\%$ for. 
% The train data for the accuracy predictors of sections \ref{sec:learnt_quad_predictors} and \ref{sec:beyond_quad} is composed of subnetworks uniformly sampled from the supernetwork and their corresponding validation accuracy is measured over the same 20\% of the Imagenet train set, considered as a validation set. The same validation set is used for the Monte-Carlo sampling mentioned in Section~\ref{sec:acc_contrib}. 
% To avoid overfitting we use regularization when learning accuracy predictors whose coefficient is tuned over 10\% of the collected data, see appendix~\ref{apdx:closed_form_sol}. 
% The test set for evaluating the ranking correlations of all the accuracy predictors is composed of another 500 samples generated uniformly in the same way. More reproduciblity details are provided in appendix~\ref{apdx:reproduce}.
\vspace{-4mm}
\subsection{Search for State-of-the-Art Architectures}
\subsubsection{Search Space Specifications. } \label{sec:search_space_specifications}
Aiming at latency efficient architectures, we adopt the search space introduced in~\cite{nayman2021hardcore}% and illustrated in Figure~\ref{fig:super_net}
, which is closely related to those used by~\cite{fbnet,mobilenetv3,tan2019mnasnet,TF-NAS,OFA}. The macro search space is composed of $S=5$ stages, each composed of at most $D=4$ blocks. 
The micro search space is based on \textit{Mobilenet Inverted Residual} (MBInvRes) blocks~\cite{mobilenetv2} and controls the internal structures of each block. 
Every MBInvRes block is configured by an expansion ratio $er\in\{3,4,6\}$ of the point-wise convolution, kernel size $k\in\{3\times3, 5\times5\}$ of the Depth-Wise Separable convolution (DWS), and Squeeze-and-Excitation (SE) layer~\cite{SE} $se\in\{\text{on}, \text{off}\}$ 
% shown at the bottom of Figure~\ref{fig:super_net} and 
(details in Appendix~\ref{apdx:search_space}).

\vspace{-4mm}
\subsubsection{Comparisons with Other Methods. }\label{sec:exp_comparison}
% Unlike OFA that requires $250+3\times 120$ epochs, our procedure of pretraining requires only 350 epochs, since it is the ranking between the different sub-networks that matters rather than their final accuracy.
We compare our generated architectures to other state-of-the-art NAS methods in Table~\ref{tab:exp} and Figures~\ref{fig:acc_vs_lat} and \ref{fig:trio} (Right). 
% We compare our generated architectures by other state-of-the-art NAS methods in Table~\ref{tab:exp_cpu},\ref{tab:exp_GPU} as presented in Figure~\ref{fig:acc_nas}. 
Aiming for surpassing the previous state-of-the-art~\cite{nayman2021hardcore} search methods, we use its search space (Section~\ref{sec:search_space_specifications}) and official supernetwork training. This way we can show that the improved search method (Section~\ref{sec:optimizers_compare} and Figure~\ref{fig:trio} (Middle)) leads to superior results for the same marginal search cost of 15 GPU hours per additional generated model, as shown in Figure~\ref{fig:trio} (Right).
For the purpose of comparing the generated architectures alone of other methods, excluding the contribution of evolved pretraining techniques, for each model in Table~\ref{tab:exp} and Figure~ \ref{fig:trio} (Right), the official PyTorch implementation~\cite{pytorch} is trained from a scratch using the exact same code and hyperparameters, as specified in appendix~\ref{apdx:reproduce}. The maximum accuracy between our training and the original paper is reported. The latency values presented are actual time measurements of the models, running on a single thread with the exact same settings and on the same hardware. We disabled optimizations, e.g., Intel MKL-DNN~\cite{mkl_dnn}, hence the latency we report may differ from the one originally reported. 
It can be seen that networks generated by our method meet the latency target closely, while at the same time are comparable to or surpassing all the other methods on the top-1 Imagenet accuracy with a reduced scalable search cost. 
The total search time consists of $435$ GPU hours computed only once as preprocessing and additional $8$ GPU hours for fine-tuning each generated network, while the search itself requires negligible several CPU minutes, see appendix~\ref{apdx:overview} for more details. Due to the negligible search cost, one can choose to train the models longer (e.g. 15 GPU hours) to achieve better results with no larger marginal cost than other methods.
\input{gpu_table}
% \input{big_table}

\vspace{-5mm}
\section{Empirical Analysis of Key Components}
In this section we analyze and discuss different aspects of the proposed method.

\vspace{-3mm}
\subsection{The Contribution of Different Terms of the Accuracy Estimator} \label{sec:terms_contrib}
% \vspace{5mm}
\begin{wrapfigure}{R}{0.5\textwidth}
% \begin{figure}[htb]
% \begin{table}[H]
\vspace{-3.5mm}
  \begin{center}
      \begin{tabular}{|c||c|c|}%c|}
    \hline
    Variant & Kendall-Tau & Spearman% &  MSE 
    \\
    \hline
    $q_\beta\equiv0$ & 0.29 & 0.42 %& 0.64 
    \\
    $Q_{\alpha\beta}\equiv0$ & 0.66 & 0.85 %& 0.24 
    \\ \hline
    $ACC(\balpha,\bbeta)$ & 0.84 & 0.97 %& 0.07  
    \\\hline
    % $q_\beta\equiv0$ & 0.36 & 0.53 & 0.55 & 1.35 \\
    % $Q_{\alpha\beta}\equiv0$ & 0.72 & 0.90 & 0.88 & 0.40 \\ \hline
    % $ACC(\balpha,\bbeta)$ & 0.87 & 0.98 & 0.95 & 0.13  \\\hline
    \end{tabular}
  \end{center}
  \vspace{-3mm}
    \captionof{table}{Contribution of terms.}% in the accuracy estimator}
    \label{tab:contrib_ablation}
% \end{table}
%   \vspace{2mm}
\end{wrapfigure}
The accuracy estimator in \eqref{eqn:estimator} aggregates the contributions of multiple architectural decisions. In section~\ref{sec:quad_estimator}, those decisions are grouped into two groups: (1) macroscopic decisions about the depth of each stage are expressed by $q_\beta$ and (2) microscopic decisions about the configuration of each block are expressed by $Q_{\alpha\beta}$. 
Table~\ref{tab:contrib_ablation} quantifies the contribution of each of those terms to the ranking correlations by setting the corresponding terms to zero. We conclude that the depth of the network is very significant for estimating the accuracy of architectures, as setting $q_\beta$ to zero specifically decreases the Kendall-Tau and Spearman's correlation coefficients from 
$0.84$ and $0.97$ to $0.29$ and $0.42$ respectively. 
% $0.87$ and $0.98$ to $0.36$ and $0.53$ respectively. 
The significance of microscopic decisions about the configuration of blocks is also viable but not as much, as setting $Q_{\alpha\beta}$ to zero decreases the Kendall-Tau and Spearman's correlation to 
$0.66$ and $0.85$ respectively.
% $0.72$ and $0.90$ respectively.

\vspace{-3mm}
\subsection{Comparison to Learning the Accuracy Predictors}\label{sec:learnt_predictors}
While the purpose of this work is not to compare many accuracy predictors, as this has been already done comprehensively by \cite{white2021powerful}, comparisons to certain learnt predictors support the validity and benefits of the proposed accuracy estimator (section~\ref{sec:quad_estimator}) despite its simple functional form. 
Hence each of the following comparisons is chosen for a reason:
(1) Showing that it is more sample efficient than learning the parameters of a bilinear predictor of the same functional form. (2) Comparing to a quadratic predictor shows that reducing the functional form to bilinear by omitting the interactions between microscopic decisions ($\balpha$ parameters) to each other and of macroscopic decisions ($\bbeta$ parameters) to each other does not cause much degradation. (3) Comparing to a parameter heavy MLP predictor shows that the simple bilinear parametric form does not lack the expressive power required for properly ranking architectures. 

\vspace{-2mm}
\subsubsection{Learning Quadratic Accuracy Predictors}\label{sec:learnt_quad_predictors}
% \begin{itemize}
%     \item One should wonder if our derived formula is the best one can do with a quadratic predictor
%     \item A closed form solution of an over-determined  linear regression problem is be the best one can do
%     \item Explain the SVD approximation of the closed for solution.\\
%     *Might be moved to the appendix if we don't have space
%     \item Show a graph of the ranking correlation vs validation epochs for such closed form - both for BC quadratic and full quadratic
%     \item Say a few words about the MLP predictor that cannot benefit from a closed form solution and has to be trained by GD - to show results in the experiments section
% \end{itemize}
One can wonder whether setting the coefficients $Q_{\alpha\beta}$, $q_\beta$ and $r$ of the bilinear form in section~\ref{sec:quad_estimator} according to the estimates in section~\ref{sec:acc_contrib} yields the best predictions of the accuracy of architectures. 
An alternative approach is to learn those coefficients by solving a linear regression:
% \begin{align} \label{eqn:regression_bcquad}
%     \min_{\tilde{r},\tilde{q}_\alpha,\tilde{q}_\beta,\tilde{Q}_{\alpha\beta}}\sum_{i=1}^n
%  ||\tilde{r} + \balpha_i^T \tilde{q}_\alpha+\bbeta_i^T \tilde{q}_\beta+\balpha_i^T \tilde{Q}_{\alpha\beta} \bbeta_i - Acc(\balpha_i, \bbeta_i)||_2^2
% \end{align}
\begin{align} \label{eqn:regression_bcquad}
    \min_{\tilde{r},\tilde{q}_\alpha,\tilde{q}_\beta,\tilde{Q}_{\alpha\beta}}\sum_{i=1}^n
 ||\mathcal{B}_{\tilde{r},\tilde{q}_\alpha,\tilde{q}_\beta,\tilde{Q}_{\alpha\beta}}(\balpha_i,\bbeta_i)- Acc(\balpha_i, \bbeta_i)||_2^2
 \\
\mathcal{B}_{\tilde{r},\tilde{q}_\alpha,\tilde{q}_\beta,\tilde{Q}_{\alpha\beta}}(\balpha_i,\beta_i) =
\tilde{r} + \balpha_i^T \tilde{q}_\alpha+\bbeta_i^T \tilde{q}_\beta+\balpha_i^T \tilde{Q}_{\alpha\beta} \bbeta_i 
\end{align}
where $\{\balpha_i,\bbeta_i\}_{i=1}^n$ and $Acc(\balpha_i, \bbeta_i)$ represent $n$ uniformly sampled subnetworks and their measured accuracy, respectively. %Hence, overall the data collection requires $n$ validation epochs.

One can further unlock the full capacity of a quadratic predictor by coupling of all components and solving the following linear regression problem:
% \begin{align*} \label{eqn:regression_quad}
%     \min_{\tilde{r},\tilde{q}_\alpha,\tilde{q}_\beta,\tilde{Q}_{\alpha\beta},\tilde{Q}_{\alpha},\tilde{Q}_{\beta}}\sum_{i=1}^n
%  ||\left(\tilde{r} + \balpha_i^T \tilde{q}_\alpha+\bbeta_i^T \tilde{q}_\beta+\balpha_i^T \tilde{Q}_{\alpha\beta} \bbeta_i\right) + \balpha_i^T \tilde{Q}_{\alpha} \balpha_i + \bbeta_i^T \tilde{Q}_{\beta} \bbeta_i - Acc(\balpha_i, \bbeta_i)||_2^2
% % = 
%     % \min_{\tilde{r},\tilde{q},\tilde{Q}}\sum_{i=1}^n
% %  ||\tilde{r} + \zeta_i^T \tilde{q}+\zeta_i^T \tilde{Q} \zeta_i - Acc(\balpha_i, \bbeta_i)||_2^2
%  \end{align*}
\begin{align*} 
    \min_{\tilde{r},\tilde{q}_\alpha,\tilde{q}_\beta,\tilde{Q}_{\alpha\beta},\tilde{Q}_{\alpha},\tilde{Q}_{\beta}}\sum_{i=1}^n
 ||\mathcal{Q}_{\tilde{r},\tilde{q}_\alpha,\tilde{q}_\beta,\tilde{Q}_{\alpha\beta},\tilde{Q}_{\alpha},\tilde{Q}_{\beta}}&(\balpha_i,\bbeta_i) - Acc(\balpha_i, \bbeta_i)||_2^2
% = 
    % \min_{\tilde{r},\tilde{q},\tilde{Q}}\sum_{i=1}^n
%  ||\tilde{r} + \zeta_i^T \tilde{q}+\zeta_i^T \tilde{Q} \zeta_i - Acc(\balpha_i, \bbeta_i)||_2^2
 \end{align*}
% \begin{align} \notag\label{eqn:regression_quad}
%     \min_{\tilde{r},\tilde{q}_\alpha,\tilde{q}_\beta,\tilde{Q}_{\alpha\beta},\tilde{Q}_{\alpha},\tilde{Q}_{\beta}}\sum_{i=1}^n
%  \left|\left|\begin{array}{c}
%         \hat{ACC}_{\tilde{r},\tilde{q}_\alpha,\tilde{q}_\beta,\tilde{Q}_{\alpha\beta},\tilde{Q}_{\alpha},\tilde{Q}_{\beta}}(\balpha_i,\bbeta_i)\\
%       - Acc(\balpha_i, \bbeta_i)
%  \end{array}\right|\right|_2^2
% % = 
%     % \min_{\tilde{r},\tilde{q},\tilde{Q}}\sum_{i=1}^n
% %  ||\tilde{r} + \zeta_i^T \tilde{q}+\zeta_i^T \tilde{Q} \zeta_i - Acc(\balpha_i, \bbeta_i)||_2^2
% \end{align}
\begin{align}
\label{eqn:regression_quad}
\mathcal{Q}_{\tilde{r},\tilde{q}_\alpha,\tilde{q}_\beta,\tilde{Q}_{\alpha\beta},\tilde{Q}_{\alpha},\tilde{Q}_{\beta}}(\balpha_i,\bbeta_i)
=
\mathcal{B}_{\tilde{r},\tilde{q}_\alpha,\tilde{q}_\beta,\tilde{Q}_{\alpha\beta}}(\balpha_i,\beta_i)  + \balpha_i^T \tilde{Q}_{\alpha} \balpha_i + \bbeta_i^T \tilde{Q}_{\beta} \bbeta_i 
 \end{align}
 
A closed form solution to 
these problems 
% problems \ref{eqn:regression_bcquad} and \ref{eqn:regression_quad} 
is derived in appendix~\ref{apdx:closed_form_sol}. While effective, this solution requires avoiding memory issues associated with inverting $N^2\times N^2$ matrix and also reducing overfitting by tuning regularization effects over train-val splits of the data points. 
The data points for training all the accuracy predictors is composed of subnetworks uniformly sampled from the supernetwork and their corresponding validation accuracy is measured over the same 20\% of the Imagenet train set split used in section~\ref{sec:acc_contrib}. 
% To avoid overfitting we use regularization when learning accuracy predictors whose coefficient is tuned over 10\% of the collected data, see appendix~\ref{apdx:closed_form_sol}. 

% Figure~\ref{fig:kt_vs_samples} shows that the estimator proposed in section~\ref{sec:quad_estimator} matches the performance of those learnt predictors while being more sample efficient.
Figure~\ref{fig:trio} (Left) presents the Kendall-Tau ranking correlation coefficients and mean square error (MSE), measured over 500 test data points generated uniformly at random in the same way, of different accuracy predictors versus the number of data points corresponding to the number of epochs of the validation set required for obtaining their parameters. 
It is noticable that the simple bilinear accuracy estimator (section~\ref{sec:quad_estimator}) is more sample efficient, as its parameters are efficiently estimated according to section~\ref{sec:acc_contrib} rather than learned. 

% Although the parameters of those predictors are learnt in a supervised manner rather than estimated in a predefined way, this flexibility also makes those more sensitive to the sampled data and less sample efficient, as shown in Figure~\ref{fig:kt_vs_samples}. Specifically one should also add regularization and tune its coefficient to avoid overfitting to the train data as discussed in section \ref{sec:correlation_cost}.

\vspace{-3mm}
\subsubsection{Beyond Quadratic Accuracy Predictors}\label{sec:beyond_quad}
% \begin{wrapfigure}{R}{0.45\textwidth}
%   \begin{center}
%       \input{kt_vs_samples}
%   \end{center}
% \end{wrapfigure}

The reader might question the expressiveness of a simple bilinear parametric form and its ability to capture the complexity of architectures. 
% Indeed,~\cite{white2021powerful} present many complex accuracy predictors and corresponding sampling techniques. 
To alleviate such concerns we show in Figure~\ref{fig:trio} (Left) 
% and Section~\ref{sec:correlation_cost} 
that the proposed bilinear estimator of section~\ref{sec:quad_estimator} matches the performance of the commonly used parameters heavy Multi-Layer-Perceptron (MLP) accuracy predictor~\cite{OFA,lu2020nsganetv2}. Moreover, the MLP predictor is more complex and requires extensive hyperparameter tuning, e.g., of the depth, width, learning rate and its scheduling, weight decay, optimizer etc. It is also less efficient, lacks interpretability, and of limited utility as an objective function for NAS (Section~\ref{sec:solvers}).

% \subsubsection{Ranking Correlation Per Cost for Different Accuracy Predictors}\label{sec:correlation_cost} Figure~\ref{fig:kt_vs_samples} presents the Kendall-Tau ranking correlation coefficients and MSE of different accuracy predictors that are introduced in Section~\ref{sec:quad_predictors} versus the number of epochs of the validation set required for obtaining their parameters. 
% It is noticable that the simple bilinear accuracy estimator (section~\ref{sec:quad_estimator}) is more sample efficient, as its parameters are estimated rather than learned. 
% The effectiveness of the Multi-path sampling \cite{nayman2021hardcore} of driving each image through a distinct subnetwork for obtaining the accuracy estimator is very clear comparing to the Single-path counterpart of driving each batch through the same subnetwork. With access to enough samples, the performance of these all are comparable. Although simple, this shows that the proposed bilinear accuracy estimator is expressive enough for the purpose of NAS.

% The later do not have a closed form solution and is harder to train with many hyperparameters to be tuned.
% In section~\ref{sec:solvers} we show that it is also harder to optimize as an objective function of problem~\ref{eqn:NAS_QCQP}.
% \input{kt_vs_samples} 

\begin{figure*}[htb]
 \begin{minipage}[t]{0.36\textwidth}
     \input{kt_vs_samples} 
  \end{minipage}
  \hfill
  \begin{minipage}[t]{0.31\textwidth}
      \input{estimator_error_bars}
  \end{minipage}
  \hfill
  \begin{minipage}[t]{0.31\textwidth}
     \input{binas_vs_hardcore_only}
  \end{minipage}
  \vspace{-4mm}
  \captionof{figure}{
  (Left) Performance of predictors vs samples. Ours is comparable to complex alternatives and sample efficient. (Middle) Comparing optimizers for solving the IQCQP over 5 seeds. All surpass optimizing the supernetwork (HardCoreNAS) directly. (Right) Surpassing State-of-the-Art. BINAS generates superior models than HardCoReNAS for the same search space, supernetwork weights and marginal cost of 15 GPU hours.
  }
  \label{fig:trio}
    %  \vspace{18mm}
%   \end{minipage*}
\end{figure*}
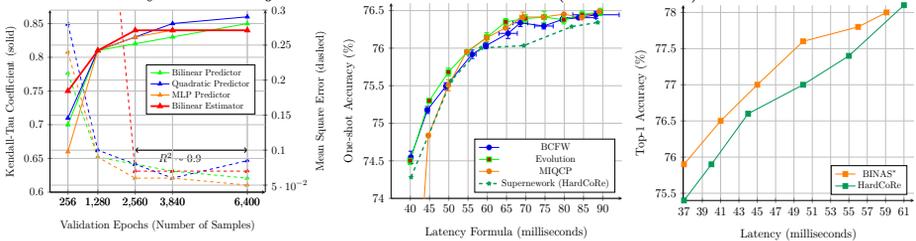

% \begin{minipage}{\textwidth}
%  \begin{minipage}{0.49\textwidth}
%      \input{kt_vs_samples} 
%   \end{minipage}
%   \hfill
%   \begin{minipage}{0.49\textwidth}
%       \input{estimator_error_bars}
%   \end{minipage}
%     % \captionof{figure}{(Left) The BINAS scheme constructs a quadratic accuracy estimator by measuring the accuracy contribution of individual design choices and then maximizing this objective under quadratic latency constraints. (Right) Subnetworks' accuracy predictions by the proposed estimator vs their measured accuracy. High ranking correlations are achieved.}
%     % \label{fig:kt_solvers}
% \end{minipage}
% % \begin{minipage}{\textwidth}
% %  \begin{minipage}{0.49\textwidth}
% %      \input{kt_vs_samples} 
% %   \end{minipage}
% %   \hfill
% %   \begin{minipage}{0.49\textwidth}
% %       \input{estimator_error_bars}
% %   \end{minipage}
% %     % \captionof{figure}{(Left) The BINAS scheme constructs a quadratic accuracy estimator by measuring the accuracy contribution of individual design choices and then maximizing this objective under quadratic latency constraints. (Right) Subnetworks' accuracy predictions by the proposed estimator vs their measured accuracy. High ranking correlations are achieved.}
% %     % \label{fig:kt_solvers}
% % \end{minipage}
% \vspace{5mm}
% ~

\subsection{Interpretability of the Accuracy Estimator}\label{sec:interpretability}
% \vspace{-1.5mm}
Given that the accuracy estimator in section~\ref{sec:quad_estimator} ranks architectures well, as demonstrated in Figure~\ref{fig:scheme} (Right), this accountability together with the way it is constructed bring insights about the contribution of different design choices to the accuracy, as shown in Figure~\ref{fig:insights}. 
\\\textbf{Deepen later stages}: 
In the left figure $\Delta_b^s-\Delta_{b-1}^s$ are presented for $b=3,4$ and $s=1,\dots,5$. This graph shows that increasing the depth of deeper stages is more beneficial than doing so for shallower stages. Showing also the latency cost for adding a block to each stage, we see that there is a strong motivation to make later stages deeper.
\\\textbf{Add width and S\&E to later stages and shallower blocks}: 
In the middle and right figures, $\Delta_{b,c}^s$ are averaged over different configurations and blocks or stages respectively for showing the contribution of microscopic design choices. Those show that increasing the expansion ratio and adding S\&E are more significant in deeper stages and at sooner blocks within each stage. 
\\\textbf{Prefer width and S\&E over bigger kernels}: 
  Increasing the kernel size is relatively less significant and is more effective at intermediate stages.
  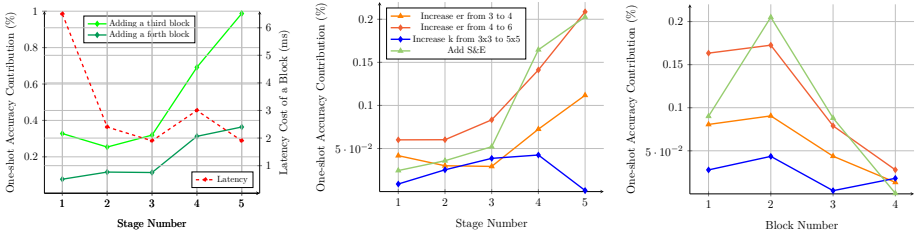
\begin{figure*}[htb]
%  \begin{minipage}{\textwidth}
 \begin{minipage}[b]{0.32\textwidth}
     \input{insights_depth_stages}
  \end{minipage}
  \hfill
  \begin{minipage}[b]{0.32\textwidth}
        \input{insights_stages}
  \end{minipage}
  \hfill
  \begin{minipage}[b]{0.32\textwidth}
    \input{insights_blocks}
  \end{minipage}
%   \vspace{-5mm}
  \captionof{figure}{Design choices insights deduced from the accuracy estimator:  The contribution of (Left) depth for different stages, (Middle) expansion ration, kernel size and S\&E for different stages and (Right) for different blocks within a stage.}
  \label{fig:insights}
    %  \vspace{18mm}
%   \end{minipage*}
\end{figure*}

\subsection{Comparison of Optimization Algorithms}\label{sec:optimizers_compare}
Formulating the NAS problem as IQCQP affords the utilization of a variety of optimization algorithms. Figure~\ref{fig:trio} (Middle) compares the one-shot accuracy and latency of networks generated by utilizing the algorithms suggested in section~\ref{sec:solvers} for solving problem~\ref{eqn:NAS_QCQP} with the bilinear estimator introduced in section~\ref{sec:quad_estimator} serving as the objective function.
Error bars for both accuracy and latency are presented for 5 different seeds. All algorithms satisfy the latency constraints up to a reasonable error of less than $10\%$. While all of them surpass the performance of BCSFW~\cite{nayman2021hardcore}, given as reference, BCFW is superior at low latency, evolutionary search does well over all and MIQCP is superior at high latency. Hence, for practical purposes we  apply the three of them for search and take the best one, with negligible computational cost of less than three CPU minutes overall.

%% file: acc_vs_lat_graph.tex
% \begin{figure}[htb]
% \begin{figure}[H]
% \centerin
\begin{adjustbox}{width=0.48\textwidth}
\begin{tikzpicture}
\begin{axis}[
            axis x line=left,
            axis y line=left,
            xmajorgrids=true,
            ymajorgrids=true,
            grid=both,
            xlabel style={below=1ex},
            enlarge x limits,
            ymin = 74.4,
            ymax = 78.1,
            xmin = 39.0,
            xmax = 60.0,
            % xmax = 90.0,
            % x tick label style={ %font=\fontsize{4}{6}\selectfont, text width=0.5cm},
            xtick = {37,39,...,65},
            ytick = {74.5,75,...,78.2},
            ylabel = Top-1 Accuracy (\%),
            xlabel = Latency  (milliseconds),
            legend pos=south east,
            % legend pos=north west,
            legend style={nodes={scale=0.8, transform shape}}
    ]

% \addplot[only marks,color=gray,mark=triangle*,very thick]coordinates {(39,74.8)(48,76.0)(61,77.1)};\addlegendentry{Ours}

% \addplot[dotted, color=orange, mark=square*, thick]coordinates {(36.96345329284668,75.506)(40.85262537002563,76.308)(44.83997821807861,77.048)(50.08040189743042,77.542)
% % (56.44809627532959,77.650)
% (56.49735927581787,77.786)
% (59.065587520599365,78.052)
% % (64.87151622772216,78.068)
% };\addlegendentry{Ours FT (68634)}

% \addplot[dotted, color=orange, mark=square*, thick]coordinates {(37,75.5)(41,76.3)(45,77.1)(50,77.5)
% % (56,77.6)
% (56,77.8)
% (59.1,78.1)
% % (65,78.1)
% };\addlegendentry{Ours FT (68634-64465)}

\addplot[solid, color=orange, mark=*, thick]coordinates {(37,75.3)(41,76.1)(45,76.5)
(50,76.8)
% % (56,77.1)
(56, 77.7)
(59.1,77.8)
% (65,77.6)
};\addlegendentry{BINAS (from scratch)}

\addplot[only marks,color=ForestGreen,mark=*,very thick]coordinates {(45, 75.2)} node (MobileNet_V3){};
\node[color=ForestGreen, above right = -0.1cm and -0.3cm of MobileNet_V3, font=\tiny] {MobileNet V3};%\addlegendentry{Others}

\addplot[only marks,color=ForestGreen,mark=*,very thick]coordinates {(40, 75.0)} node (tfnasb){};
\node[color=ForestGreen, above right of=tfnasb, font=\tiny, node distance = 0.3cm] {TF-NAS-CPU B};

\addplot[only marks,color=ForestGreen,mark=*,very thick]coordinates {(60, 76.5)} node (tfnasa){};
\node[color=ForestGreen, left of=tfnasa, font=\tiny, node distance = 1.1cm] {TF-NAS-CPU A};

% \addplot[only marks,color=ForestGreen,mark=*,very thick]coordinates {(36, 74.5)} node (mnas){};
% \node[color=ForestGreen, above of=mnas, font=\tiny, node distance = 0.2cm] {MNAS-NET};

\addplot[only marks,color=ForestGreen,mark=*,very thick]coordinates {(55, 75.2)} node (mnasa1){};
\node[color=ForestGreen, above right = 0.0cm and -0.5cm of mnasa1, font=\tiny, node distance = 0.01cm] {MNAS-NET-A1};

\addplot[only marks,color=ForestGreen,mark=*,very thick]coordinates {(39, 74.5)} node (mnasb1){};
\node[color=ForestGreen, right of=mnasb1, font=\tiny, node distance = 0.99cm] {MNAS-NET-B1};

\addplot[only marks,color=ForestGreen,mark=*,very thick]coordinates {(41, 74.9)} node (spnas){};
\node[color=ForestGreen, below right = -0.1cm and -0.3cm of spnas, font=\tiny, node distance = 0.8cm] {SPNAS-NET};

% \addplot[only marks,color=ForestGreen,mark=*,very thick]coordinates {(42, 75.2)} node (semnas){};
% \node[color=ForestGreen, below of=semnas, font=\tiny, node distance = 0.2cm] {SE-MNAS-NET};

\addplot[only marks,color=ForestGreen,mark=*,very thick]coordinates {(47, 75.7)} node (fbnet){};
\node[color=ForestGreen, right of=fbnet, font=\tiny, node distance = 0.8cm] {FB-NET-C};

\addplot[only marks,color=ForestGreen,mark=*,very thick]coordinates {(61, 77.0)} node (FairNASC){};
\node[color=ForestGreen, left of=FairNASC, font=\tiny, node distance = 0.8cm] {FairNAS-C};

% \addplot[only marks,color=ForestGreen,mark=*,very thick]coordinates {(42, 75.7)} node (OFA){};
% \node[color=ForestGreen, below of=OFA, font=\tiny, node distance = 0.2cm] {OFA};

\end{axis}
\end{tikzpicture}
\end{adjustbox}
% \label{fig:acc_nas}
% \end{figure}

%% file: gpu_table.tex
\begin{table}[t]
\small
    \centering
    \begin{tabular}{cc}
    \begin{tabular}{|l|c|c|c|}
    \hline 
    Model & \makecell[c]{Latency \\ (ms)} & \makecell[c]{Top-1 \\ (\%)} & \makecell[c]{Total  Cost\\(GPU hours)}%Search + Training \\ Cost (GPU hours)}
    \\ \toprule
          MnasNetB1 & 39 & 74.5 & 40,000N\\ 
          TFNAS-B & 40 & 75.0 & 263N \\ 
          SPNASNet & 41 & 74.9 & 288 + 408N\\ 
          OFA CPU%\footnotemark 
          & 42 & 75.7 & 1200 + 25N \\ 
          HardCoRe A & 40 & 75.8 & 400 + 15N\\ 
          \textbf{BINAS(40)} & \textbf{40} & \textbf{76.1} & 435 + \textbf{8N}\\ 
          \textbf{BINAS$^*$(40)} & \textbf{40} & \textbf{76.5} & 435 + 15N \\ \hline
         
          MobileNetV3 & 45 & 75.2 & 180N \\ 
          FBNet & 47 & 75.7 & 576N \\ 
          MnasNetA1 & 55 & 75.2 & 40,000N \\ 
        
          HardCoRe B & 44 & 76.4 & 400 + 15N\\ 
          \textbf{BINAS(45)} & \textbf{45} & \textbf{76.5} & 435 + \textbf{8N}\\ 
          \textbf{BINAS$^*$(45)} & \textbf{45} & \textbf{77.0} & 435 + 15\\ \hline
         
          MobileNetV2 & 70 & 76.5 & 150N \\ 
          TFNAS-A & 60 & 76.5 & 263N \\ 
          HardCoRe C & \textbf{50} & 77.1 & 400 + 15N \\ 
          \textbf{BINAS(50)} & \textbf{50} & 76.8 & 435 + \textbf{8N} \\ 
          \textbf{BINAS$^*$(50)} & \textbf{50} & \textbf{77.6} & 435 + 15N \\ \hline
         
          EfficientNetB0 & 85 & 77.3 &  \\ 
          HardCoRe D & 55 & 77.6 &  400 + 15N\\
          \textbf{BINAS(55)} & \textbf{55} & 77.7 &  435 + \textbf{8N}\\ 
          \textbf{BINAS$^*$(55)} & \textbf{55} & \textbf{77.8} &  435 + 15N\\ \hline
         
          FairNAS-C & 60 & 77.0 & 240N \\ 
          HardCoRe E & 61 & \textbf{78.0} & 400 + 15N\\ 
          \textbf{BINAS(60)} & \textbf{59} & 77.8 & 435 + \textbf{8N}\\ 
          \textbf{BINAS$^*$(60)} & \textbf{59} & \textbf{78.0} & 435 + 15N \\ \hline
    \end{tabular}
    &
    \begin{tabular}{|l|c|c|}%c|}
    \hline 
    Model & \makecell[c]{Latency \\ (ms)} & \makecell[c]{Top-1 \\ (\%)} %& \makecell[c]{Total  Cost\\(GPU hours)}%Search + Training \\ Cost (GPU hours)}
    \\ \toprule
          MobileNetV3 & 28 & 75.2  \\ 
          TFNAS-D & 30 & 74.2   \\ 
          HardCoRe A & 27 & 75.7  \\
          \textbf{BINAS(25)} & \textbf{26} & \textbf{76.1} \\
          \textbf{BINAS$^*$(25)} & \textbf{26} & \textbf{76.6} \\
         \hline
         
          MnasNetA1 & 37 & 75.2   \\
          MnasNetB1 & 34 & 74.5   \\ 
          FBNet & 41 & 75.7   \\ 
          SPNASNet & 36 & 74.9  \\ 
          TFNAS-B &  44 & 76.3   \\ 
          TFNAS-C & 37 & 75.2   \\ 
          HardCoRe B & 32 & \textbf{77.3} \\
          \textbf{BINAS(30)} & \textbf{31} & 76.8  \\ 
          \textbf{BINAS$^*$(30)} & \textbf{31} & 77.2  \\ 
          \hline
         
          TFNAS-A & 54 & 76.9   \\ 
          EfficientNetB0 & 48 & 77.3   \\ 
          MobileNetV2 & 50 & 76.5  \\ 
          HardCoRe C & 41 & \textbf{77.9} \\
          \textbf{BINAS(40)} & \textbf{40} & 77.6 \\ 
          \textbf{BINAS$^*$(40)} & \textbf{40} & 77.7  \\ \hline
    \end{tabular}
    \end{tabular}
    \caption{ImageNet top-1 accuracy, latency and cost comparison with other methods. The total cost stands for the search and training cost of N networks. Latency is reported for (Left) Intel Xeon CPU and (Right) NVIDIA P100 GPU with a batch size of 1 and 64 respectively.}
    \label{tab:exp}
    \vspace{-5mm}
\end{table}
% \footnotetext{\label{fn:ofa}Finetuning a model obtained by 1200 GPU hours.}

%% file: kt_vs_samples.tex
\begin{figure}[H]
\centering
\begin{adjustbox}{width=\textwidth}
\begin{tikzpicture}
\begin{axis}[
            axis x line=left,
            axis y line=right,
            xmajorgrids=true,
            ymajorgrids=true,
            grid=both,
            xlabel style={below=1ex},
            enlarge x limits,
            ymin = 0.04,
            ymax = 0.3,
            xmin = 256.0,
            xmax = 6400.0,
            % xmax = 90.0,
            % x tick label style={ %font=\fontsize{4}{6}\selectfont, text width=0.5cm},
            xtick = {256,1280,2560,3840,6400},
            ytick = {0.05,0.1,...,0.3},
            ylabel = Mean Square Error (dashed),
            % xlabel = Validation Epochs,
            xlabel = ,
            legend style={nodes={scale=0.8, transform shape}},
            % legend pos=south east,
    ]

\addplot[style=dashed, 
color=green, mark=triangle*,semithick]coordinates {(256, 0.21)(1280, 0.09)(2560, 0.08)(3840, 0.07)(6400, 0.06)
% (40, 0.86)
};%\addlegendentry{Block Quadratic Predictor}

\addplot[style=dashed, 
color=blue, mark=triangle*,semithick]coordinates {(256, 0.28)(1280, 0.1)(2560, 0.08)(3840, 0.06)(6400, 0.085)
% (40, 0.86)
};%\addlegendentry{Quadratic Predictor}

\addplot[style=dashed, 
color=orange, mark=triangle*,semithick]coordinates {(256, 0.24)(1280, 0.09)(2560, 0.06)(3840, 0.06)(6400, 0.05)
% (40, 0.84)
};%\addlegendentry{MLP Predictor}

\addplot[style=dashed, 
color=red, mark=triangle*,semithick]coordinates {(256, 0.81)(1280, 0.64)(2560, 0.07)(3840, 0.07)(6400, 0.07)};%\addlegendentry{Estimator (Multi-path sampling)}

\node[color=black] at (4096,0.09) {$R^2\sim 0.9$};
\draw[<->, black, thick] (2560,0.1) to (6400, 0.1);
\end{axis}

\begin{axis}[
            axis x line=left,
            axis y line=left,
            xmajorgrids=true,
            ymajorgrids=true,
            grid=both,
            xlabel style={below=1ex},
            enlarge x limits,
            ymin = 0.60,
            ymax = 0.87,
            xmin = 256.0,
            xmax = 6400.0,
            % xmax = 90.0,
            % x tick label style={ %font=\fontsize{4}{6}\selectfont, text width=0.5cm},
            xtick = {256,1280,2560,3840,6400},
            ytick = {0.6,0.65,0.7,...,0.85},
            ylabel = Kendall-Tau Coefficient (solid),
            xlabel = Validation Epochs (Number of Samples),
            legend style={nodes={scale=0.8, transform shape}, 
            % at={(0.97,0.5)},anchor=east}
            % t={(0.5,1.35)},anchor=north,legend cell align=left}
            at={(0.7,0.7)}, anchor=north,legend cell align=left, 
            % fill opacity=0.1, text opacity = 1
            }
    ]

\addplot[%style=dashed, 
color=green, mark=triangle*,semithick]coordinates {(256, 0.7)(1280, 0.81)(2560, 0.82)(3840, 0.83)(6400, 0.85)
% (40, 0.86)
};\addlegendentry{Bilinear Predictor}

\addplot[%style=dashed, 
color=blue, mark=triangle*,semithick]coordinates {(256, 0.71)(1280, 0.81)(2560, 0.83)(3840, 0.85)(6400, 0.86)
% (40, 0.86)
};\addlegendentry{Quadratic Predictor}

\addplot[%style=dashed, 
color=orange, mark=triangle*,semithick]coordinates {(256, 0.66)(1280, 0.81)(2560, 0.83)(3840, 0.84)(6400, 0.84)
% (40, 0.84)
};\addlegendentry{MLP Predictor}

% \addplot[style=dotted, 
% color=red, mark=triangle*,semithick]coordinates {(256, 0.61)(1280, 0.78)(2560, 0.81)(3840, 0.84)(6400, 0.84)};\addlegendentry{Estimator (Single-path sampling)}

\addplot[%style=dashed, 
color=red, mark=triangle*,ultra thick]coordinates {(256, 0.75)(1280, 0.81)(2560, 0.84)(3840, 0.84)(6400, 0.84)};\addlegendentry{Bilinear Estimator}%\addlegendentry{Estimator (Multi-path sampling)}

\end{axis}

\end{tikzpicture}
\end{adjustbox}
% \captionof{figure}{Performance of predictors vs samples. Ours is comparable to complex alternatives and sample efficient.} 
% % \captionof{figure}{Kendall-Tau correlation coefficients between predictors' output and the validation accuracy of sub-networks extracted from the one-shot model vs the number of validation epochs required for acquiring the predictor} %over a test set of 500 networks uniformly sampled from the search space}
\label{fig:kt_vs_samples}
\end{figure}

%% file: estimator_error_bars.tex
\begin{figure}[H]
\centering
% \begin{adjustbox}{width=0.85\textwidth}
\begin{adjustbox}{width=1\textwidth}
\begin{tikzpicture}
\begin{axis}[
            axis x line=left,
            axis y line=left,
            xmajorgrids=true,
            ymajorgrids=true,
            grid=both,
            xlabel style={below=1ex},
            enlarge x limits,
            ymin = 74.0,
            ymax = 76.6,
            xmin = 40.0,
            xmax = 90.0,
            % xmax = 90.0,
            % x tick label style={ %font=\fontsize{4}{6}\selectfont, text width=0.5cm},
            xtick = {40,45,...,90},
            ytick = {73.0,73.5,...,79.0},
            ylabel = One-shot Accuracy (\%),
            xlabel = Latency Formula  (milliseconds),
            legend pos=south east,
            legend style={nodes={scale=0.8, transform shape}},
            legend pos=south east,
    ]

\addplot+[color=blue,
% mark=square*,very thick
  error bars/.cd, 
    y fixed,
    y dir=both, 
    y explicit,
    x fixed,
    x dir=both, 
    x explicit
] table [x=x, y=y, x error=ex, y error=ey, col sep=comma] {
x, y, ex, ey
40.1415625215, 74.5472000254, 0.442107147799, 0.0816392000091
44.4331619143, 75.1764000151, 0.521277782521, 0.0459808571922
49.2570230365, 75.4936000117, 0.306043088234, 0.04703021861
56.0383793712, 75.9184000186, 1.04242471869, 0.0563332935285
59.6068625649, 76.0316000083, 0.672287025719, 0.0405837487535
65.4006982843, 76.1968000073, 2.3893996531, 0.0677831903537
68.5943294565, 76.3372000127, 1.92246258259, 0.0496000087891
74.6927519639, 76.2944000039, 2.42779181687, 0.0391182811074
79.5122796297, 76.3824000098, 0.521778982294, 0.00195960375463
83.7483993173, 76.405600019, 4.5079436476, 0.023371776967
88.2328752677, 76.4444000088, 6.02896559333, 0.047030189905
};\addlegendentry{BCFW}

\addplot+[color=green,
% mark=square*, green,
  error bars/.cd, 
    y fixed,
    y dir=both, 
    y explicit,
    x fixed,
    x dir=both, 
    x explicit
] table [x=x, y=y, x error=ex, y error=ey, col sep=comma] {
x, y, ex, ey
39.9446955323, 74.502800022, 0.0474540775655, 0.0519168528642
44.916981558, 75.2988000176, 0.0343261345705, 0.0228595817489
49.9190304677, 75.6820000146, 0.0775377104481, 0.0511937628953
54.7916676601, 75.948400019, 0.234971007873, 0.0299706521229
59.8350801071, 76.1428000005, 0.115177009846, 0.0626941730403
64.8044832548, 76.3500000054, 0.175285977232, 0.0439454243165
69.9057773749, 76.3940000034, 0.0614879574881, 0.0515363899423
74.752805233, 76.4136000068, 0.178644672749, 0.0779579448149
79.8447579145, 76.3784000039, 0.0800718190295, 0.0554458364716
84.6323028207, 76.4464000137, 0.301708172663, 0.0457016441512
89.3507918715, 76.4832000127, 0.379851297189, 0.0484247788618
};\addlegendentry{Evolution}

\addplot+[color=orange,
% mark=square*,very thick
  error bars/.cd, 
    y fixed,
    y dir=both, 
    y explicit,
    x fixed,
    x dir=both, 
    x explicit
] table [x=x, y=y, x error=ex, y error=ey, col sep=comma] {
x, y, ex, ey
32.4221819639, 65.2200000024, 0.0, 0.0
44.6809334556, 74.8376000034, 0.0840146112177, 0.0268745284624
49.8999319474, 75.5064000146, 0.0458702694263, 0.0739123826472
54.6505713463, 75.9524000156, 0.236731323331, 0.0191582915205
59.8148682714, 76.1356, 0.190107247353, 0.0484049590693
64.742282033, 76.2700000171, 0.0, 0.0
69.2713287473, 76.4140000024, 0.826951111835, 0.0659939465145
74.5630600055, 76.4144000137, 0.148398081462, 0.00079998828125
79.8846612374, 76.4519999951, 0.0, 0.0
84.5536619425, 76.4060000195, 0.0, 0.0
89.0582546592, 76.4880000195, 0.0, 0.0
};\addlegendentry{MIQCP}

\addplot[style=dashed, color=ForestGreen, mark=star,very thick]coordinates {
(40.2,74.2820)
(50.6,75.5660)(58.4,76.0020)(69.5,76.0340)(81.9,76.2860)(88.7,76.3420)
};\addlegendentry{Supernework (HardCoRe)}

\end{axis}
\end{tikzpicture}
\end{adjustbox}
% \captionof{figure}{Comparing optimizers for solving the IQCQP over 5 seeds. All surpass optimizing the supernetwork directly.}
\label{fig:estimator_error_bars}
\end{figure}
% \vspace{3mm}

%% file: binas_vs_hardcore_only.tex
% \begin{figure}[htb]
\begin{figure}[H]
\centering
\begin{adjustbox}{width=1\textwidth}
% \begin{adjustbox}{width=0.48\textwidth}
\begin{tikzpicture}
\begin{axis}[
            axis x line=left,
            axis y line=left,
            xmajorgrids=true,
            ymajorgrids=true,
            grid=both,
            xlabel style={below=1ex},
            enlarge x limits,
            ymin = 75.4,
            ymax = 78.1,
            xmin = 39.0,
            xmax = 60.0,
            % xmax = 90.0,
            % x tick label style={ %font=\fontsize{4}{6}\selectfont, text width=0.5cm},
            xtick = {37,39,...,65},
            ytick = {74.5,75,...,78.2},
            ylabel = Top-1 Accuracy (\%),
            xlabel = Latency  (milliseconds),
            legend pos=south east,
            % legend pos=north west,
            legend style={nodes={scale=0.8, transform shape}}
    ]

\addplot[color=orange, mark=square*, thick]coordinates {(37,75.9)(41,76.5)(45,77.0)(50,77.6)
% (56,77.3)
(56,77.8)
(59.1,78.0)
% (65,77.7)
};\addlegendentry{BINAS$^*$}%\addlegendentry{Ours FT (68634-64262)}

\addplot[color=ForestGreen, mark=square*, thick]coordinates {(37, 75.4)(40,75.9)(44, 76.6)(50, 77.0)(55, 77.4)(61,78.1)};\addlegendentry{HardCoRe}

\end{axis}
\end{tikzpicture}
\end{adjustbox}
% \label{fig:acc_nas}
\end{figure}

%% file: insights_depth_stages.tex
\begin{figure}[H]
\centering
\begin{adjustbox}{width=1\textwidth}
\begin{tikzpicture}
\begin{axis}[
            axis x line=left,
            axis y line=left,
            xmajorgrids=true,
            ymajorgrids=true,
            grid=both,
            xlabel style={below=1ex},
            enlarge x limits,
            ymin = 0.0,
            ymax = 1.0,
            xmin = 1.0,
            xmax = 5.0,
            % xmax = 90.0,
            % x tick label style={ %font=\fontsize{4}{6}\selectfont, text width=0.5cm},
            xtick = {1, 2,..., 5},
            ytick = {0.2, 0.4,..., 1},
            ylabel = One-shot Accuracy Contribution (\%),
            xlabel = Stage Number,
            legend style={nodes={scale=0.8, transform shape},at={(0.7,0.9)},anchor=east},
    ]

\addplot[color=green, mark=diamond*,very thick]coordinates {
(1,0.327805)(2,0.253799)(3,0.319214)(4,0.691796)(5,0.987206)
};\addlegendentry{Adding a third block}

\addplot[color=ForestGreen, mark=diamond*,very thick]coordinates {
(1,0.077003)(2,0.116189)(3,0.113785)(4,0.312606)(5,0.363792)
};\addlegendentry{Adding a forth block}

\end{axis}

\begin{axis}[
            axis x line=left,
            axis y line=right,
            xmajorgrids=true,
            ymajorgrids=true,
            grid=both,
            xlabel style={below=1ex},
            enlarge x limits,
            ymin = 0.0,
            ymax = 6.6,
            xmin = 1.0,
            xmax = 5.0,
            % xmax = 90.0,
            % x tick label style={ %font=\fontsize{4}{6}\selectfont, text width=0.5cm},
            xtick = {1, 2,..., 5},
            ytick = {1, 2,..., 6},
            ylabel = Latency Cost of a Block (ms),
            xlabel = Stage Number,
            legend style={nodes={scale=0.8, transform shape}},
            legend pos=south east,
    ]
    \addplot[style=dashed, color=red, mark=diamond*,very thick]coordinates {
(1,6.5)(2,2.4)(3,1.9)(4,3.0)(5,1.9)
};\addlegendentry{Latency}

\end{axis}
\end{tikzpicture}
\end{adjustbox}
% \caption{Insights per stage}
\label{fig:insights_depth_stage}
\end{figure}

%% file: insights_stages.tex
\begin{figure}[H]
\centering
\begin{adjustbox}{width=1\textwidth}
\begin{tikzpicture}
\begin{axis}[
            axis x line=left,
            axis y line=left,
            xmajorgrids=true,
            ymajorgrids=true,
            grid=both,
            xlabel style={below=1ex},
            enlarge x limits,
            ymin = 0.0,
            ymax = 0.22,
            xmin = 1.0,
            xmax = 5.0,
            % xmax = 90.0,
            % x tick label style={ %font=\fontsize{4}{6}\selectfont, text width=0.5cm},
            xtick = {1, 2,..., 5},
            ytick = {0.05, 0.1,...,0.22},
            ylabel = One-shot Accuracy Contribution (\%),
            xlabel = Stage Number,
            legend style={nodes={scale=0.8, transform shape}},
            legend pos=north west,
    ]

\addplot[color=orange, mark=triangle*,very thick]coordinates {
(1,0.04170036315917969)(2,0.029999256134033203)(3,0.029224872589111328)(4,0.07246255874633789)(5,0.1118006706237793)
};\addlegendentry{Increase er from 3 to 4}

\addplot[color=RedOrange, mark=diamond*,very thick]coordinates {
(1,0.06004810333251953)(2,0.060211181640625)(3,0.08327674865722656)(4,0.14120054244995117)(5,0.2086467742919922)
};\addlegendentry{Increase er from 4 to 6}

\addplot[color=blue, mark=diamond*,very thick]coordinates {
(1,0.008801142374674479)(2,0.025275230407714844)(3,0.03856658935546875)(4,0.0424502690633138)(5,0.0012760162353515625)
};\addlegendentry{Increase k from 3x3 to 5x5}

\addplot[color=YellowGreen, mark=triangle*,very thick]coordinates {
(1,0.024433135986328125)(2,0.03580697377522787)(3,0.05210049947102864)(4,0.16463375091552734)(5,0.20270792643229166)
};\addlegendentry{Add S\&E}

\end{axis}
\end{tikzpicture}
\end{adjustbox}
% \caption{Insights per stage}
\label{fig:insights_stage}
\end{figure}

%% file: insights_blocks.tex
\begin{figure}[H]
\centering
\begin{adjustbox}{width=1\textwidth}
\begin{tikzpicture}
\begin{axis}[
            axis x line=left,
            axis y line=left,
            xmajorgrids=true,
            ymajorgrids=true,
            grid=both,
            xlabel style={below=1ex},
            enlarge x limits,
            ymin = 0.0,
            ymax = 0.22,
            xmin = 1.0,
            xmax = 4.0,
            % xmax = 90.0,
            % x tick label style={ %font=\fontsize{4}{6}\selectfont, text width=0.5cm},
            xtick = {1, 2,..., 4},
            ytick = {0.05, 0.1,...,0.2},
            ylabel = One-shot Accuracy Contribution (\%),
            xlabel = Block Number,
            legend style={nodes={scale=0.8, transform shape}},
            legend pos=north east,
    ]
\addplot[color=orange, mark=triangle*,very thick]coordinates {
(1,0.0805816650390625)(2,0.09049797058105469)(3,0.04370002746582031)(4,0.013370513916015625)
};%\addlegendentry{Increase er from 3 to 4}

\addplot[color=RedOrange, mark=diamond*,very thick]coordinates {
(1,0.16334991455078124)(2,0.1725605010986328)(3,0.07890777587890624)(4,0.02788848876953125)
};%\addlegendentry{Increase er from 4 to 6}

\addplot[color=blue, mark=diamond*,very thick]coordinates {
(1,0.027814229329427086)(2,0.04351882934570313)(3,0.0037096659342447918)(4,0.018052673339843752)
};%\addlegendentry{Increase k from 3x3 to 5x5}

\addplot[color=YellowGreen, mark=triangle*,very thick]coordinates {
(1,0.09007975260416667)(2,0.2050000508626302)(3,0.0879992167154948)(4,0.0006668090820312498)
};%\addlegendentry{Add S\&E}

\end{axis}
\end{tikzpicture}
\end{adjustbox}
% \caption{Insights per block}
\label{fig:insights_block}
\end{figure}

%% file: conclusion.tex
\vspace{-2mm}
\section{Conclusion} \label{Conclusion}
\vspace{-2mm}
The problem of resource-aware NAS is formulated as an IQCQP optimization problem. Bilinear constraints express resource requirements and a bilinear accuracy estimator serves as the objective function. This estimator is constructed by measuring the individual contribution of design choices, which makes it intuitive and interpretable. Indeed, its interpretability brings several insights and design rules. Its performance is comparable to complex predictors that are more expensive to acquire and harder to optimize. Efficient optimization algorithms are proposed for solving the resulted IQCQP problem. BINAS is a faster search method, scalable to many devices and requirements, while generating comparable or better architectures than those of other state-of-the-art NAS methods.

% \section{Reproducibility Statement}
% In this paper, we have made efforts to produce explanations and present clear algorithms describing every step of the different components of our method. In addition, a source code will be released upon publication that will allow to reproduce every result presented in the paper. In the appendix, we present an overview of our method and its computational cost (Appendix~\ref{apdx:overview}), a clear description of our search space (Appendix~\ref{apdx:search_space}), the experimental setting with training procedure details to obtain our results (Appendix~\ref{apdx:reproduce}), as well as proof of Theorem 3.1 (Appendix~\ref{apdx:proof_estimator}), and the concise description of the algorithm we used to find the closed form regularized solution of the linear regression problems described in section~\ref{sec:quad_predictors} (Appendix~\ref{apdx:closed_form_sol}), to ensure full reproducibility. In addition, we present more theoretical results related to the convergence of Algorithm~\ref{alg:BCFW_QCQP} and the sparsity of it solution (Appendix~\ref{apdx:proof_convergence_bcfw}, \ref{apdx:proof_sparsity_bcfw}, \ref{apdx:kt_transitivity}). 

%% file: supplementary.tex
\newpage
\begin{center}
    \huge
    Appendix
\end{center}
\appendix

% \section{More Details on Differences with Previous Methods}
% Our methods differs from previous method in several ways. Here we clarify some of the key differences and advantage compared to previous methods.

% There are two key differences compared to FBNet:
% \begin{enumerate}
%     \item  The latency formulation itself is different - our formulation includes $\beta$ parameters for specifying the depth of each stage.
%     \item The utilization of the formula is different. While we incorporated it as a hard constraint, FBNet utilizes it through a multiplicative soft penalty with hyperparameters. Our approach leads to more accurate latency of our models and improved accuracy. 
% \end{enumerate}

% There are three advantages of our method over evolutionary algorithms (EA), as done for example in OFA.
% \begin{enumerate}
% \item The differentiable space we propose enables the utilization of our multipath sampling for efficiently training the supernet in 400 GPU hours compared to 1200 of OFA. 
% \item Our method is based on a more principled differentiable constrained optimization, requires fewer heuristics and leads to superior results. \item EA relies on inaccurate and expensive accuracy predictors, e.g., OFA samples 16K subnetworks and measures their accuracy on 10K validation images, hence, training the predictor costs 40 GPU hours.
% \end{enumerate}

\section{An Overview of the Method and Computational Costs}\label{apdx:overview}
Figure~\ref{fig:iqnas_overview} presents and overview scheme of the method:
\begin{figure}[htb]
    \centering
        \includegraphics[width=0.85\textwidth]{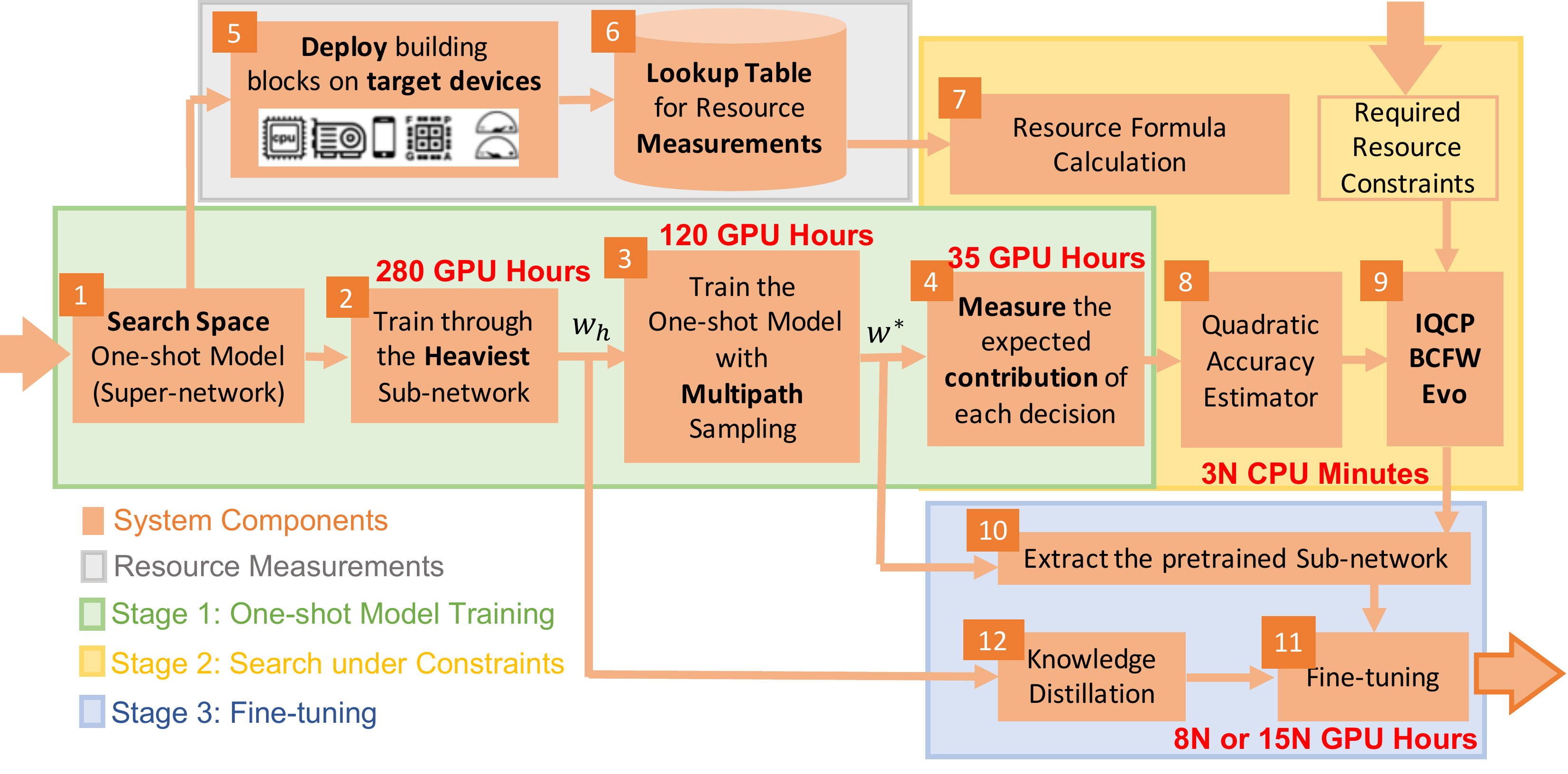}
    \caption{An Overview scheme of the IQNAS method with computational costs}
    \label{fig:iqnas_overview}
    \vspace{5mm}
\end{figure}

The search space, latency measurements and formula, supernetwork training and fine-tuning blocks (1,2,3,5,6,7,10,11,12) are identical to those introduced in HardCoRe-NAS:

We first train for 250 epochs a one-shot model $w_h$ using the heaviest possible configuration, i.e., a depth of $4$ for all stages, with $er=6, k=5\times 5, se=on$ for all the blocks. Next, to obtain $w^*$, for additional 100 epochs of fine-tuning $w_h$ over 80\% of a 80-20 random split of the ImageNet train set~\cite{imagenet_cvpr09}. The training settings are specified in appendix~\ref{apdx:reproduce}. 
The first 250 epochs took 280 GPU hours and the additional 100 fine-tuning epochs took 120 GPU hours, both Running with a batch size of 200 on 8$\times$NVIDIA V100, summing to a total of 400 hours on NVIDIA V100 GPU to obtain both $w_h$ and $w^*$.

A significant benefit of this training scheme is that it also shortens the generation of trained models. The common approach of most NAS methods is to re-train the extracted sub-networks from scratch.
Instead, we follow HardCoRe-NAS and leverage having two sets of weights: $w_h$ and $w^*$. Instead of retraining the generated sub-networks from a random initialization we opt for fine-tuning $w^*$ guided by knowledge distillation~\cite{KD} from the heaviest model $w_h$.
Empirically, as shown in figure~\ref{fig:acc_vs_lat} by comparing the dashed red line with the solid one, we observe that this surpasses the accuracy obtained when training from scratch at a fraction of the time: 8 GPU hours for each generated network.

The key differences from HardCoRe-NAS reside in the latency constrained search blocks (4,8,9 of figure~\ref{fig:iqnas_overview})
that have to do with constructing the quadratic accuracy estimator of section~\ref{sec:quad_estimator} and solving the IQCQP problem. The later requires only several minutes on CPU for each generated netowork (compared to 7 GPU hours of HardCoRe-NAS), while the former requires to measure the individual accuracy contribution of each decision. Running a validation epoch to estimate $\mathbb{E}[Acc]$ and also for each decision out of all $255$ entries in the vector $\zeta$ to obtain $\Delta_{b,c}^s$ and $\Delta_b^s$ requires 256 validation epochs in total that last for 3.5 GPU hours. Figure~\ref{fig:trio} (Left) shows that reducing the variance by taking 10 validation epochs per measurement is beneficial. Thus a total of 2560 validation epochs requires 35 GPU hours only once.

Overall, we are able to generate a trained model within a small marginal cost of 8 GPU hours. The total cost for generating $N$ trained models is $435 + 8N$, much lower than the $1200 + 25N$ reported by OFA~\cite{OFA} and more scalable compared to the $400+15N$ reported by HardCoRe-NAS. See Table~\ref{tab:exp}. The reduced search cost frees more compute that can be utilized for a longer fine-tuning of 15 GPU hours to surpass HardCoRe-NAS with the same marginal cost of $15N$. 
This makes our method scalable for many devices and latency requirements. 

\section{More Specifications of the Search Space}\label{apdx:search_space}
Inspired by EfficientNet~\cite{effnet} and TF-NAS~\cite{TF-NAS}, HardCoRe-NAS~\cite{nayman2021hardcore} builds a layer-wise search space that we utilize, as explained in Section~\ref{sec:search_space} and detailed in 
% Figure~\ref{fig:super_net} and in
Table~\ref{tab:search_space_specs}.
The input shapes and the channel numbers are the same as EfficientNetB0. Similarly to TF-NAS and differently from EfficientNet-B0, we use ReLU in the first three stages. As specified in Section~\ref{sec:search_space}, the ElasticMBInvRes block is the \textit{elastic} version of the MBInvRes block as in HardCoRe-NAS, introduced in~\cite{mobilenetv2}. Those blocks of stages 3 to 8 are to be searched for, while the rest are fixed.
~
\\
~
\begin{table}[h]
\begin{subtable}[t]{0.7\textwidth}
\centering
% \resizebox{0.45\textwidth}{!}{
\begin{tabular}{c|c|c|c|c|c} \hline
  Stage & Input & Operation & $C_{out}$ & Act & b \\
  \hline
  1  & $224^2\times3$  & $3\times3$ Conv & 32 & ReLU & 1 \\
  2  & $112^3\times32$ & MBInvRes & 16 & ReLU & 1 \\
  3  & $112^2\times16$ & ElasticMBInvRes & 24  & ReLU  & $[2, 4]$ \\
  4  & $56^2\times24$  & ElasticMBInvRes & 40  & Swish & $[2, 4]$ \\
  5  & $28^2\times40$  & ElasticMBInvRes & 80  & Swish & $[2, 4]$ \\
  6  & $14^2\times80$  & ElasticMBInvRes & 112 & Swish & $[2, 4]$ \\
  7  & $14^2\times112$ & ElasticMBInvRes & 192 & Swish & $[2, 4]$ \\
  8  & $7^2\times192$  & ElasticMBInvRes & 960 & Swish & 1 \\
  9  & $7^2\times960$  & $1\times1$ Conv & 1280 & Swish & 1 \\
  10 & $7^2\times1280$ & AvgPool & 1280 & - & 1 \\
  11 & $1280$ & Fc & 1000 & - & 1 \\
  \hline
  \end{tabular}
%   }
  \caption{Macro architecture of the one-shot model.}
  \label{tab:search_space}
  \end{subtable}
  \hfill
\begin{subtable}[t]{0.25\textwidth}
% \vspace{-1.5mm}
    \centering
    \begin{tabular}{|c||c|c|c|}
    \hline
    c & er & k & se \\
    \hline
    1 & 2 & $3\times 3$ & off \\
    2 & 2 & $3\times 3$ & on \\
    3 & 2 & $5\times 5$ & off \\
    4 & 2 & $5\times 5$ & on \\
    5 & 3 & $3\times 3$ & off \\
    6 & 3 & $3\times 3$ & on \\
    7 & 3 & $5\times 5$ & off \\
    8 & 3 & $5\times 5$ & on \\
    9 & 6 & $3\times 3$ & off \\
    10 & 6 & $3\times 3$ & on \\
    11 & 6 & $5\times 5$ & off \\
    12 & 6 & $5\times 5$ & on \\
    \hline
    \end{tabular}
    \caption{Configurations.}
    \label{tab:configurations}
    \end{subtable}
    \vspace{1.5mm}
    \caption{Search space specifications and indexing. "MBInvRes" is the basic block in~\cite{mobilenetv2}. "ElasticMBInvRes" denotes the \textit{elastic} blocks (Section~\ref{sec:search_space}) to be searched for. "$C_{out}$" stands for the output channels. Act denotes the activation function used in a stage. "b" is the number of blocks in a stage, where $[\underline{b}, \bar{b}]$ is a discrete interval. If necessary, the down-sampling occurs at the first block of a stage. 
    "er" stands for the expansion ratio of the point-wise convolutions, "k" stands for the kernel size of the depth-wise separable convolutions and "se" stands for Squeeze-and-Excitation (SE) with $on$ and $off$ denoting with and without SE respectively. The configurations are indexed according to their expected latency.}
    \label{tab:search_space_specs}
\end{table}

\section{Reproducibility and Experimental Setting}
\label{apdx:reproduce}
In all our experiments we train the networks using SGD with a learning rate of $0.1$, cosine annealing, Nesterov momentum of $0.9$, weight decay of $10^{-4}$, applying label smoothing \cite{label_smoothing} of 0.1, cutout, Autoaugment~\cite{autoaugment}, mixed precision and EMA-smoothing.

The supernetwork is trained following~\cite{nayman2021hardcore} over 80\% of a random 80-20 split of the ImageNet train set. We utilize the remaining 20\% as a validation set for collecting data to obtain the accuracy predictors and for architecture search with latencies of $40,45,50,\dots,60$ and $25, 30, 40$ milliseconds running with a batch size of 1 and 64 on an Intel Xeon CPU and and NVIDIA P100 GPU, respectively.  

The evolutionary search implementation is adapted from \cite{OFA} with a population size of $100$, mutation probability of $0.1$, parent ratio of $0.25$ and mutation ratio of $0.5$. It runs for $500$ iterations, while the   the BCFW runs for $2000$ iterations and its projection step for $1000$ iterations. The MIQCP solver runs up to $100$ seconds with CPLEX default settings.

\section{Proof of Theorem \ref{theo:estimator}}\label{apdx:proof_estimator}
\begin{theo}\label{theo:aux}
Consider $n$ independent random variables $\{X_i\}_{i\in[1,2\ldots,n]}$  conditionally independent with another random variable $A$. Suppose in addition that there exists a positive real number $0 < \epsilon\ll 1$ such that for any given $X_i$, the following term is bounded by above:
$$
\left|\frac{\prob{A=a|X_i=x_i}}{\prob{A=a}}-1\right|<\epsilon .
$$
Then we have that:
\begin{align*}
\expec{A|X_1=x_1,\ldots, X_n=x_n} = \expec{A} + \sum_i \left(\expec{A|X_i=x_i}-\expec{A}\right)(1 +O(n\epsilon))
\label{eqn:main_bound}
\end{align*}

\end{theo}

\begin{proof}
We consider two independent random variables $X, Y$ that are also conditionally independent with another random variable $A$. 
Our purpose is to approximate the following conditional expectation:
$$
\mathbb{E}\left[A|X=x, Y=y\right].
$$
We start by writing :
$$
\prob{A=a|X=x, Y=y}=\frac{\prob{X=x, Y=y|A=a} \prob{A=a}}{\prob{X=x, Y=y}}
$$
Assuming the conditional independence, that is
$$
\prob{X=x, Y=y|A=a} = \prob{X=x|A=a}\prob{Y=y|A=a}, 
$$
we have
\begin{equation}
\begin{disarray}{lll}
\prob{A=a|X=x, Y=y}&=&\frac{\prob{X=x|A=a}\prob{Y=y|A=a}\prob{A=a}}{\prob{X=x} \prob{Y=y}}\\
&=&\frac{\prob{A=a|X=x}\prob{A=a|Y=y}}{\prob{A=a}}
\end{disarray}
\label{eqn:eqn1}
\end{equation}
Next, we assume that the impact on $A$ of the knowledge of $X$ is bounded, meaning that there is a positive real number $0 < \epsilon\ll 1$ such that:
$$
\left|\frac{\prob{A=a|X=x}}{\prob{A=a}}-1\right|<\epsilon.
$$
We have:
$$
\begin{disarray}{lll}
\prob{A=a|X=x}&=&\prob{A=a} + \prob{A=a|X=x} - \prob{A=a} \\
&=&\prob{A=a}\left(1+\underbrace{\left(\frac{\prob{A=a|X=x}}{\prob{A=a}}-1\right)}_{\epsilon_x}\right)
\end{disarray}
$$
Using a similar development for $\prob{A=a|Y=y}$ we also have:
$$
\begin{disarray}{lll}
\prob{A=a|Y=y}&=&\prob{A=a}\left(1+\underbrace{\left(\frac{\prob{A=a|Y=y}}{\prob{A=a}}-1\right)}_{\epsilon_y}\right)
\end{disarray}
$$
Plugging the two above equations in (\ref{eqn:eqn1}), we have:
\begin{align*}
&\prob{A=a|X=x, Y=y}=
\frac{\prob{A=a}(1+\epsilon_x)\prob{A=a}(1+\epsilon_y)}{\prob{A=a}} \\
&= \prob{A=a}(1+\epsilon_x)(1+\epsilon_y) \\
&= \prob{A=a}(1+\epsilon_x+\epsilon_y)+O(\epsilon^2)\\
&= \prob{A=a}\left(1+\left(\frac{\prob{A=a|X=x}}{\prob{A=a}}-1\right)+\left(\frac{\prob{A=a|Y=y}}{\prob{A=a}}-1\right)\right) +O(\epsilon^2)\\
&= \prob{A=a}\left(\frac{\prob{A=a|X=x}}{\prob{A=a}}+\frac{\prob{A=a|Y=y}}{\prob{A=a}}-1\right) +O(\epsilon^2)\\
&= \prob{A=a|X=x} + \prob{A=a|Y=y} - \prob{A=a}+O(\epsilon^2)
\end{align*}
% \label{eqn:eqn2}
Then, integrating over $a$ to get the expectation leads to:
\begin{align*}
\expec{A|X=x, Y=y}=&\int_a a\prob{A=a|X=x, Y=y}da \\
=& \expec{A|X=x} + \expec{A|Y=y} - \expec{A} +O(\epsilon^2) \\
=& \expec{A} + (\expec{A|X=x}- \expec{A}) \\
& + (\expec{A|Y=y} - \expec{A}) +O(\epsilon^2).
% \label{eqn:eqn3}
\end{align*}
In the case of more than two random variables, $X_1, X_2, \ldots, X_n$, denoting by $u_n = \expec{A|X_1=x_1,\ldots, X_n=x_n} - \expec{A}$, and by $v_n = \expec{A|X_n=x_n}-\expec{A}$, we have
$
\left|u_n - u_{n-1} - v_n \right| < \epsilon u_{n-1}
$
A simple induction shows that:
\begin{align*}
v_n + (1-\epsilon)v_{n-1} +& (1-\epsilon)^2v_{n-2} + \cdots + (1-\epsilon)^nv_{1} 
\\<& u_n < v_n + (1+\epsilon)v_{n-1} + (1+\epsilon)^2v_{n-2} + \cdots + (1+\epsilon)^nv_{1}
\end{align*}
Hence:
$$
(1-\epsilon) \sum v_i < u_n < (1+\epsilon)^n \sum v_i
$$
that shows that

$$
u_n = \left(\sum v_i\right)(1 + O(n\epsilon))
$$

% $$
% \expec{A|X_1=x_1,\ldots, X_n=x_n} = \expec{A} + \sum_i \left(\expec{A|X_i=x_i}-\expec{A}\right) +O(\epsilon^2).
% $$
\end{proof}

Now we utilize Theorem~\ref{theo:aux} for proving Theorem~\ref{theo:estimator}.
Consider a one-shot model whose subnetworks' accuracy one wants to estimate: \\$  \mathbb{E}[Acc|\cap_{s=1}^S\cap_{b=1}^D O^s_b, \cap_{s=1}^S d^s]$, 
with $\alpha_{b,c}^s = \mathds{1}_{O^s_b=O_c}$ the one-hot vector specifying the selection of configuration $c$ for block $b$ of stage $s$ and $\beta_b^s=\mathds{1}_{d^s=b}$ the one-hot vector specifying the selection of depth $b$ for stage $s$, such that $(\balpha,\bbeta)\in\mathcal{S}$ with $\mathcal{S}$ specified in \eqref{eqn:integer_search_space} as described in section~\ref{sec:search_space}.

We simplify our problem and assume $\{O^s_b, d^s\}, d^s$ for $s=1,\dots,S$ and $b=1,\dots,D$ are conditionally independent with the accuracy $Acc$.
% assumption that the couples $(\alpha_{b}^s,\beta^s)$ and the $\beta^s$ are conditionally independent with respect to the accuracy. 
In our setting, we have 
\begin{align}
    \mathbb{E}
    % \left[
    [Acc 
    % \middle
    \vert 
    \cap_{s=1}^S&\cap_{b=1}^{D} O^s_b ;
    \cap_{s=1}^S d^s  ]
%   \right] 
  = \label{eqn:masking}
    \mathbb{E}
    % \left[ 
    [Acc 
    % \middle
    \vert 
    \cap_{s=1}^S\cap_{b=1}^{d^s} \{O^s_b, d^s\} ;
    \cap_{s=1}^S d^s  ]
    % \right] 
    %   \expec{Acc|O_1^1, O_2^1,\ldots, O_D^S, d^1, \cdots,d^S}
    % \\ 
    % &= \expec{Acc|(O_1^1, d^1), (O_2^1, d^1),\ldots, (O_D^S, d^S), d_1, \cdots,d_S}
    % \nonumber
    \\ 
    &\approx \notag
    \expec{Acc} 
    \\&+ \label{eqn:utilize_theo_alpha}
    \sum_{s=1}^S\sum_{b=1}^{D} \left(\expec{Acc|O_b^s,d^s}-\expec{Acc}\right)
    \mathds{1}_{b \leq d_s} \left(1+O(N\epsilon)\right)
    \\ 
    &+\sum_{s=1}^S \left(\expec{Acc|d^s}-\expec{Acc}\right)
    \mathds{1}_{b \leq d^s} \left(1+O(N\epsilon)\right)
    \label{eqn:predictor}
\end{align}
where \eqref{eqn:masking} is since the accuracy is independent of blocks that are not participating in the subnetowrk, i.e. with $b>d^s$, and equations \ref{eqn:utilize_theo_alpha} and \ref{eqn:predictor} are by utilizing Theorem~\ref{theo:aux}.

Denote by $b^s$ and $c_b^s$ to be the single non zero entries of $\beta^s$ and $\alpha^s_{b}$ respectively, whose entries are $\beta^s_b$ for $b=1,\dots,D$ and $\alpha^s_{b,c}$ for $c\in\mathcal{C}$ respectively. Hence $\beta^s_b=\mathds{1}_{b=b^s}$ and $\alpha^s_{b,c}=\mathds{1}_{c=c^s_b}$.
% In addition, since $\bbeta^s$ is a one hot vector, and only a single index is activated, we have:
Thus we have,
\begin{align} \label{eqn:adapt_beta}
\expec{Acc|d^s=b^s}-\expec{Acc}
&=
\sum_{b=1}^D \mathds{1}_{b=b^s}(\expec{Acc|d^s=b}-\expec{Acc})
\notag\\&=
\sum_{b=1}^D \beta^s_b(\expec{Acc|d^s=b}-\expec{Acc})
=
\sum_{b=1}^D \beta^s_b\Delta^s_b 
\end{align}

% Noticing that $\mathds{1}_{b \leq d_s}=\sum_{b'=b}^{D} \beta^s_{b'}$, the same reasoning leads to write that:
Similarly,  
\begin{align}
\expec{Acc|O_b^s=O_{c^s_b},d_s=b}-\expec{Acc}
&=
\sum_{c\in\mathcal{C}} \mathds{1}_{c=c^s_b}(\expec{Acc|O_b^s=O_{c},d_s=b}-\expec{Acc})
\notag\\&=
\sum_{c\in\mathcal{C}} \alpha^s_{b,c}(\expec{Acc|O_b^s=O_{c},d_s=b}-\expec{Acc})
\notag \\&=  \label{eqn:adapt_alpha}
\sum_{c\in\mathcal{C}}\alpha^s_{b,c}\Delta^s_{b,c}
\end{align}
% And since $\mathds{1}_{b \leq d^s}=\sum_{b'=b}^{D} \mathds{1}_{b'=d^s}=\sum_{b'=b}^{D} \beta^s_{b'}$ we have,
And since effectively $\mathds{1}_{b \leq d^s}=\sum_{b'=b}^{D} \beta^s_{b'}$ we have,
\begin{align}
% \expec{Acc|O_b^s=\text{index}(\alpha^s_b),d_s=\text{index}(\beta^s)}-\expec{Acc}&=\sum_{c\in\mathcal{C}} \alpha^s_{b,c}\cdot 
%     (\expec{Acc|O_b^s=O_c, d_s=b}-\expec{Acc})\cdot  \left(\sum_{b'=b}^{D}\beta^s_{b'}\right)
\left(\expec{Acc|O_b^s=O_{c^s_b},d_s=b}-\expec{Acc}\right)
\mathds{1}_{b \leq d_s}
&=
\sum_{c\in\mathcal{C}}\alpha^s_{b,c}\Delta^s_{b,c}
\cdot  \mathds{1}_{b \leq d_s}
% \\
% &=\sum_{c\in\mathcal{C}} \alpha^s_{b,c}\cdot 
    % (\expec{Acc|O_b^s=O_c, d_s=b}-\expec{Acc})\cdot  \left(\sum_{b'=b}^{D} \beta^s_{b'}\right)
\notag \\ &= \label{eqn:adapt_alpha_final}
\sum_{b'=b}^{D}\sum_{c\in\mathcal{C}} \alpha^s_{b,c}\cdot 
    \Delta_{b,c}^s \cdot  \beta^s_{b'}
\end{align}
% where $\text{id}(v)$ represents the index of the non zero entry of one hot vector $v$. 
Finally by setting equations \ref{eqn:adapt_beta} and \ref{eqn:adapt_alpha_final} into \ref{eqn:utilize_theo_alpha} and \ref{eqn:predictor} respectively, we have,
% \begin{align*}
% %   ACC(\zeta)=ACC(\balpha, \bbeta)=
%   \mathbb{E}\left[Acc \middle\vert 
%   \begin{array}{c}
%     \cap_{s=1}^S\cap_{b=1}^D O^s_b  \\
%     \cap_{s=1}^S d^s  
%   \end{array}
%   \right]  
%   = \mathbb{E}[Acc] &+
%     %   \sum_{s=1}^S \sum_{c\in\mathcal{C} \alpha^s_{1,c}\cdot \Delta_{1,c}^s
%     % \left(
%     % E[A | O^s_{b=1}=O_c]
%     % -
%     % E[A]\right)
% %  \\&+ 
%     % +
%       \left(\sum_{s=1}^S \sum_{b=1}^D \beta^s_b\cdot 
%     \Delta_b^s
%     % \left(E[A\mid d_s=b] 
%     % -
%     % E[A]\right)
% % \\&+ 
%     +
%       \sum_{s=1}^S \sum_{b=1}^D \sum_{b'=b}^{D}\sum_{c\in\mathcal{C}} \alpha^s_{b,c}\cdot 
%     \Delta_{b,c}^s
%     % \left(
%     % E[A | O^s_b=O_c]
%     % -
%     % E[A ]
%     % \right)
%     \cdot\beta^s_{b'}\right)\left(1
%     +\mathcal{O}(N\epsilon)\right)
% \end{align*}
\begin{align*}
%   ACC(\zeta)=ACC(\balpha, \bbeta)=
  \mathbb{E}\left[Acc \middle\vert   
    \cap_{s=1}^S\cap_{b=1}^D O^s_b, 
    \cap_{s=1}^S d^s  
%   \begin{array}{c}
%     \cap_{s=1}^S\cap_{b=1}^D O^s_b  \\
%     \cap_{s=1}^S d^s  
%   \end{array}
  \right]  
     &= \mathbb{E}[Acc] 
%   +
    %   \sum_{s=1}^S \sum_{c\in\mathcal{C} \alpha^s_{1,c}\cdot \Delta_{1,c}^s
    % \left(
    % E[A | O^s_{b=1}=O_c]
    % -
    % E[A]\right)
%  \\ \notag &+ 
    +
    \left(1
    +\mathcal{O}(N\epsilon)\right)\cdot
      \sum_{s=1}^S \sum_{b=1}^D \beta^s_b\cdot 
    \Delta_b^s
    % \left(E[A\mid d_s=b] 
    % -
    % E[A]\right)
\\ \notag &+ 
    % +
    \left(1
    +\mathcal{O}(N\epsilon)\right)\cdot
      \sum_{s=1}^S \sum_{b=1}^D \sum_{b'=b}^{D}\sum_{c\in\mathcal{C}} \alpha^s_{b,c}\cdot 
    \Delta_{b,c}^s
    % \left(
    % E[A | O^s_b=O_c]
    % -
    % E[A ]
    % \right)
    \cdot\beta^s_{b'}
\end{align*}
% Any network can be extracted from the parameters characterizing it. We first needs to fix the number of layers, and then, for every active layer, the chosen operations. Meaning that we set all the $\beta^s$ and then associated $\alpha_{b}^s$ such as $b$ is smaller or equal to the active index of the one hot vector assigned to $\beta^s$. We define as $\mathds{1}_i$ a one-hot vector whose $\text{i}^th$ coordinate is $1$ while other coordinates are $0$. 
% Lets define as 
% \begin{equation}
%     v_{s,b',b,c}=\expec{A|\alpha_{b'}^s=\mathds{1}_{c},\beta^s=\mathds{1}_{b}}
%     \label{eqn:precomputation}
% \end{equation}
% One can see that given the vectors defining our architecture $\alpha^s_b, \beta^s$, we have:

% \begin{equation}
% \hat{\mathcal{A}} =\expec{A}+ \sum_{s,b,c}\sum_{b'\leq b}\left(\alpha^s_{b'c}\beta^s_bv_{s,b',b,c}-\expec{A}\right)
% \label{eqn:predictor2}
% \end{equation}

\section{Deriving a Closed Form Solution for a Linear Regression}\label{apdx:closed_form_sol}
We are given a set $(X,Y)$ of architecture encoding vectors and their accuracy measured on a validation set.

We seek for a quadratic predictor $f$ defined by parameters $\mathbf{Q}\in\mathbb{R}^{n\times n},\mathbf{a}\in\mathbb{R}^n,b\in\mathbb{R}$ such as 
$$
f(\mathbf{x})=\mathbf{x}^T\mathbf{Q}\mathbf{x}+\mathbf{a}^T\mathbf{x}+b
$$

Our purpose being to minimise the MSE over a training-set $X_\text{train}$, we seek to minimize:

\begin{equation}
\min_{\mathbf{Q},\mathbf{a},b} \sum_{(\mathbf{x},\mathbf{y})\in (X_\text{train},Y_\text{train})} \|\mathbf{x}^T\mathbf{Q}\mathbf{x}+\mathbf{x}^T\mathbf{a}+b-\mathbf{y}\|^2
\label{eqn:minreg}
\end{equation}
We also have that 
$$
\mathbf{x}^T\mathbf{Q}\mathbf{x}=\operatorname{trace}\left(\mathbf{Q}\mathbf{x}\mathbf{x}^T\right)
$$
Denoting by $\mathbf{q}$ the column-stacking of $\mathbf{Q}$, the above expression can be expressed as:
$$
\mathbf{x}^T\mathbf{Q}\mathbf{x}=\operatorname{trace}\left(\mathbf{Q}\mathbf{x}\mathbf{x}^T\right)=\mathbf{q}\left(\mathbf{x}\otimes\mathbf{x}\right)
$$
where $\otimes$ denotes the Kronecker product.
Hence, \eqref{eqn:minreg} can be expressed as:
\begin{equation}
\min_{\mathbf{q},\mathbf{a},b} \sum_{(\mathbf{x},y)\in (X_\text{train},Y_\text{train})} \|\left(\mathbf{x},\mathbf{x}\otimes\mathbf{x}\right)^T(\mathbf{a},\mathbf{q})+b-y\|^2.
\label{eqn:minreg2}
\end{equation}
Denoting by $\tilde{\mathbf{x}}=\left(\mathbf{x},\mathbf{x}\otimes\mathbf{x}\right)$ and $\mathbf{v}=(\mathbf{a},\mathbf{q})$, we are led to a simple regression problem:

\begin{equation}
\min_{\mathbf{v},b} \sum_{(\tilde{\mathbf{x}},\mathbf{y})\in (\tilde{X}_\text{train},Y_\text{train})} \|\tilde{\mathbf{x}}^T\mathbf{v}+b-y\|^2.
\label{eqn:minreg3}
\end{equation}
We rewrite the objective function of \eqref{eqn:minreg3} as:
$$
(\tilde{\mathbf{x}}^T\mathbf{v}+b-\mathbf{y})^T(\tilde{\mathbf{x}}^T\mathbf{v}+b-y)=\mathbf{v}^T\tilde{\mathbf{x}}\tilde{\mathbf{x}}^T\mathbf{v}+2(b-y)\tilde{\mathbf{x}}^T\mathbf{v}+(b-y)^2
$$
Stacking the $\tilde{\mathbf{x}}, \mathbf{y}$ in matrices $\tilde{\mathbf{X}}, \mathbf{Y}$, the above expression can be rewritten as:
\begin{align}
\mathbf{v}^T\tilde{\mathbf{X}}\tilde{\mathbf{X}}^T\mathbf{v}+2(b\mathds{1}-\mathbf{Y})^T\tilde{\mathbf{X}}^T\mathbf{v}+(b\mathds{1}-\mathbf{Y})^T(b\mathds{1}-\mathbf{Y})
\label{eqn:minreg4}
\end{align}
Deriving with respect to $b$ leads to:
$2\mathds{1}^T\tilde{\mathbf{X}}^T\mathbf{v}+2nb-2\mathds{1}^T\mathbf{Y}=0$
Hence:
$$
b=\frac{1}{n}\left(\mathds{1}^T(\mathbf{Y}-\tilde{\mathbf{X}}^T\mathbf{v})\right)=\frac{1}{n}\left((\mathbf{Y}-\tilde{\mathbf{X}}^T\mathbf{v})^T\mathds{1}\right)
$$
We hence have
\begin{align*}
(b\mathds{1}-\mathbf{Y})^T(b\mathds{1}-\mathbf{Y}) &= nb^2-2b\mathds{1}^T\mathbf{Y} +\mathbf{Y}^T\mathbf{Y}\\
&=\frac{1}{n}\left(\mathbf{v}^T\tilde{\mathbf{X}}\mathds{1}\mathds{1}^T\tilde{\mathbf{X}}^T\mathbf{v}-2\mathbf{Y}^T\mathds{1}\mathds{1}^T
\tilde{\mathbf{X}}^T\mathbf{v}+2\mathbf{Y}^T\mathds{1}\mathds{1}^T
\tilde{}^T\mathbf{v}\right)\\
&=\frac{1}{n}\mathbf{v}^T\tilde{\mathbf{X}}\mathds{1}\mathds{1}^T\tilde{\mathbf{X}}^T\mathbf{v}
\end{align*}
In addition:
\begin{align*}
(b\mathds{1}-\mathbf{Y})^T\tilde{\mathbf{X}}^T\mathbf{v}&=\frac{1}{n}\left(\mathbf{Y}^T\mathds{1}\mathds{1}^T
\tilde{\mathbf{X}}^T\mathbf{v}-\mathbf{v}^T\tilde{\mathbf{X}}\mathds{1}\mathds{1}^T\tilde{\mathbf{X}}^T\mathbf{v}\right)-\mathbf{Y}^T\tilde{\mathbf{X}}^T\mathbf{v}
\end{align*}
Hence, \eqref{eqn:minreg4} can be rewritten as:
$$
\mathbf{v}^T\tilde{\mathbf{X}}\tilde{\mathbf{X}}^T\mathbf{v}-\frac{1}{n}\mathbf{v}^T\tilde{\mathbf{X}}\mathds{1}\mathds{1}^T\tilde{\mathbf{X}}^T\mathbf{v}+\mathbf{Y}^T\left(\textbf{Id}-\frac{1}{n}\mathds{1}\mathds{1}^T\right)\tilde{\mathbf{X}}^T\mathbf{v}
$$
Denoting by $\mathbf{I}=\left(\textbf{Id}-\frac{1}{n}\mathds{1}\mathds{1}^T\right)$, $\hat{\mathbf{X}}=\mathbf{I}\tilde{\mathbf{X}}^T$ and by $\hat{\mathbf{Y}}=\mathbf{I}\mathbf{Y}$, and noticing that $\mathbf{I}^T\mathbf{I}=\mathbf{I}$, we then have:
$$
\hat{\mathbf{X}}^T\hat{\mathbf{X}}v=\hat{\mathbf{X}}^T\hat{\mathbf{Y}}
$$
To solve this problem we can find am SVD decomposition of $\hat{\mathbf{X}}=\mathbf{U}\mathbf{D}\mathbf{V}^T$, hence:

$$
\mathbf{V}\mathbf{D}^2\mathbf{V}^Tv=\mathbf{V}\mathbf{D}\mathbf{U}^T\hat{\mathbf{Y}}
$$
that leads to:
$$
v=\mathbf{V}\mathbf{D}^{-1}\mathbf{U}^T\hat{\mathbf{Y}}
$$

The general algorithm to find the decomposition is the following:
\input{closed_form_reg_algo}

In order to choose the number $k$ of principal components described in the above algorithm, we can perform a simple hyper parameter search using a test set. In the below figure, we plot the Kendall-Tau coefficient and MSE of a quadratic predictor trained using a closed form regularized solution of the regression problem as a function of the number of principal components $k$, both on test and validation set. We can see that above 2500 components, we reach a saturation that leads to a higher error due to an over-fitting on the training set. Using 1500 components leads to a better generalization. The above scheme is another way to regularize a regression and, unlike Ridge Regression, can be used to solve problems of very a high dimesionality without the need to find the pseudo inverse of a high dimensional matrix, without using any optimization method, and with a relatively robust discrete unidimensional parameter that is easier to tune. 

% \begin{minipage}[b]{0.7\textwidth}
% \begin{center}
    \input{kt_mse_vs_num_components}
% \end{center}
% \end{minipage}

\section{Convergence Guarantees for Solving BLCP with BCFW with Line-Search}\label{apdx:proof_convergence_bcfw}
In this section we proof Theorem~\ref{theo:convergence}, guaranteeing that after $\mathcal{O}(1/\epsilon)$ many iterations, Algorithm~\ref{alg:BCFW_QCQP} obtains an $\epsilon$-approximate solution to problem \ref{eqn:NAS_QCQP}.

\subsection{Convergence Guarantees for a General BCFW over a Product Domain}\label{apdx:prod_domain}
The proof is heavily based on the convergence guarantees provided by \cite{BCFW} for solving:
\begin{align}\label{eqn:prod_domain}
    \min_{\zeta \in \mathcal{M}^{(1)}\times\dots\times\mathcal{M}^{(n)}} f(\zeta)
\end{align}
with the BCFW algorithm~\ref{alg:plain_BCFW}, where $\mathcal{M}^{(i)}\subset\R^{m_i}$ is the convex and compact domain of the $i$-th coordinate block and the product $\mathcal{M}^{(1)}\times\dots\times\mathcal{M}^{(n)}\subset\R^m$ specifies the whole domain, as $\sum_{i=1}^n m_i=m$. $\zeta^{(i)}\in\R^{m_i}$ is the $i$-th coordinate block of $\zeta$ and $\zeta^{\setminus(i)}$ is the rest of the coordinates of $\zeta$. $\nabla^{(i)}$ stands for the partial derivatives vector with respect to the $i$-th coordinate block.

\input{plain_bcfw_algo}

The following theorem shows that after $\mathcal{O}(1/\epsilon)$ many iterations, Algorithm~\ref{alg:plain_BCFW} obtains an $\epsilon$-approximate solution to problem \ref{eqn:prod_domain}, and guaranteed $\epsilon$-small duality gap.

 \begin{theo}\label{theo:plain_convergence}
    For each $k>0$ the iterate $\zeta_k$ Algorithm~\ref{alg:plain_BCFW} satisfies:
    $$E[f(\zeta_k)] - f(\zeta^*) \leq \frac{2n}{k+2n}\left(C_f^\otimes + \left(f(\zeta_0) - f(\zeta^*)\right)\right)$$
    where $\zeta^*$ is the solution of problem \ref{eqn:prod_domain} and the expectation is over the random choice of the block $i$ in the steps of the algorithm.
    
    Furthermore, there exists an iterate $0\leq \hat{k}\leq K$ of Algorithm~\ref{alg:plain_BCFW}  with a duality gap bounded by $E[g(\zeta_{\hat{k}})]\leq \frac{6n}{K+1}\left(C_f^\otimes + \left(f(\zeta_0) - f(\zeta^*)\right)\right)$.
    \end{theo}

Here the duality gap $g(\zeta)\geq f(\zeta)-f(\zeta^*)$ is defined as following:
\begin{align}
    g(\zeta)=\max_{s\in\mathcal{M}^{(1)}\times\dots\times\mathcal{M}^{(n)}}\left(\zeta-s\right)^T\cdot\nabla f(\zeta)
\end{align}
and the \textit{global product curvature} constant $C_f^\otimes=\sum_{i=1}^n C_f^{(i)}$ is the sum of the (partial) curvature constants of $f$ with respect to the individual domain $\mathcal{M}^{(i)}$:
\begin{align}\label{eqn:curvature_const}
    C_f^{(i)}=\sup_{
    \begin{array}{c}
    x\in\mathcal{M}^{(1)}\times\dots\times\mathcal{M}^{(n)},
    \\
    s^{(i)}\in\mathcal{M}^{(i)}, \gamma\in[0,1],
    \\
    y^{(i)}=(1-\gamma)x^{(i)}+\gamma s^{(i)},
    \\
    y^{\setminus(i)} = x^{\setminus(i)}
    \end{array}
    }
    \frac{2}{\gamma^2}\left(f(y)-f(x)-\left(y^{(i)}-x^{(i)}\right)^T\nabla^{(i)}f(x)\right)
\end{align}
which quantifies the maximum relative deviation of the objective function $f$ from its linear approximations, over the domain $\mathcal{M}^{(i)}$.

The proof of theorem~\ref{theo:plain_convergence} is given in \cite{BCFW}.

\subsection{Analytic Line-Search for Bilinear Objective Functions}\label{apdx:analytic_line_search}
The following theorem provides a trivial analytic solution for the line-search of algorithm~\ref{alg:plain_BCFW} (line 5) where the objective function has a bilinear form.
% Consider a bilinear objective function to problem~\ref{eqn:prod_domain} of the following form:
% \begin{align}\label{eqn:block_quad_obj_prod}
%     f(\zeta)=\sum_{i=1}^n \left(\zeta^{(i)}\right)^T\cdot p_f^{(i)}  + \sum_{i=1}^n\sum_{j=i}^n \left(\zeta^{(i)}\right)^T\cdot Q_f^{(i,j)}\cdot\zeta^{(j)}
% \end{align}
% with $p_f^{(i)}\in\R^{m_i}$ and $Q_f^{(i,j)}\in\R^{m_i\times m_j}$.

\begin{theo}\label{theo:analytic_line_search}
The analytic solution of the line-search step of algorithm~\ref{alg:plain_BCFW} (line 5) with a bilinear objective function of the form:
\begin{align}\label{eqn:block_quad_obj_prod}
    f(\zeta)=\sum_{i=1}^n \left(\zeta^{(i)}\right)^T\cdot p_f^{(i)}  + \sum_{i=1}^n\sum_{j=i}^n \left(\zeta^{(i)}\right)^T\cdot Q_f^{(i,j)}\cdot\zeta^{(j)}
\end{align}
with $p_f^{(i)}\in\R^{m_i}$ and $Q_f^{(i,j)}\in\R^{m_i\times m_j}$, reads $\gamma\equiv 1$ at all the iterations.
\end{theo}

\begin{proof}
In each step of algorithm~\ref{alg:plain_BCFW} at line 3, a linear program is solved:
\begin{align}\label{eqn:line_search_minizer}
    s &= 
    \argmin_{s'\in\mathcal{M}^{(i)}}\nabla^{(i)}f(\zeta)^T\cdot s'
    \\ &=
    \argmin_{s'\in\mathcal{M}^{(i)}}     
    \left(
   \left(p_f^{(i)}\right)^T
    +
    \sum_{j=1}^{i-1}
    % \sum_{j\in\{1,\dots,i-1\}}
    \left(\zeta^{(j)}\right)^T\cdot Q_f^{(i,j)} 
    +
    \sum_{j=i+1}^n
    % \sum_{j\in\{i+1,\dots,n\}}
    \left(\zeta^{(j)}\right)^T\cdot \left(Q_f^{(i,j)}\right)^T 
    \right) 
    \cdot s'
    \notag
\end{align}

and the Line-Search at line 5 reads:
\begin{align}
  \gamma =:  
  &
  \argmin_{\gamma' \in[0,1]} f\left((1-\gamma')\cdot\zeta + \gamma'\cdot \tilde{s}\right)
  \notag \\ = &
   \argmin_{\begin{array}{c}
          \gamma' \in[0,1]\\
         y=(1-\gamma')\cdot\zeta^{(i)}+ \gamma'\cdot s      
   \end{array}} 
    \left(
    \begin{array}{c}
    \left(p_f^{(i)}\right)^T
    +
    \sum_{j=1}^{i-1}
    % \sum_{j\in\{1,\dots,i-1\}}
    \left(\zeta^{(j)}\right)^T\cdot Q_f^{(i,j)}
           \\
     +
    \sum_{j=i+1}^n
    % \sum_{j\in\{i+1,\dots,n\}}
    \left(\zeta^{(j)}\right)^T\cdot \left(Q_f^{(i,j)}\right)^T     
    \end{array}
    \right) 
    \cdot y
    \notag \\ = &
    \argmin_{\begin{array}{c}
          \gamma' \in[0,1]\\
         y=(1-\gamma')\cdot\zeta^{(i)}+ \gamma'\cdot s      
   \end{array}} 
    \nabla^{(i)}f(\zeta)^T\cdot y
    \label{eqn:line_search_actual}
\end{align}
Since $\zeta^{(i)}, s\in\mathcal{M}^{(i)}$ and $\gamma\in[0,1]$ then the convex combination of those also satisfies $y\in\mathcal{M}^{(i)}$, hence considering that $s$ is the optimizer of 
\ref{eqn:line_search_minizer} the solution to \ref{eqn:line_search_actual} reads $y:=s$ and hence $\gamma:=1$. Thus, effectively the analytic solution to line-search for a bilinear objective function is $\gamma\equiv 1$ at all times.
\end{proof}

% optimizes the linear objective in $\mathcal{M}^{(i)}$ the following holds:
% \begin{align}
%       \min_{\begin{array}{c}
%           \gamma' \in[0,1]\\
%          y=(1-\gamma')\cdot\zeta^{(i)}+ \gamma'\cdot s      
%   \end{array}} 
%     \left(
%     p_f^{(i)}
%     +
%     \sum_{j\in\{1,\dots,i-1\}} \left(\zeta^{(j)}\right)^T\cdot Q_f^{(i,j)} 
%     +
%     \sum_{j\in\{i+1,\dots,n\}} \left(\zeta^{(j)}\right)^T\cdot \left(Q_f^{(i,j)}\right)^T 
%     \right) 
%     \cdot y
%     \geq
%       \argmin_{s'\in\mathcal{M}^{(i)}}
%     \left(
%     p_f^{(i)}
%     +
%     \sum_{j\in\{1,\dots,i-1\}} \left(\zeta^{(j)}\right)^T\cdot Q_f^{(i,j)} 
%     +
%     \sum_{j\in\{i+1,\dots,n\}} \left(\zeta^{(j)}\right)^T\cdot \left(Q_f^{(i,j)}\right)^T 
%     \right) 
%     \cdot s' 
%     = 
%      \left(
%     p_f^{(i)}
%     +
%     \sum_{j\in\{1,\dots,i-1\}} \left(\zeta^{(j)}\right)^T\cdot Q_f^{(i,j)} 
%     +
%     \sum_{j\in\{i+1,\dots,n\}} \left(\zeta^{(j)}\right)^T\cdot \left(Q_f^{(i,j)}\right)^T 
%     \right) 
%     \cdot s 
% \end{align}
% Hence $\gamma\equiv 1$ is the analytic closed form solution for the line-search 

% The following Line-Search specifies $\gamma$ by solving the following problem:
% \begin{align}
%   \gamma =:  \argmin_{\gamma' \in[0,1]} f\left((1-\gamma')\cdot\zeta + \gamma'\cdot \tilde{s}\right)
% \end{align}
% where $\tilde{s}\in\R^m$ is the zero padding of $s$ such that $\tilde{s}^{(i)}:=s$ and $\tilde{s}^{\setminus(i)}:=0$.

\subsection{Solving BLCP by BCFW with Line-Search}
In addition to a bilinear objective function as in equation~\ref{eqn:block_quad_obj_prod}, consider also a domain that is specified by the following bilinear constraints:
\begin{align}
    \sum_{i=1}^n \left(\zeta^{(i)}\right)^T\cdot p_{\mathcal{M}}^{(i)}  + \sum_{i=1}^n\sum_{j=i}^n \left(\zeta^{(i)}\right)^T\cdot Q_{\mathcal{M}}^{(i,j)}\cdot\zeta^{(j)} \leq T
    \quad ; \quad
    A\cdot \zeta \leq b
\end{align}
with $p_{\mathcal{M}}^{(i)}\in\R^{m_i}$, $Q_{\mathcal{M}}^{(i,j)}\in\R^{m_i\times m_j}$, $A\in\R^{C\times m}$ and $b\in\R^C$ for $C\leq 0$, such that the individual domain of the $i$-th coordinate block is specified by the following linear constraints:
\begin{align}
    \left(
    \left(p_{\mathcal{M}}^{(i)}\right)^T    +
    \sum_{j=1}^{i-1}
    % \sum_{j\in\{1,\dots,i-1\}}
    \left(\zeta^{(j)}\right)^T\cdot Q_{\mathcal{M}}^{(i,j)} 
    +
    \sum_{j=i+1}^n
    % \sum_{j\in\{i+1,\dots,n\}}
    \left(\zeta^{(j)}\right)^T\cdot \left(Q_{\mathcal{M}}^{(i,j)}\right)^T 
    \right) 
    \cdot \zeta^{(i)} &\leq T
    \\
    A^{(i)}\cdot \zeta^{(i)} &\leq b^{(i)}
\end{align}
where $A^{(i)}\in\R^{C\times m_i}$ are the rows $r\in\{1+\sum_{j < i} m_j,\dots,\sum_{j \leq i} m_j\}$ of $A$ and $b^{(i)}\in\R^m_i$ are the corresponding elements of $b$.

Thus in each step of algorithm~\ref{alg:plain_BCFW} at line 3, a linear program is solved:
\begin{align*}
    \min_{\zeta^{(i)}} &     
    \left(
    \left(p_f^{(i)}\right)^T
    +
    \sum_{j=1}^{i-1}
    % \sum_{j\in\{1,\dots,i-1\}}
    \left(\zeta^{(j)}\right)^T\cdot Q_f^{(i,j)} 
    +
    \sum_{j=i+1}^n
    % \sum_{j\in\{i+1,\dots,n\}}
    \left(\zeta^{(j)}\right)^T\cdot \left(Q_f^{(i,j)}\right)^T 
    \right) 
    \cdot \zeta^{(i)}
    \\ s.t. &
        \left(
    \left(p_{\mathcal{M}}^{(i)}\right)^T    +
    \sum_{j=1}^{i-1}
    % \sum_{j\in\{1,\dots,i-1\}}
    \left(\zeta^{(j)}\right)^T\cdot Q_{\mathcal{M}}^{(i,j)} 
    +
    \sum_{j=i+1}^n
    % \sum_{j\in\{i+1,\dots,n\}}
    \left(\zeta^{(j)}\right)^T\cdot \left(Q_{\mathcal{M}}^{(i,j)}\right)^T 
    \right) 
    \cdot \zeta^{(i)} \leq T
    \\ &
    A^{(i)}\cdot \zeta^{(i)} \leq b^{(i)}
\end{align*}

And thus equipped with theorem~\ref{theo:analytic_line_search}, algorithm~\ref{alg:BCFW_QCQP_prod} provides a more specific version of algorithm~\ref{alg:plain_BCFW} for solving BLCP. 

\input{bcfw_qcqp_prod_algo}

In section~\ref{sec:method}, we deal with $n=2$ blocks where $\zeta=(\balpha,\bbeta)$ such that:
\begin{align}\label{eqn:2_blocks}
\begin{array}{ccccccc}
     \zeta^{(1)}=\balpha & m_1=D\cdot S\cdot|\mathcal{C}| & p_f^{(1)}=p_\alpha & p_\mathcal{M}^{(1)}=0 & A^{(1)} = A_\mathcal{S}^\alpha & b^{(1)} = b_\mathcal{S}^\alpha & Q_f^{(1,2)} = Q_{\alpha\beta}
     \\
     \zeta^{(2)}=\bbeta & m_2=D\cdot S &  p_f^{(2)}=p_\beta & p_\mathcal{M}^{(2)}=0 & A^{(2)} = A_\mathcal{S}^\beta& b^{(2)} = b_\mathcal{S}^\beta & Q_\mathcal{M}^{(1,2)} = \Theta
\end{array}
\end{align}

Thus for this particular case of interest algorithm~\ref{alg:BCFW_QCQP_prod} effectively boils down to algorithm~\ref{alg:BCFW_QCQP}.

\subsubsection{Proof of Theorem~\ref{alg:BCFW_QCQP}}
Let us first compute the curvature constants $C_f^{(i)}$ (equation~\ref{eqn:curvature_const}) and $C_f^\otimes$ for the bilinear objective function as in equation~\ref{eqn:block_quad_obj_prod}.
\begin{lem}
Let $f$ have a bilinear form, such that:\\
$
    f(x)=\sum_{i=1}^n \left(x^{(i)}\right)^T\cdot p_f^{(i)}  + \sum_{i=1}^n\sum_{j=i}^n \left(x^{(i)}\right)^T\cdot Q_f^{(i,j)}\cdot x^{(j)}
$
 then $C_f^\otimes = 0$.
\end{lem}

\begin{proof}
Separating the $i$-th coordinate block:
\begin{align}
f(x) 
&= \sum_{l=1}^n \left(x^{(l)}\right)^T\cdot p_f^{(l)} + \sum_{l=1}^n\sum_{j=l}^n \left(x^{(l)}\right)^T\cdot Q_f^{(l,j)}\cdot x^{(j)}
\\
&=
\left(x^{(i)}\right)^T\cdot p_f^{(i)} 
+ \sum_{j=1}^{i-1}
%  \sum_{j\in\{1,\dots,i-1\}}
\left(x^{(j)}\right)^T\cdot Q_f^{(i,j)} \cdot x^{(i)}
+ \sum_{j=i+1}^n
%  \sum_{j\in\{i+1,\dots,n\}}
x^{(i)}\cdot Q_f^{(i,j)}\cdot x^{(j)}
\\ &  +  \sum_{l=1}^n \mathbbm{1}_{l\neq i} \left(x^{(l)}\right)^T\cdot p_f^{(l)}
+ \sum_{l=1}^n\sum_{j=l}^n \mathbbm{1}_{l\neq i}\cdot \mathbbm{1}_{j\neq i}\left(x^{(l)}\right)^T\cdot Q_f^{(l,j)}\cdot x^{(j)}
\label{eqn:cancel_x}
\end{align}
where $\mathds{1}_A$ is the indicator function that yields $1$ if $A$ holds and $0$ otherwise.

Thus for $y$ with $y^{(i)}=(1-\gamma)x^{(i)}+\gamma s^{(i)}$ and $y^{\setminus(i)} = x^{\setminus(i)}$, we have:
\begin{align}
f(y) 
&=
\left(y^{(i)}\right)^T\cdot p_f^{(i)} 
+ \sum_{j=1}^{i-1} \left(y^{(j)}\right)^T\cdot Q_f^{(i,j)} \cdot y^{(i)}
+ \sum_{j=i+1}^n
%  \sum_{j\in\{i+1,\dots,n\}}
y^{(i)}\cdot Q_f^{(i,j)}\cdot y^{(j)}
\\ &  +  \sum_{l=1}^n \mathbbm{1}_{l\neq i} \left(y^{(l)}\right)^T\cdot p_f^{(l)}
+ \sum_{l=1}^n\sum_{j=l}^n \mathbbm{1}_{l\neq i}\cdot \mathbbm{1}_{j\neq i}\left(x^{(l)}\right)^T\cdot Q_f^{(l,j)}\cdot y^{(j)}
\\&=
\left(y^{(i)}\right)^T\cdot p_f^{(i)} 
+ 
\sum_{j=1}^{i-1}
% \sum_{j\in\{1,\dots,i-1\}}
\left(x^{(j)}\right)^T\cdot Q_f^{(i,j)} \cdot y^{(i)}
+ 
\sum_{j=i+1}^n
% \sum_{j\in\{i+1,\dots,n\}}
y^{(i)}\cdot Q_f^{(i,j)}\cdot x^{(j)}
\\ &  +  
\sum_{l=1}^n \mathbbm{1}_{l\neq i} \left(x^{(l)}\right)^T\cdot p_f^{(l)}
+  
\sum_{l=1}^n\sum_{j=l}^n \mathbbm{1}_{l\neq i}\cdot \mathbbm{1}_{j\neq i}\left(x^{(l)}\right)^T\cdot Q_f^{(l,j)}\cdot x^{(j)}
\label{eqn:cancel_y}
\end{align}
Hence, 
\begin{align}
    f(y)-f(x) 
    & =
    \nabla^{(i)}f(x)
    \cdot \left(y^{(i)}-x^{(i)}\right)
\end{align}
since \ref{eqn:cancel_x} and \ref{eqn:cancel_y} cancel out and,
\begin{align}
    \nabla^{(i)}f(x)=
    \left(
    \left(p_f^{(i)}\right)^T
    +
    \sum_{j=1}^{i-1}
    % \sum_{j\in\{1,\dots,i-1\}}
    \left(x^{(j)}\right)^T\cdot Q_f^{(i,j)} 
    +
    \sum_{j=i+1}^n
    % \sum_{j\in\{i+1,\dots,n\}} 
    \left(x^{(j)}\right)^T\cdot \left(Q_f^{(i,j)}\right)^T 
    \right) 
\end{align}
Hence we have,
\begin{align}
C_f^{(i)}=0 \quad \forall i\in\{1,\dots,n\} \qquad ; \qquad C_f^\otimes = \sum_{i=1}^n C_f^{(i)} = 0
\end{align}
\end{proof}

Thus for a bilinear objective function, theorem~\ref{theo:plain_convergence} boils down to:

\begin{theo}\label{theo:block_quad_convergence}
    For each $k>0$ the iterate $\zeta_k$ Algorithm~\ref{alg:BCFW_QCQP_prod} satisfies:
    $$E[f(\zeta_k)] - f(\zeta^*) \leq \frac{2n}{k+2n}\left(f(\zeta_0) - f(\zeta^*)\right)$$
    where $\zeta^*$ is the solution of problem \ref{eqn:prod_domain} and the expectation is over the random choice of the block $i$ in the steps of the algorithm.
    Furthermore, there exists an iterate $0\leq \hat{k}\leq K$ of Algorithm~\ref{alg:BCFW_QCQP_prod}  with a duality gap bounded by $E[g(\zeta_{\hat{k}})]\leq \frac{6n}{K+1}\left(f(\zeta_0) - f(\zeta^*)\right)$.
\end{theo}

And by setting $n=2$ with equations~\ref{eqn:2_blocks} for $f(\zeta):=ACC(\zeta)$, theorem~\ref{theo:convergence} follows.

\section{Sparsity Guarantees for Solving BLCP with BCFW with Line-Search}\label{apdx:proof_sparsity_bcfw}
In order to proof \ref{thm:one_hot_sol}, we start with providing auxiliary lemmas proven at \cite{nayman2021hardcore}. To this end we define the \textit{relaxed} Multiple Choice Knapsack Problem (MCKP):

\begin{de}
Given $n\in\mathbb{N}$, and a collection of $k$ distinct covering subsets of $\{1,2,\cdots, n\}$ denoted as $N_i,i\in\{1,2,\cdots,k\}$, such that $\cup_{i=1}^k N_i=\{1,2,\cdots, n\}$ and $\cap_{i=1}^k N_i=\varnothing$ with 
% \begin{align*}
%     \displaystyle{\cup_{i=1}^k N_i}&=\{1,2,\cdots, n\} \\
%     \displaystyle{\cap_{i=1}^k N_i}&=\varnothing
% \end{align*}
associated values and costs $p_{ij}, t_{ij}~ \forall i \in \{1,\dots,k\}, j \in N_i$ respectively, the relaxed Multiple Choice Knapsack Problem (MCKP) is formulated as following:
\vspace{-3mm}
% \begin{equation}
% \label{eqn:MCKP}
% \begin{array}{l}
% \displaystyle{\max\sum_{i=1}^k \sum_{j\in N_i}} p_{ij} \vu_{ij} \\
% \text{subject to} \\
% \displaystyle{\sum_{i=1}^k\sum_{j\in N_i} t_{ij} \vu_{ij} \leq T,} \\
% \displaystyle{\sum_{j \in N_i}\vu_{ij} = 1,\forall~ 1 \leq i \leq k} \\
% \displaystyle{\vu_{ij} \in [0,1], \forall~ 1 \leq i \leq k, \forall j \in N_i}
% \end{array}
% \end{equation}
\begin{align}
\label{eqn:MCKP}
\notag
\max_{vu} &\sum_{i=1}^k \sum_{j\in N_i}  p_{ij} \vu_{ij} \\
\text{s.t.}
&\sum_{i=1}^k\sum_{j\in N_i} t_{ij} \vu_{ij} \leq T \\
\notag
&\sum_{j \in N_i}\vu_{ij} = 1 \qquad \forall i \in \{1,\dots,k\} \\
\notag
&\vu_{ij} \geq 0  \qquad \forall i \in \{1,\dots,k\}, j \in N_i
\end{align}
where the binary constraints $\displaystyle{\vu_{ij}\in \{0,1\}}$ of the original MCKP formulation~\cite{MCKP} are replaced with $\displaystyle{\vu_{ij} \geq 0}$.
\end{de}

\begin{de}
 An one-hot vector $\vu_i$ satisfies: $$||\vu_i^*||^0=\sum_{j\in N_i} |\vu_{ij}^*|^0=\sum_{j\in N_i}\mathds{1}_{\vu_{ij}^*>0}=1$$ 
where $\mathds{1}_A$ is the indicator function that yields $1$ if $A$ holds and $0$ otherwise.
\end{de}

\begin{lem}\label{lem:single_non_one_hot}
The solution $\vu^*$ of the \textit{relaxed} MCKP \eqref{eqn:MCKP} is composed of vectors $\vu_i^*$ that are all one-hot but a single one. 
\end{lem}

\begin{lem}\label{lem:two_nonzeros}
The single non one-hot vector of the solution $\vu^*$ of the \textit{relaxed} MCKP \eqref{eqn:MCKP} has at most two nonzero elements.
\end{lem}

See the proofs for Lemmas~\ref{lem:single_non_one_hot} and~\ref{lem:single_non_one_hot} in \cite{nayman2021hardcore} (Appendix F).

In order to prove Theorem~\ref{thm:one_hot_sol}, we use Lemmas~\ref{lem:single_non_one_hot} and~\ref{lem:single_non_one_hot} for each coordinate block $\zeta^{(i)}$ for $i\in\{1,\dots,n\}$ separately, based on the observation that at every iteration $k=0,\dots,K$ of algorithm~\ref{alg:BCFW_QCQP_prod}, each sub-problem (lines 3,5) forms a \textit{relaxed} MCKP~\eqref{eqn:MCKP}. Thus replacing
\begin{itemize}
    \item $\vu$ in~\eqref{eqn:MCKP} with $\zeta^{(i)}$.
    
    \item $p$ with $\left(p_f^{(i)}\right)^T    +
\sum_{j=1}^{i-1}
% \sum_{j\in\{1,\dots,i-1\}}
\left(\zeta^{(j)}\right)^T\cdot Q_f^{(i,j)} 
+
\sum_{j=i+1}^n
% \sum_{j\in\{i+1,\dots,n\}}
\left(\zeta^{(j)}\right)^T\cdot \left(Q_f^{(i,j)}\right)^T$

\item The elements of \textbf{$t$} with the elements of $$\left(p_{\mathcal{M}}^{(i)}\right)^T
+
\sum_{j=1}^{i-1}
% \sum_{j\in\{1,\dots,i-1\}}
\left(\zeta^{(j)}\right)^T\cdot Q_{\mathcal{M}}^{(i,j)} 
+
\sum_{j=i+1}^n
% \sum_{j\in\{i+1,\dots,n\}}
\left(\zeta^{(j)}\right)^T\cdot \left(Q_{\mathcal{M}}^{(i,j)}\right)^T$$

\item The simplex constraints with the linear inequality constraints specified by $A^{(i)}, b^{(i)}$.
\end{itemize}
Hence for every iteration theorem~\ref{thm:one_hot_sol} holds and in particular for the last iteration $k=K$ which is the output of solution of algorithm~\ref{alg:BCFW_QCQP_prod}.

By setting $n=2$ with equations~\ref{eqn:2_blocks} for $f(\zeta):=ACC(\zeta)$, algorithm~\ref{alg:BCFW_QCQP_prod} boils down to algorithm~\ref{alg:BCFW_QCQP} and thus theorem~\ref{thm:one_hot_sol} holds for the later as special case of the former.

% In order to prove Theorem~\ref{thm:one_hot_sol}, we use Lemmas~\ref{lem:single_non_one_hot} and~\ref{lem:single_non_one_hot} for $\balpha$ and $\bbeta$ separately, based on the observation that at every iteration $k=0,\dots,K$ of algorithm~\ref{alg:BCFW_QCQP}, each sub-problem (lines 3,5) forms a \textit{relaxed} MCKP~\eqref{eqn:MCKP}. Thus replacing $\vu$ in~\eqref{eqn:MCKP} with $\balpha$ and $\bbeta$, $p$ is replaced with $q^T_\alpha + \beta_k^T Q_{\alpha\beta}^T$ and $q^T_\beta + \alpha_k^T Q_{\alpha\beta}$ and the elements of \textbf{$t$} are replaced with the elements of $\bbeta_k^T \Theta^T$ and $\balpha_k^T \Theta$ respectively. Hence for every iteration theorem~\ref{thm:one_hot_sol} holds and in particular for the last iteration $k=K$ which is the output of solution of algorithm~\ref{alg:BCFW_QCQP}

\section{On the Transitivity of Ranking Correlations}\label{apdx:kt_transitivity}
While the predictors in section~\ref{sec:learnt_predictors} yields high ranking correlation between the predicted accuracy and the accuracy measured for a subnetwork of a given one-shot model, the ultimate ranking correlation is with respect to the same architecture trained as a standalone from scratch. Hence we are interested also in the transitivity of ranking correlation. \cite{langford2001property} provides such transitivity property of the Pearson correlation between random variables $P,O,S$ standing for the predicted, the one-shot and the standalone accuracy respectively:
\begin{align*}
    % Cor(X,Y)\cdot Cor(X,Z) - \sqrt{\left(1-Cor(X,Y)^2\right)\cdot\left(1-Cor(X,Z)^2\right)}
    % \leq
    |Cor(P,S) -
    Cor(P,O)\cdot Cor(O,S)| \leq \sqrt{\left(1-Cor(P,O)^2\right)\cdot\left(1-Cor(O,S)^2\right)}
\end{align*}
This is also true for the Spearman correlation as a Pearson correlation over the corresponding ranking.
Hence, while the accuracy estimator can be efficiently acquired for any given one-shot model, the quality of this one-shot model contributes its part to the overall ranking correlation. In this paper we use the official one-shot model provided by \cite{nayman2021hardcore} with a reported Spearman correlation of $\rho_{P,O}=0.99$ to the standalone networks. Thus together with the Spearman correlation of 
$\rho_{O,S}=0.97$ 
% $\rho_{O,S}=0.98$ 
of the proposed accuracy estimator, the overall Spearman ranking correlation satisfies 
$\rho_{P,S}\geq 0.93$.
% $0.94\leq\rho_{P,S}\leq 1$.

\section{Averaging Individual Accuracy Contribution for Interpretability}
\label{apdx:avg_contrib}
In this section we provide the technical details about the way we generate the results presented in section~\ref{sec:interpretability}, Figure~\ref{fig:insights} (Middle and Right).

We measure the average contribution of adding a S\&E layer at each stage $s$ by:
$$\frac{1}{D}\sum_{b=1}^D \left(\frac{1}{|\mathcal{C}_{se}|} \sum_{c\in\mathcal{C}_{se}}\Delta^s_{b,c}-\frac{1}{|\bar{\mathcal{C}}_{se}|}\sum_{c\in\bar{\mathcal{C}}_{se}}\Delta^s_{b,c}\right)$$
and each block $b$ by:
$$\frac{1}{S}\sum_{s=1}^S \left(\frac{1}{|\mathcal{C}_{se}|} \sum_{c\in\mathcal{C}_{se}}\Delta^s_{b,c}-\frac{1}{|\bar{\mathcal{C}}_{se}|}\sum_{c\in\bar{\mathcal{C}}_{se}}\Delta^s_{b,c}\right)$$
where $\mathcal{C}_{se}=\{2,4,6,8,10,12\}$ ($|\mathcal{C}_{se}|=6$) are all the configurations that include S\&E layers according to Table~\ref{tab:configurations} and its complementary set $\bar{\mathcal{C}}_{se}=\{1,3,5,7,9,11\}$ ($|\bar{\mathcal{C}}_{se}|=6$).

Similarly the average contribution of increasing the kernel size from $3\times 3$ to $5\times 5$ at each stage $s$ by:
$$\frac{1}{D}\sum_{b=1}^D \left(\frac{1}{|\mathcal{C}_{5\times 5}|} \sum_{c\in\mathcal{C}_{5\times 5}}\Delta^s_{b,c}-\frac{1}{|\mathcal{C}_{3\times 3}|}\sum_{c\in\mathcal{C}_{3\times 3}}\Delta^s_{b,c}\right)$$
and each block $b$ by:
$$\frac{1}{S}\sum_{s=1}^S \left(\frac{1}{|\mathcal{C}_{5\times 5}|} \sum_{c\in\mathcal{C}_{5\times 5}}\Delta^s_{b,c}-\frac{1}{|\mathcal{C}_{3\times 3}|}\sum_{c\in\mathcal{C}_{3\times 3}}\Delta^s_{b,c}\right)$$
where $\mathcal{C}_{3\times 3}=\{1,2,5,6,9,10\}$ ($|\mathcal{C}_{3\times 3}|=6$) and $\mathcal{C}_{5\times 5}=\{3,4,7,8,11,12\}$ ($|\mathcal{C}_{5\times 5}|=6$) are all the configurations with a kernel size of $3\times 3$ and $5\times 5$ respectively according to Table~\ref{tab:configurations}.

The average contribution of increasing the expansion ratio from $2$ to $3$ at each stage $s$ by:
$$\frac{1}{D}\sum_{b=1}^D \left(\frac{1}{|\mathcal{C}_{3}|} \sum_{c\in\mathcal{C}_{3}}\Delta^s_{b,c}-\frac{1}{|\mathcal{C}_{2}|}\sum_{c\in\mathcal{C}_{2}}\Delta^s_{b,c}\right)$$
and each block $b$ by:
$$\frac{1}{S}\sum_{s=1}^S \left(\frac{1}{|\mathcal{C}_{3}|} \sum_{c\in\mathcal{C}_{3}}\Delta^s_{b,c}-\frac{1}{|\mathcal{C}_{2}|}\sum_{c\in\mathcal{C}_{2}}\Delta^s_{b,c}\right)$$
where $\mathcal{C}_{3}=\{5,6,7,8\}$ ($|\mathcal{C}_{3}|=4$) and $\mathcal{C}_{2}=\{1,2,3,4\}$ ($|\mathcal{C}_{2}|=4$) are all the configurations with an expansion ratio of $3$ and $2$ respectively according to Table~\ref{tab:configurations}.

Finally, the average contribution of increasing the expansion ratio from $3$ to $6$ at each stage $s$ by:
$$\frac{1}{D}\sum_{b=1}^D \left(\frac{1}{|\mathcal{C}_{6}|} \sum_{c\in\mathcal{C}_{6}}\Delta^s_{b,c}-\frac{1}{|\mathcal{C}_{3}|}\sum_{c\in\mathcal{C}_{3}}\Delta^s_{b,c}\right)$$
and each block $b$ by:
$$\frac{1}{S}\sum_{s=1}^S \left(\frac{1}{|\mathcal{C}_{6}|} \sum_{c\in\mathcal{C}_{6}}\Delta^s_{b,c}-\frac{1}{|\mathcal{C}_{3}|}\sum_{c\in\mathcal{C}_{3}}\Delta^s_{b,c}\right)$$
where $\mathcal{C}_{6}=\{9,10,11,12\}$ ($|\mathcal{C}_{6}|=4$) are all the configurations with an expansion ratio of $6$ according to Table~\ref{tab:configurations}.

\section{Extended Figures}
Due to space limit, here we present extended figures to better describe the estimation of individual accuracy contribution terms described in section~\ref{sec:acc_contrib}.

\begin{figure}[htb]
    \centering
     \includegraphics[width=1\textwidth]{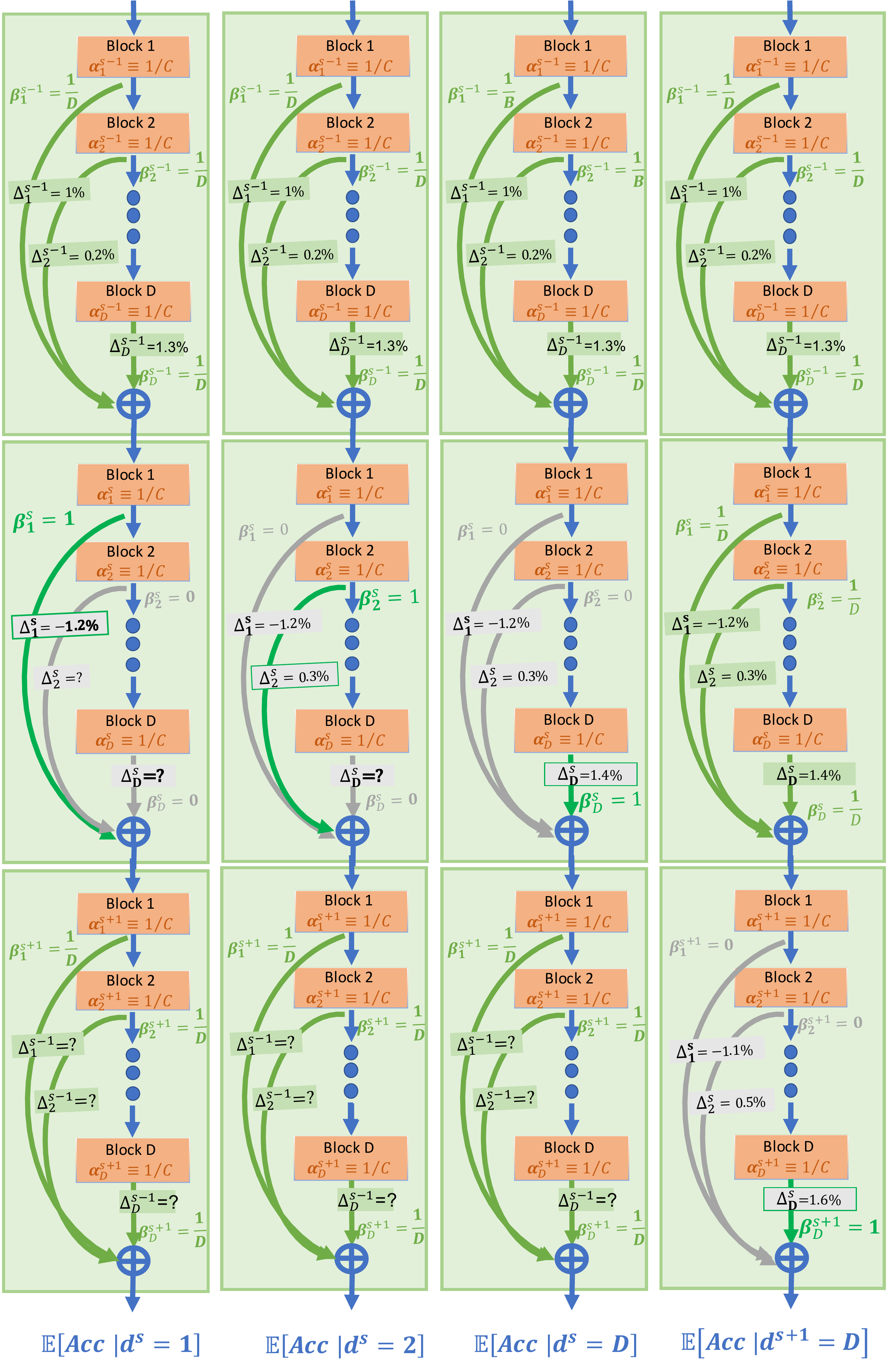}
     \caption{Extension of Figure~\ref{fig:estimate_beta}: Estimating the expected accuracy gap caused by macroscopic design choices of the depth of the stages.}
     \label{fig:estimate_beta_portrait}
\end{figure}

\begin{figure}[htb]
\begin{sideways}
 \centering
 \includegraphics[width=.95\textheight]{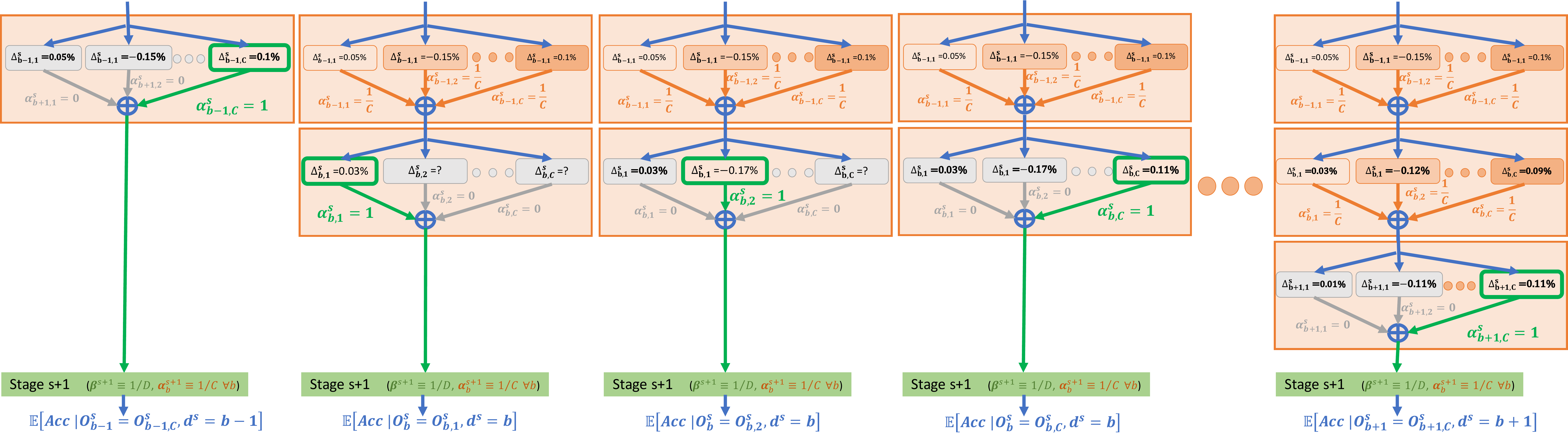}
 \end{sideways}
\caption{Extension of Figure~\ref{fig:estimate_alpha}: Estimating the expected accuracy gap caused by microscopic design choices on the operation applied at every block one at a time.}
 \label{fig:estimate_alpha_full}
\end{figure}

%% file: closed_form_reg_algo.tex
\begin{algorithm}[htb]
  \caption{Closed Form Solution of a Linear Regression for the Quadratic Predictors}
  \label{alg:plain_regression}
\begin{algorithmic}[1]
% \INPUT Validation set $\mathcal{D}_v$, loss function $\mathcal{L}$
\INPUT $\{x_i=(\balpha_i, \bbeta_i)\in\mathbb{R}^n, y_i = Acc(\balpha_i, \bbeta_i)\}_{i=1}^N, k=$ number of principal components
% \FOR{$(x_t, y_t)$ {\bfseries in} $\mathcal{D}_{val}$}
\STATE Compute $\tilde{x}_i= (x_i, x_i\otimes x_i),~ \forall i\in\{1,\dots,N\}$
\STATE Perform a centering on $\tilde{x}_i$ computing $\hat{x}_i = \tilde{x}_i - \operatorname{mean}_{i=1,\dots,N}\left(\tilde{x}_i\right),~ \forall i\in\{1,\dots,N\}$
\STATE Perform a centering on $y_i$ computing $\hat{y}_i = y_i -\operatorname{mean}_{i=1,\dots,N}\left(y_i\right), ~ \forall i\in\{1,\dots,N\}$
\STATE Define $\hat{X} = \operatorname{stack}(\{\hat{x}_i\}_{i=1}^N)$
\STATE Compute a $k$-low rank SVD decomposition of $\hat{X}$, defined as $U\operatorname{diag}(s)V^T$
\STATE Compute $W=V\operatorname{diag}(s^{-1})U^T\hat{y}$
\STATE Compute $b=\operatorname{mean}\left(y-\tilde{X}W\right)$
\STATE Define $a=W_{1:n}$
\STATE Reshape the end of the vector $W$ as an $n\times n$ matrix,  $Q=\operatorname{reshape}(W_{n+1:n+1+n^2},n,n)$
\OUTPUT $b, a, Q$

\end{algorithmic}
\end{algorithm}

%% file: kt_mse_vs_num_components.tex
\begin{figure}[htb]
\begin{center}
\begin{adjustbox}{width=0.6\textwidth}
\begin{tikzpicture}
\begin{axis}[
            axis x line=left,
            axis y line=left,
            xmajorgrids=true,
            ymajorgrids=true,
            grid=both,
            xlabel style={below=1ex},
            enlarge x limits,
            ymin = 0.68,
            ymax = 0.87,
            xmin = 0.0,
            xmax = 3000.0,
            % xmax = 90.0,
            % x tick label style={ %font=\fontsize{4}{6}\selectfont, text width=0.5cm},
            xtick = {0,500,..., 3000},
            ytick = {0.6,0.65,0.7,...,0.85},
            ylabel = Kendall-Tau Coefficient (solid),
            xlabel = Number of principal components,
            legend style={nodes={scale=0.8, transform shape}, at={(0.97,0.5)},anchor=east}
    ]

\addplot[%style=dashed, 
color=green, mark=.,very thick]coordinates {(50,0.6804638201571267)(100,0.7197331766345908)(150,0.727869043988393)(200,0.8017833459634631)(250,0.831772442230198)(300,0.8329271026212166)(350,0.8327827700723394)(400,0.8302489319920483)(450,0.8309064469369339)(500,0.8291216501653685)(550,0.8315318879820691)(600,0.8321847198684003)(650,0.8297517865459153)(700,0.8265349851784869)(750,0.8340336521626096)(800,0.834686494076112)(850,0.8342100586112373)(900,0.8362627881952706)(950,0.8335204697666012)(1000,0.8368401183907799)(1050,0.8394060303708213)(1100,0.8323130159816161)(1150,0.8380268526815491)(1200,0.8375457441852913)(1250,0.839197550022443)(1300,0.834253494694008)(1350,0.8379145940324222)(1400,0.8340657260623601)(1450,0.8369797870998475)(1500,0.8414587599548545)(1550,0.8348194627064973)(1600,0.8382513699798027)(1650,0.8346270193079941)(1700,0.8360222339471417)(1750,0.8338412087641064)(1800,0.8398711019172038)(1850,0.8326705114232125)(1900,0.8355571624007592)(1950,0.833211088774128)(2000,0.8364986766752875)(2050,0.833692199198688)(2100,0.8323176985259568)(2150,0.8356694210498861)(2200,0.8354449037516323)(2250,0.8347506421327198)(2300,0.831098890335437)(2350,0.8230002306484312)(2400,0.8218776441571629)(2450,0.8202418752698865)(2500,0.8202418752698865)(2550,0.8202418752698865)(2600,0.8202418752698865)(2650,0.8202418752698865)(2700,0.8202418752698865)(2750,0.8202418752698865)(2800,0.8202418752698865)(2850,0.8202418752698865)(2900,0.8202418752698865)(2950,0.8202418752698865)(3000,0.8202418752698865)
% (40, 0.86)
};\addlegendentry{Validation set}

\addplot[%style=dashed, 
color=blue, mark=.,very thick]coordinates {
(50,0.6813684306182659)(100,0.7076923590317828)(150,0.7205228966398649)(200,0.8043544217366713)(250,0.829758886200674)(300,0.8296305808245931)(350,0.830759912290086)(400,0.8287965958800677)(450,0.8292409692281053)(500,0.8263861631681755)(550,0.8322608410342499)(600,0.8296740016079822)(650,0.829758886200674)(700,0.8312023216815831)(750,0.8326891900308292)(800,0.8338599812801263)(850,0.83398828717046)(900,0.834169383503452)(950,0.8331750168388257)(1000,0.8372166361853716)(1050,0.8356448953283815)(1100,0.8332618871796533)(1150,0.8370883308092907)(1200,0.8361581168327048)(1250,0.843049890675293)(1300,0.834634490491745)(1350,0.8339127727512905)(1400,0.8328542533986237)(1450,0.8382865344966464)(1500,0.8356448953283815)(1550,0.8382979135599362)(1600,0.836827054994098)(1650,0.8359977351126037)(1700,0.829951344264795)(1750,0.8379497315345199)(1800,0.8329985969467145)(1850,0.8298183457346079)(1900,0.8313145888856539)(1950,0.8323684253049944)(2000,0.827272969539108)(2050,0.8275890308900561)(2100,0.8283475270637848)(2150,0.8284597942678555)(2200,0.829369275118439)(2250,0.8247229001895017)(2300,0.8189331200938547)(2350,0.8158650801607944)(2400,0.8049638722730552)(2450,0.8020930394832468)(2500,0.8020930394832468)(2550,0.8020930394832468)(2600,0.8020930394832468)(2650,0.8020930394832468)(2700,0.8020930394832468)(2750,0.8020930394832468)(2800,0.8020930394832468)(2850,0.8020930394832468)(2900,0.8020930394832468)(2950,0.8020930394832468)(3000,0.8020930394832468)
% (40, 0.86)
};\addlegendentry{Test set}

\end{axis}

\begin{axis}[
            axis x line=left,
            axis y line=right,
            xmajorgrids=true,
            ymajorgrids=true,
            grid=both,
            xlabel style={below=1ex},
            enlarge x limits,
            ymin = 0.04,
            ymax = 0.3,
            xmin = 0,
            xmax = 3000,
            % xmax = 90.0,
            % x tick label style={ %font=\fontsize{4}{6}\selectfont, text width=0.5cm},
            xtick = {0,500,..., 3000},
            ytick = {0.05,0.1,...,0.3},
            ylabel = Mean Square Error (dashed),
            xlabel = Number of principal components,
            % legend pos=south east,
    ]

\addplot[style=dashed, 
color=green, mark=.,very thick]coordinates {
(50,0.21760588884353638)(100,0.18251392245292664)(150,0.1677221953868866)(200,0.10122755169868469)(250,0.07752274721860886)(300,0.07838350534439087)(350,0.07708679139614105)(400,0.07670054584741592)(450,0.07677818089723587)(500,0.07902830839157104)(550,0.07547295838594437)(600,0.07465256005525589)(650,0.07712951302528381)(700,0.07793530821800232)(750,0.07273732125759125)(800,0.07208386063575745)(850,0.07225771248340607)(900,0.07257652282714844)(950,0.07211950421333313)(1000,0.07076971232891083)(1050,0.06786496192216873)(1100,0.07363921403884888)(1150,0.0695011243224144)(1200,0.06912882626056671)(1250,0.06535405665636063)(1300,0.06954794377088547)(1350,0.06887944787740707)(1400,0.06860537081956863)(1450,0.06782416999340057)(1500,0.06677697598934174)(1550,0.0684238150715828)(1600,0.06777096539735794)(1650,0.06856570392847061)(1700,0.06716743111610413)(1750,0.06867130845785141)(1800,0.0644439160823822)(1850,0.06660627573728561)(1900,0.06576474010944366)(1950,0.0676741674542427)(2000,0.0656094178557396)(2050,0.0655587762594223)(2100,0.06543447822332382)(2150,0.06486628949642181)(2200,0.06466192752122879)(2250,0.06580190360546112)(2300,0.0705181360244751)(2350,0.07378999143838882)(2400,0.0762462317943573)(2450,0.07846670597791672)(2500,0.07846667617559433)(2550,0.07846676558256149)(2600,0.07846669107675552)(2650,0.07846681028604507)(2700,0.0784667506814003)(2750,0.07846682518720627)(2800,0.07846678048372269)(2850,0.07846669107675552)(2900,0.07846669852733612)(2950,0.07846686989068985)(3000,0.07846677303314209)
% (40, 0.86)
};%\addlegendentry{Test set}

\addplot[style=dashed, 
color=blue, mark=.,very thick]coordinates {
(50,0.22740685939788818)(100,0.18871505558490753)(150,0.1670151948928833)(200,0.10347355157136917)(250,0.08390036970376968)(300,0.0819314494729042)(350,0.08379531651735306)(400,0.08100903779268265)(450,0.08300598710775375)(500,0.08378270268440247)(550,0.08112466335296631)(600,0.08017654716968536)(650,0.08136790245771408)(700,0.07908304035663605)(750,0.07955654710531235)(800,0.07718194276094437)(850,0.07693900913000107)(900,0.0776262879371643)(950,0.07805842906236649)(1000,0.07297008484601974)(1050,0.07582627981901169)(1100,0.07342806458473206)(1150,0.07288077473640442)(1200,0.07254590094089508)(1250,0.069464311003685)(1300,0.07154858112335205)(1350,0.07120031863451004)(1400,0.070747010409832)(1450,0.06944076716899872)(1500,0.07085224986076355)(1550,0.06845612823963165)(1600,0.06582783162593842)(1650,0.06569112837314606)(1700,0.06980772316455841)(1750,0.06504995375871658)(1800,0.06778036803007126)(1850,0.06928718835115433)(1900,0.06864005327224731)(1950,0.06609855592250824)(2000,0.07063547521829605)(2050,0.06900838762521744)(2100,0.06892953813076019)(2150,0.06699403375387192)(2200,0.0671030730009079)(2250,0.07189314067363739)(2300,0.07401886582374573)(2350,0.07727222889661789)(2400,0.0840664952993393)(2450,0.08692549914121628)(2500,0.08692538738250732)(2550,0.0869252011179924)(2600,0.0869251936674118)(2650,0.08692539483308792)(2700,0.0869252011179924)(2750,0.08692542463541031)(2800,0.08692549914121628)(2850,0.08692537993192673)(2900,0.08692531287670135)(2950,0.0869251936674118)(3000,0.08692538738250732)
% (40, 0.86)
};%\addlegendentry{Validation set}
\end{axis}
\end{tikzpicture}
\end{adjustbox}
\captionof{figure}{Kendall-Tau correlation coefficients and MSE of different predictors vs number of principal components} 
% \captionof{figure}{Kendall-Tau correlation coefficients between predictors' output and the validation accuracy of sub-networks extracted from the one-shot model vs the number of validation epochs required for acquiring the predictor} %over a test set of 500 networks uniformly sampled from the search space}
\label{fig:kt_mse_vs_components}
\end{center}
\end{figure}

%% file: plain_bcfw_algo.tex
\begin{algorithm}[htb]
  \caption{Block Coordinate Frank-Wolfe (BCFW) on Product Domain}
  \label{alg:plain_BCFW}
\begin{algorithmic}[1]
% \INPUT Validation set $\mathcal{D}_v$, loss function $\mathcal{L}$
\INPUT $\zeta_0 \in \mathcal{M}^{(1)}\times\dots\times\mathcal{M}^{(n)}\subset\R^m$
% \FOR{$(x_t, y_t)$ {\bfseries in} $\mathcal{D}_{val}$}
\FOR{$k=0,\dots,K$}
\STATE Pick $i$ at random in $\{1,\dots,n\}$
\STATE Find $s_k=\argmin_{s \in \mathcal{M}^{(i)}}s^T\cdot\nabla^{(i)} f(\zeta_k)$
\STATE Let $\tilde{s}_k=: 0_m\in\R^m$ is the zero padding of $s_k$ such that we then assign $\tilde{s}_k^{(i)}:=s_k$ 
\STATE 
Let $\gamma:=\frac{2n}{k+2n}$, \\ or perform line-search: 
   $\gamma =:  \argmin_{\gamma' \in[0,1]} f\left((1-\gamma')\cdot\zeta_k + \gamma'\cdot \tilde{s}_k\right)$
\STATE Update $\zeta_{k+1}= (1-\gamma)\cdot\zeta_k + \gamma\cdot \tilde{s}_k$
% \STATE Update $\zeta_{k+1}^{(i)}= (1-\gamma)\cdot\zeta_k^{(i)} + \gamma\cdot s_k$ and $\zeta_{k+1}^{\setminus(i)} = \zeta_k^{\setminus(i)}$
\ENDFOR
\end{algorithmic}
\end{algorithm}

%% file: bcfw_qcqp_prod_algo.tex
\begin{algorithm}[htb]
  \caption{BCFW with Line-Search on QCQP Product Domain}
  \label{alg:BCFW_QCQP_prod}
\begin{algorithmic}[1]
\INPUT $\zeta_0 \in \{\zeta \mid 
  \sum_{i=1}^n \left(\zeta^{(i)}\right)^T\cdot p_{\mathcal{M}}^{(i)}  + \sum_{i=1}^n\sum_{j=i}^n \left(\zeta^{(i)}\right)^T\cdot Q_{\mathcal{M}}^{(i,j)}\cdot\zeta^{(j)} \leq T
    \quad ; \quad
    A\cdot \zeta \leq b\}$
\FOR{$k=0,\dots,K$}
\STATE Pick $i$ at random in $\{1,\dots,n\}$
\STATE Keep the same values for all other coordinate blocks $\zeta_{k+1}^{\setminus(i)} = \zeta_k^{\setminus(i)}$ and update:
\tiny
\begin{align*}
    \zeta_{k+1}^{(i)}=\argmin_{s} & 
    \left(
    \left(p_f^{(i)}\right)^T    +
    \sum_{j=1}^{i-1}
    % \sum_{j\in\{1,\dots,i-1\}}
    \left(\zeta^{(j)}\right)^T\cdot Q_f^{(i,j)} 
    +
    \sum_{j=i+1}^n
    % \sum_{j\in\{i+1,\dots,n\}}
    \left(\zeta^{(j)}\right)^T\cdot \left(Q_f^{(i,j)}\right)^T 
    \right) 
    \cdot s
    \\ s.t. &
        \left(
    \left(p_{\mathcal{M}}^{(i)}\right)^T
    +
    \sum_{j=1}^{i-1}
    % \sum_{j\in\{1,\dots,i-1\}}
    \left(\zeta^{(j)}\right)^T\cdot Q_{\mathcal{M}}^{(i,j)} 
    +
    \sum_{j=i+1}^n
    % \sum_{j\in\{i+1,\dots,n\}}
    \left(\zeta^{(j)}\right)^T\cdot \left(Q_{\mathcal{M}}^{(i,j)}\right)^T 
    \right) 
    \cdot s \leq T
    \\ &
    A^{(i)}\cdot s \leq b^{(i)}
\end{align*}
\small
\ENDFOR
\end{algorithmic}
\end{algorithm}